\documentclass[12pt]{article}
\usepackage{amsmath, amssymb, amsfonts}
\usepackage{amsthm}
\usepackage{graphicx}
\usepackage{hyperref}
\usepackage{amsfonts}
\usepackage{graphicx}
\usepackage{epstopdf}
\usepackage{subfigure}
\usepackage{epsfig}
\usepackage{float}
\usepackage{xcolor}
\usepackage{cite}

\usepackage{authblk}

\usepackage{multirow}
\usepackage{booktabs}
\usepackage{threeparttable}

\usepackage[ruled,linesnumbered,algo2e]{algorithm2e}
\usepackage{algorithm}      
\usepackage{algpseudocode}

\ifpdf
  \DeclareGraphicsExtensions{.eps,.pdf,.png,.jpg}
\else
  \DeclareGraphicsExtensions{.eps}
\fi

\usepackage{enumitem}
\setlist[enumerate]{leftmargin=.5in}
\setlist[itemize]{leftmargin=.5in}

\newtheorem{theorem}{Theorem}

\theoremstyle{remark}
\newtheorem{remark}{Remark}

\renewenvironment{proof}{{\bf Proof.}}{\qed}

\newcommand{\modified}[1]{\textcolor{blue}{#1}}
\renewcommand{\modified}[1]{#1}
\newcommand{\revised}[1]{\textcolor{blue}{#1}}
\renewcommand{\revised}[1]{#1}

\graphicspath{{images/}}

\begin{document}
\title{Inverse Evolution Layers: Physics-informed Regularizers for Image Segmentation}


\author[*]{Chaoyu Liu}
\author[$\dag$]{Zhonghua Qiao}
\author[*,$\ddag$]{Chao Li}
\author[*]{Carola-Bibiane Schönlieb}
\affil[*]{Department of Applied Mathematics and Theoretical Physics, University of Cambridge, UK (\href{mailto:cl920@cam.ac.uk}{cl920@cam.ac.uk}, \href{mailto:cl647@cam.ac.uk}{cl647@cam.ac.uk}, \href{mailto:cbs31@cam.ac.uk}{cbs31@cam.ac.uk}).}
\affil[$\dag$]{Department of Applied Mathematics \& Research Institute for Smart Energy, The Hong Kong Polytechnic University, Hong Kong (\href{mailto:zhonghua.qiao@polyu.edu.hk}{zhonghua.qiao@polyu.edu.hk}).}
\affil[$\ddag$]{School of Science and Engineering \& School of Medicine, University of Dundee.}

\date{}
\maketitle

\begin{abstract}
    \revised{Traditional image processing methods employing partial differential equations (PDEs) offer a multitude of meaningful \modified{regularizers}, along with valuable theoretical foundations for a wide range of image-related tasks. \modified{This makes their integration into neural networks a promising avenue.} In this paper, we introduce a novel regularization approach inspired by the reverse process of PDE-based evolution models. Specifically, we propose inverse evolution layers (IELs), which serve as bad property amplifiers to penalize neural networks of which outputs have undesired characteristics. Using IELs, one can achieve specific regularization objectives and endow neural networks' outputs with corresponding properties of the PDE models. Our experiments, focusing on semantic segmentation tasks using heat-diffusion IELs, demonstrate their effectiveness in mitigating noisy label effects. Additionally, we develop curve-motion IELs to enforce convex shape regularization in neural network-based segmentation models for preventing the generation of concave outputs. Theoretical analysis confirms the efficacy of IELs as an effective regularization mechanism, particularly in handling training with label issues.}
\end{abstract}



\section{Introduction}
\label{intro}
\revised{Image segmentation is a fundamental task in computer vision, aiming to partition an image into meaningful regions or objects. It plays a crucial role in various applications such as medical image analysis, object recognition, and scene understanding. Traditional approaches to image segmentation often rely on handcrafted features and mathematical models, including variational methods based on partial differential equations (PDEs).}

\revised{Variational energy minimization and other PDE-based methods  \cite{chan2001active,liu2022two,liu2023active,morel2012variational,mumford1989optimal} have long been used in image segmentation tasks. These methods formulate segmentation problems as energy minimization problems, where the energy functional consists of fidelity and regularization components. The fidelity term measures the agreement between the model and the observed data, while the regularization term enforces desired properties on the solution. Classic examples include the Chan-Vese model \cite{chan2001active}, which minimizes the squared distance between the image and an approximated constant within each segment, and imposes length regularization on segment boundaries. The minimization process of the energy functional will drive the segments to evolve according to specific PDEs. In addition to energy-based methods, some PDE-based models directly focus on developing PDEs of segments to do the segmentation \cite{osher1988fronts}. For PDE-based methods, they have a strong theoretical foundation and are well-established in the field of mathematics and physics. This allows for theoretical analysis on properties of the solutions, as well as reduces the complexity of analysis and comprehension of the models' behavior.}

\revised{In recent years, deep learning techniques, particularly deep neural networks, have emerged as powerful tools for image segmentation \cite{badrinarayanan2017segnet,carion2020end,dosovitskiy2020image,isensee2021nnu}. Deep neural networks can automatically learn hierarchical representations from raw image data, enabling them to capture complex patterns and variations in the data. These networks have achieved state-of-the-art performance in various image segmentation tasks, often surpassing traditional PDE-based methods when provided with sufficient training data. However, it can be challenging to mathematically analyze the outputs of neural networks and regulate them to exhibit desired properties.}

\revised{Motivated by the strengths of both traditional mathematical models and data-driven approaches, there is an opportunity to integrate powerful mathematical properties and tools into neural networks for image segmentation tasks. Many efforts have been made to discover connections between PDE-based models and neural networks \cite{celledoni2021structure, weinan2017proposal, lu2018beyond, ruthotto2020deep}, as well as to integrate them into neural networks.}The most direct approach to achieve this integration is to include functions derived from various PDE models in the loss to regularize the network's outputs \cite{bertozzi2016diffuse, chen2019learning}. This loss-inserting method can be designed on a range of PDEs that can be interpreted as gradient flows of specific energy functionals, which is a common feature in energy-based models designed for image segmentation tasks. As an illustration, given an image $I$, an energy-based model can be conceptualized as follows,
\begin{equation}
  \label{energy-based model}
  \min_{\mathbf{u}, \mathbf{f}} E(\mathbf{u}, \mathbf{f}):=E_{fit}(\mathbf{u}, \mathbf{f}, I)+R(\mathbf{u}),
\end{equation}
where $\mathbf{u}$ denotes the image segments and $\mathbf{f}$ are approximating functions for the image $I$ on segments. The terms $E_{fit}(\mathbf{u}, \mathbf{f}, I)$ and $R(\mathbf{u})$ represent the fidelity and regularization components, respectively. \modified{For instance, the energy functional formulated in the Chan-Vese model \cite{chan2001active} is given as follows,
\begin{equation}
  \label{cv model}
  \min_{\mathbf{u}, \mathbf{f}} E_{cv}(\mathbf{u}, \mathbf{f}):=\sum_{i=1}^n \int_{u_i} (c_i - I)^2\mbox{d}x + \sum_{i=1}^n Length(u_i) \mbox{d}x,
\end{equation}
In Chan-Vese model, the fidelity term within each segment $u_i$ is characterized by the squared distance between given image $I$ and an approximated constant $c_i$, and the regularization term is defined as the length of the boundaries of $u_i$'s.} A straightforward way to incorporate the energy-based models into neural networks \modified{is to use (\ref{energy-based model}) as a loss function for neural networks.} Some related work can be found in \cite{bertozzi2016diffuse, chen2019learning, liu2022deep}. \modified{It is also worth mentioning} that all the loss-inserting methods are tailored to energy-based PDE models. \modified{However, numerous PDEs lack specific energy characteristics and cannot be depicted by gradient flows of an explicit energy. For instance, integrating curvature flows based non-energy PDE models into neural networks through loss-inserting methods is a notably challenging task. The integration of non-energy PDEs into the loss function may be achievable by introducing an additional time variable and PDE residuals, akin to the methodology applied in  Physics-Informed Neural Networks (PINNs) \cite{raissi2019physics}. Nevertheless, it can significantly increase the training complexity and computational costs.}

In addition to loss-inserting methods, researchers have attempted to modify the architectures of neural networks to enhance their interpretability. For example, Chen $et$ $al.$ \cite{chen2015learning} developed a neural network for image restoration by parameterizing a reaction-diffusion process derived from conventional PDE-based image restoration models. Similar techniques that employ convolutional neural networks (CNNs) to learn parameters in active contour models can be found in \cite{cheng2019darnet,hatamizadeh2019end,marcos2018learning}. In \cite{lunz2018adversarial}, Lunz $et$ $al.$ combine the neural networks with variational models for inverse problems by using a neural network to learn regularizers in the models. Ruthotto $et$ $al.$ \cite{ruthotto2020deep} imparted mathematical properties of different types of PDEs by imposing certain constraints on layers. Recently, Tai $et$ $al.$ \cite{tai2023pottsmgnet} proposed the PottsMGNet based on a control problem and a first-order operator splitting algorithm. While networks with specialized architectures or constraints derived from PDE models have seen significant improvements in interpretability, their performance and learning capabilities can be restricted by these specific architectural constraints. Furthermore, the fixed architecture and constraints pose challenges in striking an optimal balance between interpretability and learning ability. Consequently, their specialized structures and constraints limit their applicability to specific problem types and render them unsuitable to be employed to other widely used and advanced neural network architectures. Furthermore, for the majority of neural networks, it is still quite difficult to provide theoretical analysis of their outputs, even with architectural modifications.

\revised{In this paper, we propose a novel method to integrate PDE models into neural networks using inverse evolution layers (IELs). These layers, derived from the inverse processes of PDE models, can be used as bad property amplifiers to regularize the neural networks. We insert the IELs after the neural networks to amplify some undesirable properties in the neural networks outputs. Consequently, the IELs will penalize the outputs with undesirable properties and force neural networks to produce more desirable results. Our proposed method differs from traditional loss-inserting techniques in that it is amenable to analysis, and theoretical results can be provided to substantiate the effectiveness of the IELs. And it is applicable to both energy-based and non-energy PDE models. Furthermore, experiments show that our method can achieve better performance than traditional loss-inserting techniques on noisy label problems.}

The main contributions of this work are:
\begin{quote}
  \noindent
  \revised{1. We develop the inverse evolution layers (IELs) for incorporating PDE models into neural networks. These layers offer good interpretability and serve as effective regularizers by amplifying undesirable properties of neural networks, thus compelling the neural networks' outputs to possess the desired properties during training. The IELs can be designed based on both energy-based models and non-energy models. Additionally, our framework preserves the learning capacity of the given neural networks as it does not alter their original structure or impose any constraints on them.}

  \noindent
  2. We introduce heat-diffusion IELs and curve-motion IELs to incorporate smoothness and convexity regularization to the neural network outputs, respectively. Subsequently, we evaluate their performance on several semantic segmentation neural networks and datasets. Experimental results substantiate the effectiveness of our proposed heat-diffusion and curve-motion IELs.

  \noindent
  3. We offer theoretical backing for the heat-diffusion IELs, demonstrating their roles as regularizers. We also provide theoretical evidence showcasing the IELs' capability to address noisy label issues under reasonable assumptions.
\end{quote}

The rest of this paper is organized as follows. In section \ref{iels}, we provide a comprehensive description of the inverse evolution layers and the integrated architecture of neural networks and these layers, along with an explanation of how these layers can be utilized to regularize neural networks. In section \ref{diffusion_iels}, we present the heat-diffusion IELs to tackle the challenges of semantic segmentation with noisy labels. The experimental results validate the effectiveness of our heat-diffusion IELs in significantly reducing noise overfitting in neural networks. In Section \ref{curve-motion IELs}, we introduce the curve-motion IELs, which are designed to impose convex shape regularization on the outputs of neural networks. Our experiments demonstrate that the curve-motion IELs are effective in enhancing the convexity of network outputs. The theoretical results for our IELs will be showed in section \ref{theoretical results}. In section \ref{conclusion}, we summarize the main contributions and findings of our work, and discuss potential avenues for future researches.

\section{Neural Networks with Inverse Evolution Layers}
\label{iels}
\subsection{Inverse Evolution Layers}
In this section, we introduce the inverse evolution layers (IELs) which play a pivotal role in our framework for integrating specific characteristics of mathematical models into neural networks.

In physics, evolution processes can be modelled by partial differential equations that vary over time. For instance, diffusion processes can be modelled by the diffusion equations \cite{perona1990scale,weickert1998anisotropic}. Additionally, advection–diffusion–reaction systems can be depicted through advection–diffusion–reaction equations \cite{nordstrom1990biased}. The solutions of these processes inherently possess favorable properties. For example, the solution to a diffusion process will have good smoothness. Rather than enhancing these desirable properties, we aim to develop layers that amplify opposing unfavorable properties. Define a rectangular domain $\Omega$ of the form $[a, b]\times[c,d]$ with $a,b,c,d \in \mathbb{R}, a\le b, c\le d$ and let $u: (\Omega,[0,T]) \to \mathbb{R}$ be the solution of the partial differential equations of the following form
\begin{equation}
  \label{general_pde}
  \left\{
  \begin{array}{ll}
    u_t = \mathcal{F}(u), &  t\in[0,T]\\
    u(x, 0) = u_0, & \forall x\in\Omega\\
    \frac{\partial u}{\partial x}(x,t) = 0, & \forall x \in \partial\Omega,\quad t\in[0,T]
\end{array}
\right.
\end{equation}
where $\mathcal{F}$ refers to a combination of linear and nonlinear differential operators which may include a variety of gradient operators and $u_0$ is a given initial value. Additionally, the boundary condition is set as the Neumann condition.

\begin{remark}
  One can also consider more complicated PDE systems with different boundary conditions. Nevertheless, as traditional mathematical models for image processing typically revolve around a single partial differential equation with the Neumann condition, our focus in this paper is solely on PDEs conforming to the form of \eqref{general_pde}.
\end{remark}

Using powerful tools in numerical mathematics, one can efficiently solve these PDEs and determine the numerical solution at any time in the evolution. For example, if we let $h$ represent the spacing distance between two adjacent pixels and employ the central finite difference scheme for spatial discretization, combined with solving the general partial differential equation using a straightforward forward Euler's formula for temporal discretization, we can derive the following discrete numerical scheme
\begin{equation}
  \label{forward_pde}
  \frac{U^{n+1}-U^{n}}{\Delta t} = F_h(U^{n}).
\end{equation}
Here $U^{n}$ and $U^{n+1}$ denote the numerical solutions of $u(t)$ at time $T_{n}$ and $T_{n+1}$, respectively. And $T_{n+1}-T_n = \Delta t$, with $T_0 = 0$ and time step being $\Delta t$. The symbol $F_h$ represents finite difference approximations of $\mathcal{F}$. Typically, $F_h(U^{n})$ can be conceptualized as a combination of convolutions of $U^{n}$ with designated filters.

We rewrite equation~(\ref{forward_pde}) as follows
\begin{equation}
  \label{forward_pde2}
  U^{n+1} = U^{n} + F_h(U^{n})\Delta t.
\end{equation}
The equation~(\ref{forward_pde2}) demonstrates that once we have the value of $U^{n}$, we can determine the value of $U^{n+1}$. Furthermore, starting with an initial condition of $U_0$, one can obtain a numerical solution at any time by iteratively applying this equation.

Based on the equation~(\ref{forward_pde2}), we propose inverse evolution layers. Instead of solving evolutions in a forward manner, these layers are utilized to numerically simulate the inverse process of evolutions. From the equation~(\ref{forward_pde2}), we can see that $U^{n+1}$ is calculated by adding the term $F_h(U^{n})\Delta t$ to $U^{n}$. Therefore, a natural way to obtain the simulation of the inverse evolution is to replace the ``+" with ``-".  On this basis, we introduce the concept of inverse evolution layers. The construction of an inverse evolution layer is quite straightforward: we design a layer such that, given an input $U$, its output is
\begin{equation}
  \label{backward_pde}
  L(U) = U - F_h(U)\Delta t.
\end{equation}
In contrast to forward evolution, the inverse evolution layers are expected to amplify certain undesired properties, making them suitable for use in neural networks as a type of regularizers. \modified{It is also worth mentioning that $h$ and $\Delta t$ are fixed hyperparameters of the network, therefore IELs do not involve any trainable parameters.}

\begin{remark}
  The deduction process for IELs described above is based on the explicit scheme, which is known to suffer accuracy and stability issues. However, in section \ref{theoretical results}, we will explain why we choose explicit schemes and prove the feasibility of the deduction process.
\end{remark}

\subsection{The Architecture of Neural Networks plus Inverse Evolution Layers}
In this section, we will demonstrate how to apply IELs to neural networks to achieve regularization goals. First, we insert inverse evolution layers into neural networks. Given a neural network, the way we incorporate IELs is to add them to the output of the network before computing the loss function, as depicted in Figure~\ref{architecture}. In our framework, we activate the inverse evolution layers during training, while during the evaluation and prediction we deactivate the inverse evolution layers.
\begin{figure*}
\begin{center}
\includegraphics[width=0.6\linewidth]{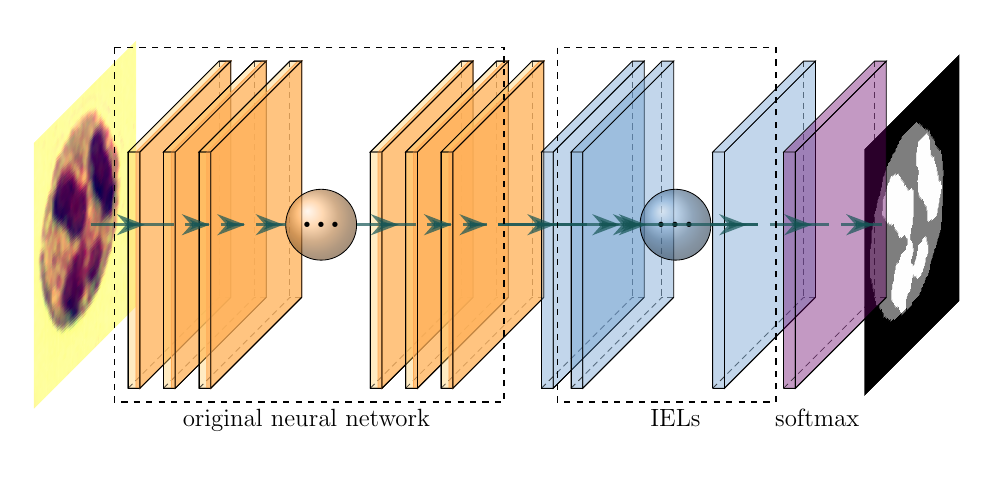}
\end{center}
   \caption{The architecture of neural networks with IELs. From left to right: input, original neural networks, inverse evolution layers (IELs), softmax layer and final output. The original neural network can be any kind of neural networks. The IELs behave like a bad property amplifier. During training, the IELs are activated while in prediction, they are deactivated.}
\label{architecture}
\end{figure*}

As previously mentioned, inverse evolution layers are designed to amplify undesirable properties that run counter to those of the corresponding forward physical process. For example, as the solution of the heat diffusion process is typically smooth, the corresponding inverse evolution layers will accentuate the roughness of their inputs. Consequently, if the input of the inverse evolution layers, $i.e.$ the output of the neural network, contains noise, the noise will be amplified after passing through the IELs. A visual demonstration illustrating IELs as noise amplifiers is presented in Figure~\ref{demonstration}.
\begin{figure*}
\begin{center}
\includegraphics[width=0.6\linewidth]{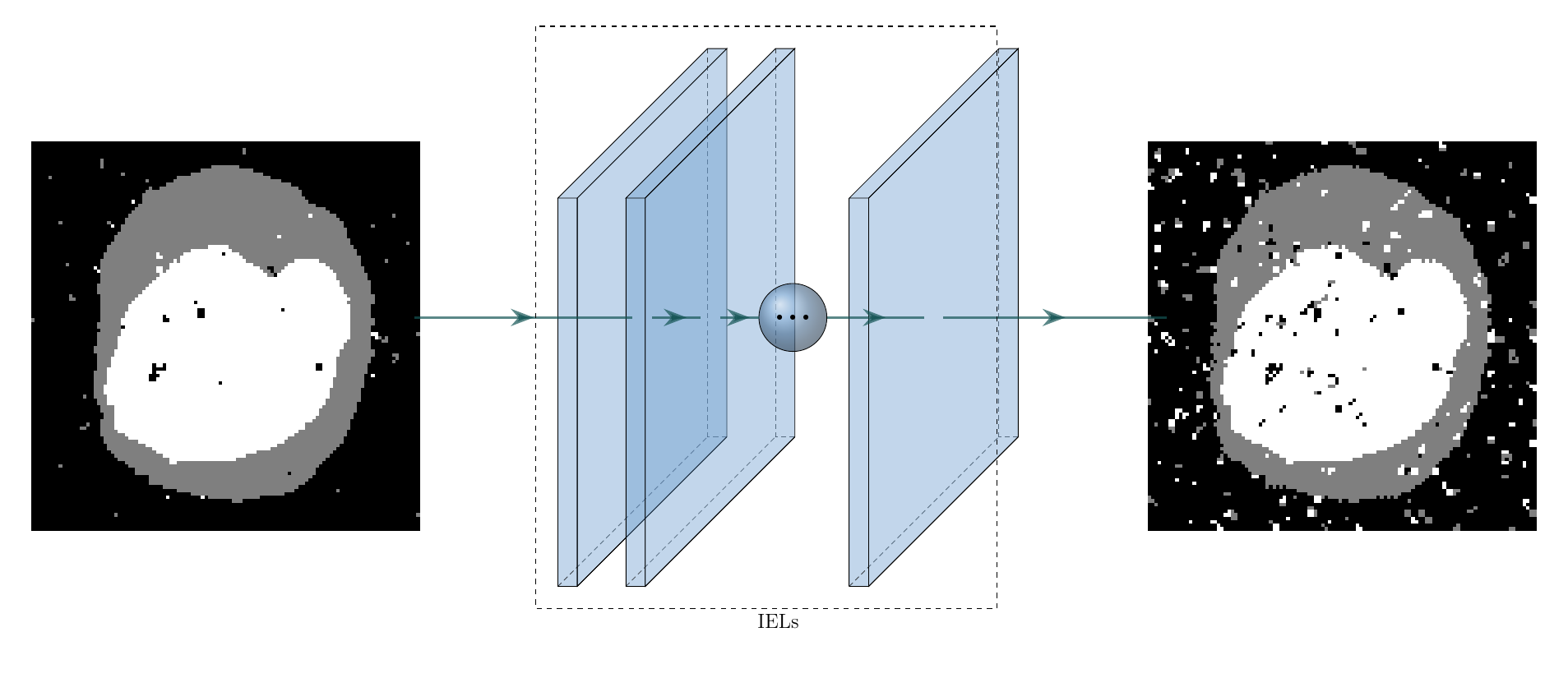}
\end{center}
   \caption{A visual demonstration illustrating IELs as noise amplifiers. From left to right: input, IELs and output.}
\label{demonstration}
\end{figure*}

Due to the specialized construction of the IELs, it can be anticipated that during the training phase, the undesired characteristics of the neural network's outputs will be amplified after passing through the IELs. Thus, the IELs can function as a form of regularization by forcing the neural network not to generate outputs with undesirable properties.

\section{Heat-diffusion IELs}
\label{diffusion_iels}
\subsection{Derivation of Heat-diffusion IELs}
In this section, we present an example by developing inverse evolution layers based on the heat diffusion equation.

The heat diffusion equation is formulated as
\begin{equation}
    \left\{
  \label{heat diffusion equation}
  \begin{array}{ll}
    u_t = \Delta u, &  t\in[0,T]\\
    u(x, 0) = u_0, & \forall x\in\Omega\\
    \frac{\partial u}{\partial x}(x_0,t) = 0, & \forall x_0 \in \partial\Omega, \quad t\in[0,T]
\end{array}
\right.
\end{equation}
For the heat-diffusion process, one can prove
\begin{equation}
  \frac{\mbox{d}(\int_\Omega |\nabla u|^2 \mbox{d}x)}{\mbox{d}t}\le 0, \quad \forall t.
\end{equation}
This formula implies that the $L^2_\Omega$-norm of $\nabla u$ will become smaller and smaller. \modified{Consequently, $u$ will become smoother and smoother over time. Therefore, we employ the heat-diffusion IELs to prevent neural networks from overfitting noise in images. We will demonstrate this for the noise issue in labels for semantic segmentation.}

Define a discrete rectangular domain
 $$\Omega_h: (h,2h,\cdots,M_1h)\times(h,2h,\cdots,M_2h).$$
Let $U$ be a function defined on $\Omega_h$. According to equation~(\ref{backward_pde}), we can construct heat-diffusion IELs through the following formula
\begin{equation}
  \label{heat_diffusion_iels}
  L(U) = U - F_h^{\Delta}(U)\Delta t,
\end{equation}
Here $F_h^{\Delta}$ is the finite difference approximation of \modified{Laplacian operator, more precisely} the value $F_h^{\Delta}(U)$ at the pixel $(ih,jh)\in \Omega_h$ is given as
\begin{equation}
  \label{F_laplacian}
  \begin{split}
      (F_h^{\Delta}(U))_{i,j} &= \frac{U_{i+1,j}+U_{i-1,j}+U_{i,j-1}+U_{i,j+1}-4U_{i,j}}{h^2},\\
      & i=1,2,\cdots,M_1; j=1,2,\cdots,M_2,
  \end{split}
\end{equation}
where $U_{i,j}$ is short for $U(ih,jh)$.

Considering the Neumann boundary condition, we define
\begin{equation}
  \label{Neumann bc}
\begin{split}
  &U_{0, \cdot}=U_{1, \cdot}, U_{M_1+1, \cdot}=U_{M_1, \cdot}\\
  &U_{\cdot ,0}=U_{\cdot ,1}, U_{,M_2+1}=U_{\cdot ,M_2}
\end{split}
\end{equation}

Substituting \eqref{Neumann bc} into \eqref{F_laplacian}, one can get the following matrix form for $F_h^{\Delta}(U)$
$$
F_h^{\Delta}(U) = \frac{1}{h^2}(UD_{M_2} +D_{M_1}U),
$$
where
\begin{equation}
  \label{d_i}
D_{ M_k } = \left[ \begin{array} { c c c c c c } - 1 & 1 & & & & \\ 1 & - 2 & 1 & & & \\ & & \ddots & \ddots & \ddots & \\ & & & 1 & - 2 & 1 \\  & & & & 1 & - 1 \end{array} \right]_{ M_k \times M_k }, \quad k=1, 2.
\end{equation}
Under the matrix form, it becomes evident that $F_h^{\Delta}(U)$ is equivalent to convolution operators with replicate padding and the corresponding convolution kernel is the following $3\times 3$ filter,
\begin{equation}
  \left[ \begin{matrix} 0 &1 &0 \\ 1 &-4 &1 \\ 0 &1 &0 \\ \end{matrix} \right].
\end{equation}

As previously discussed, the designed heat-diffusion IELs are expected to amplify noise. During training, these heat-diffusion IELs are activated, resulting in the amplification of noise in the neural networks's outputs. Instead of using the original outputs of the neural network, the outputs with amplified noise are utilized to fit the noisy labels under the loss. Therefore, the impact of noisy labels will be mitigated by the heat-diffusion IELs. \modified{During evaluation and prediction, the heat-diffusion IELs are deactivated.}
\subsection{Efficiency of the Heat-diffusion IELs}
In this part, we will assess the efficiency of the heat-diffusion IELs. In fact, minimal computational cost and storage are required for the heat-diffusion IELs. The numerical scheme used to design the IELs is explicit, making the computation \modified{very} efficient. Furthermore, the number of IELs does not significantly impact the overall efficiency. This is because, in real implementation, one can merge all the IELs into one layer. Let $L_1$ and $L_2$ are two IELs such that \eqref{heat_diffusion_iels} is satisfied, i.e.,
\begin{equation}
  L_i(U) =U - \Delta t\times F_h^{\Delta}(U), i=1,2
\end{equation}
Then we have
\begin{equation}
\begin{split}
    L_2(L_1(U)) &= L_1(U) - \Delta t\times F_h^{\Delta}(L_1(U))\\
    &= U - 2\Delta t\times F_h^{\Delta}(U)+ (\Delta t)^2\times (F_h^{\Delta})^2(U),
\end{split}
\end{equation}
where $(F_h^{\Delta})^2$ can also be depicted by a convolution with filters. Therefore, we can use only one layer $L$ such that
\begin{equation}
  L(U) = U - 2\Delta t F_h^{\Delta}(U)+ (\Delta t)^2(F_h^{\Delta})^2(U)
\end{equation}
to replace $L_1, L_2$. By repeatly doing so, all IELs can be merged to one layer. More generally, for $n$ IELs, the final layer can be computed through the following formula
\begin{equation*}
  \begin{split}
    L(U) &= (1-\Delta t F_h^{\Delta})^n(U)\\
         &= \sum\limits_{k=0}^{n}C_k^n(-\Delta t)^k(F_h^{\Delta})^{k}(U),
  \end{split}
\end{equation*}
where $C_k^n$ is the number of $k$-combinations. This implies that $L(U)$ can be computed in parallel.
Hence the heat-diffusion IELs have minimal impact on the \modified{computational} efficiency of the original model. In experiments, we observed that models with heat-diffusion IELs require more epochs to reach their peak performance. However, the additional training time is acceptable, especially when considering the improvement on datasets with noisy labels. Furthermore, it is worth noting that we maintained consistent hyperparameters with the original model throughout all experiments. Fine-tuning hyperparameters for models with IELs could potentially accelerate convergence.

\subsection{Experiments for Heat-diffusion IELs}
In our experiments\footnote{Codes are available at \href{https://github.com/cyliu111/Inverse-Evolution-Layers-Segmentation}{https://github.com/cyliu111/Inverse-Evolution-Layers-Segmentation}}, we will evaluate the effectiveness of heat-diffusion IELs on semantic segmentation tasks using different datasets with both normal and noisy labels. The noisy labels are pixel-wise, which means the classes of some pixels will be replaced with random ones. Our experiments are conducted on three well-known neural network architectures: Unet \cite{ronneberger2015u}, DeepLab \cite{chen2018encoder} and HRNetV2-W48 \cite{sun2019high}. Table~\ref{num_iels} provides the details on the number of IELs and $\Delta t$ we adopt for each network and for all the experiments, the spacing distance $h$ used for spatial discretization is chosen to be 1. \revised{In addition, we recorded the training time costs across several experiments to demonstrate the efficiency of IELs, as presented in Table \ref{table:time_comparison}.} Our experiments demonstrate that the heat-diffusion IELs are particularly effective for datasets with noisy labels.

\begin{table}
\label{num_iels}
\begin{center}
  \begin{tabular}{cccc}
  		\hline
      neural networks &datasets & $\Delta t$ &the number of IELs\cr\hline
  		\multirow{2}{*}{Unet} & WBC & 0.1 & 20\cr\cmidrule(lr){2-4}
      & DSB & 0.1 & 30\cr\hline
      DeepLabV3+& Cityscapes & 0.1 & 35\cr\hline
      HRNetV2-W48& Cityscapes & 0.1 & 30\cr\hline
  	\end{tabular}
\end{center}
\caption{IELs configurations for different neural networks on different datasets.}
\end{table}

\begin{table}
    \label{table:time_comparison}
    \begin{center}
        \begin{tabular}{ccccc}
            \hline
        Type of IELs &Datasets & GPU &No IELs &IELs\cr\hline
            \multirow{2}{*}{Heat-diffusion} & WBC & V100 & 8.414&8.558\cr\cmidrule(lr){2-5}
        & DSB & V100 & 32.57&33.21\cr\hline
        Curve-motion& REFUGE & A100 & 28.93&37.50\cr\hline
        \end{tabular}
    \end{center}
    \caption{Comparison of time costs (seconds per epoch) between the Unet with and without IELs.}
\end{table}

\subsubsection{Unet with Heat-diffusion IELs on WBC}
We conducted a performance comparison between Unet and Unet with heat-diffusion IELs on a small dataset called the White Blood Cell (WBC) \cite{zheng2018fast} dataset which consists of three hundred images. In the experiment, 270 samples were used for training and the remaining 30 samples were used for validation. The comparison is conducted on both normal and noisy labels. The way we add noise to training labels is to choose some small windows on each image and replace their label values with a random class. All noise windows are $2\times 2$ and fixed, and the total area of the noise is set to 10 percent of the image size. \revised{Noise is introduced exclusively in the training set, while all images and labels in the validation and test sets remain clean.} The Unet structure used in the experiment has depth of 5 and the specific configuration in each level consists of $3\times 3$ convolution layers, instance norm and leaky ReLu. Downsampling and upsampling are achieved by pooling operation and up-convolution, respectively. The total number of epochs is set to 20, with a learning rate of 0.0001. The loss function used is the cross-entropy loss.

The experimental results for the WBC dataset are presented in Figure~\ref{wbc} where the evaluation metric is the mean of dice score (DS) on the validation dataset. Segmentation results on noisy labels are displayed in Figure~\ref{wbc_fig}. Figure~\ref{wbc} shows that the original Unet and Unet with IELs have comparable performance on normal labels but Unet with IELs is much more robust to noisy labels. These results suggest that heat-diffusion IELs do not reduce the performance of the original Unet but rather significantly improve it on noisy labels. In Figure~\ref{wbc_fig}, the results of the Unet are full of noise, which can be expected since neural networks are prone to the noise \cite{zhang2021understanding}, while the segmentation maps of Unet with IELs have almost no noise. These results also indicate that the designed heat-diffusion IELs can act as regularizers and effectively prevent the Unet from overfitting noise.
\begin{figure}
  \centering
  \includegraphics[width=8cm]{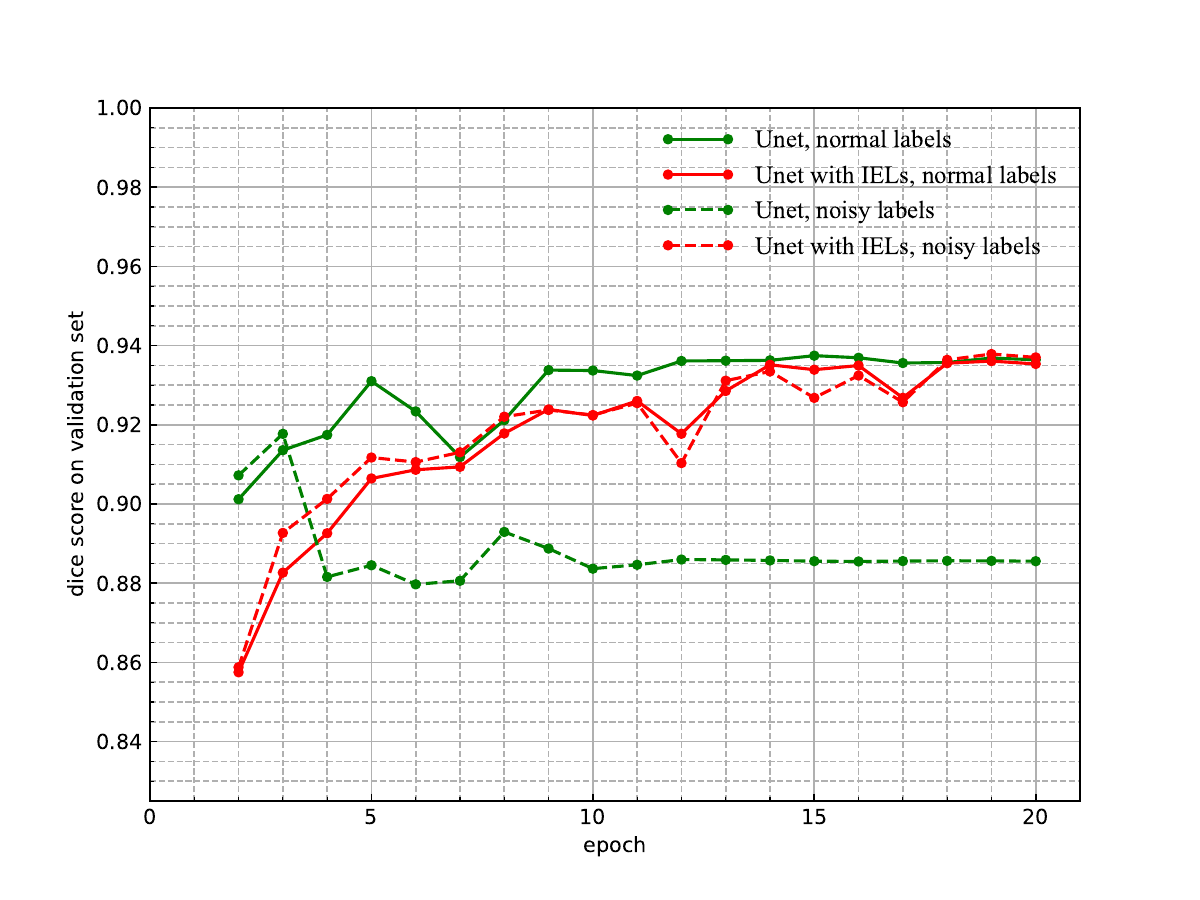}
 \caption{Dice score on the validation set of White Blood Cell (WBC) dataset based on Unet \cite{ronneberger2015u} and Unet with heat-diffusion IELs. The results of Unet and Unet with heat-diffusion IELs are highlighted by green lines and red lines, respectively. The results on normal labels and noisy labels are marked by solid lines and dotted lines, respectively.} \label{wbc}
\end{figure}

\begin{figure*}
	\centering
	\subfigure{
	\begin{minipage}[b]{0.12\linewidth}
	\includegraphics[width=1\linewidth]{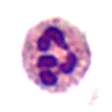}
	\includegraphics[width=1\linewidth]{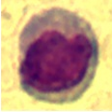}
	\includegraphics[width=1\linewidth]{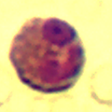}
	\includegraphics[width=1\linewidth]{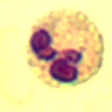}
	\includegraphics[width=1\linewidth]{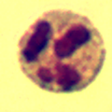}
	\centerline{\tiny (a) Images}
	\end{minipage}
	}
	\subfigure{
	\begin{minipage}[b]{0.12\linewidth}
    \includegraphics[width=1\linewidth]{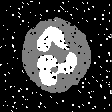}
  	\includegraphics[width=1\linewidth]{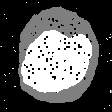}
  	\includegraphics[width=1\linewidth]{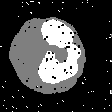}
  	\includegraphics[width=1\linewidth]{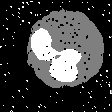}
  	\includegraphics[width=1\linewidth]{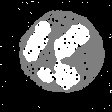}
	\centerline{\tiny (b) Unet}
	\end{minipage}
	}
	\subfigure{
	\begin{minipage}[b]{0.12\linewidth}
    \includegraphics[width=1\linewidth]{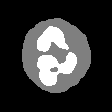}
  	\includegraphics[width=1\linewidth]{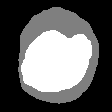}
  	\includegraphics[width=1\linewidth]{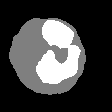}
  	\includegraphics[width=1\linewidth]{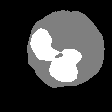}
  	\includegraphics[width=1\linewidth]{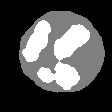}
	\centerline{\tiny (c) Unet with IELs}
	\end{minipage}
	}
	\subfigure{
	\begin{minipage}[b]{0.12\linewidth}
    \includegraphics[width=1\linewidth]{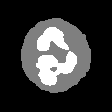}
  	\includegraphics[width=1\linewidth]{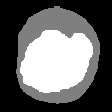}
  	\includegraphics[width=1\linewidth]{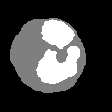}
  	\includegraphics[width=1\linewidth]{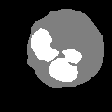}
  	\includegraphics[width=1\linewidth]{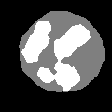}
	\centerline{\tiny (d) ground truth}
	\end{minipage}
	}
	 \caption{Comparison between Unet \cite{ronneberger2015u} and Unet with heat-diffusion IELs on White Blood Cell (WBC) dataset with noisy labels. (a) and (d) are the original images and the corresponding ground truth, respectively. (b) and (c) exhibit the segmentation results of the original Unet and Unet with heat-diffusion IELs, respectively.} \label{wbc_fig}
\end{figure*}

\subsubsection{Unet with Heat-diffusion IELs on 2018 Data Science Bowl}
After conducting experiments on the WBC dataset, we extended the comparison to include the 2018 Data Science Bowl dataset \cite{caicedo2019nucleus}, which contains 607 training and 67 test images. In our experiment, the test images are used for validation during training and the comparison is also conducted on both normal labels and noisy labels. The way we add noise is identical to that in the WBC dataset except that the size of noise windows is tuned to $3\times 3$. The comparison results are displayed in Figure~\ref{dsb2018} and Figure~\ref{dsb2018_fig}. Figure~\ref{dsb2018} shows that on the 2018 Data Science Bowl dataset, the heat-diffusion IELs can still maintain the performance of the Unet on normal labels and prevent the Unet from overfitting to noisy labels. As shown in the Figure~\ref{dsb2018_fig}, the original Unet still suffers a lot from noisy labels while Unet with IELs achieves much more satisfactory results on noisy labels.

\begin{figure}
  \centering
  \includegraphics[width=8cm]{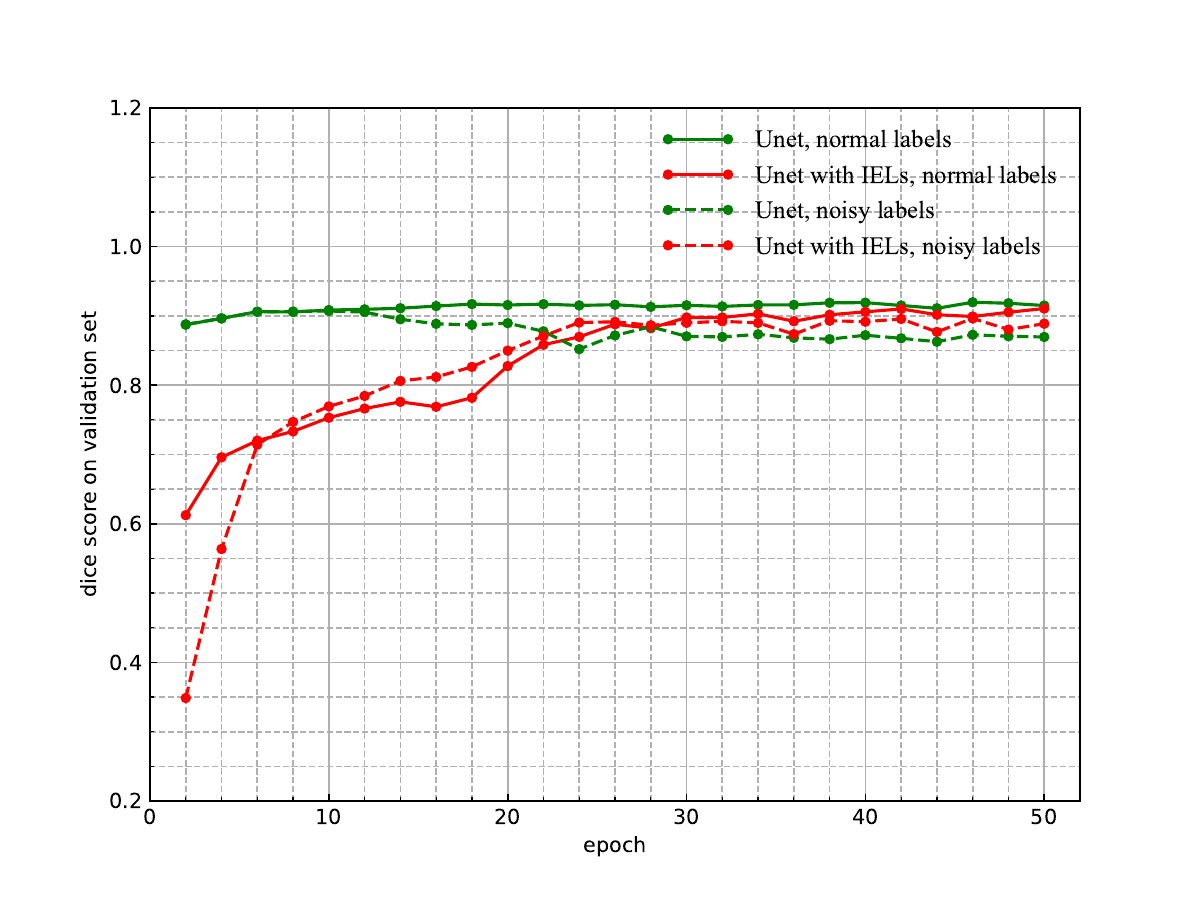}
 \caption{Dice score on the validation set of 2018 Data Science Bowl (DSB) dataset based on Unet \cite{ronneberger2015u} and Unet with heat-diffusion IELs. The results of Unet and Unet with heat-diffusion IELs are highlighted by green lines and red lines, respectively. The results on normal labels and noisy labels are marked by solid lines and dotted lines, respectively.} \label{dsb2018}
\end{figure}

\begin{figure*}
	\centering
	\subfigure{
	\begin{minipage}[b]{0.12\linewidth}
	\includegraphics[width=1\linewidth]{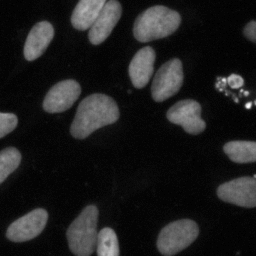}
	\includegraphics[width=1\linewidth]{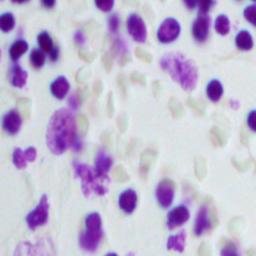}
	\includegraphics[width=1\linewidth]{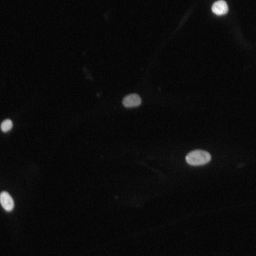}
	\includegraphics[width=1\linewidth]{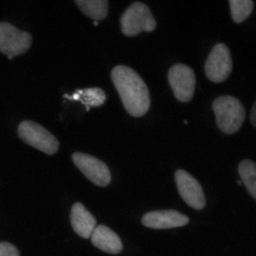}
	\includegraphics[width=1\linewidth]{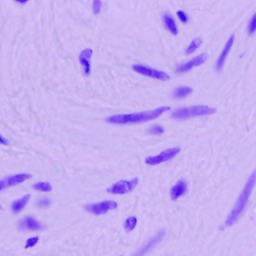}
	\centerline{\tiny (a) Images}
	\end{minipage}
	}
	\subfigure{
	\begin{minipage}[b]{0.12\linewidth}
    \includegraphics[width=1\linewidth]{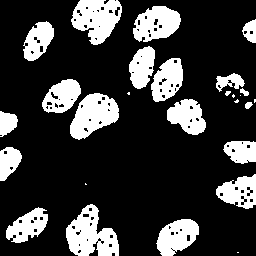}
  	\includegraphics[width=1\linewidth]{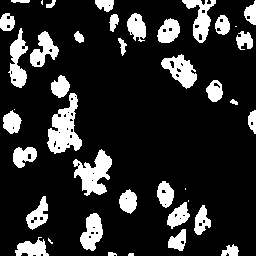}
  	\includegraphics[width=1\linewidth]{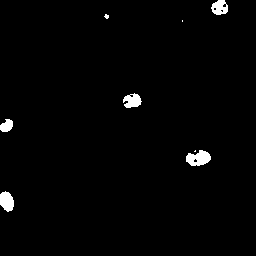}
  	\includegraphics[width=1\linewidth]{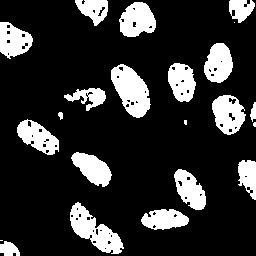}
  	\includegraphics[width=1\linewidth]{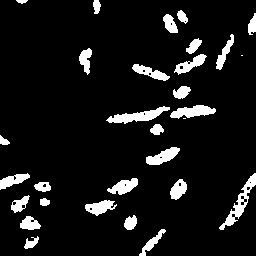}
	\centerline{\tiny (b) Unet}
	\end{minipage}
	}
	\subfigure{
	\begin{minipage}[b]{0.12\linewidth}
    \includegraphics[width=1\linewidth]{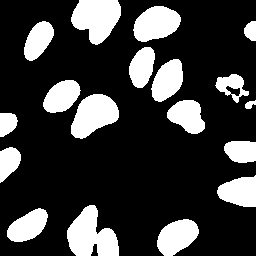}
  	\includegraphics[width=1\linewidth]{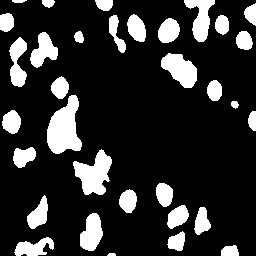}
  	\includegraphics[width=1\linewidth]{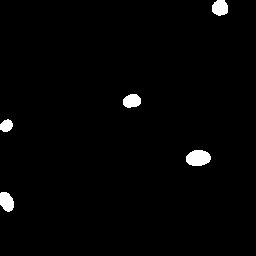}
  	\includegraphics[width=1\linewidth]{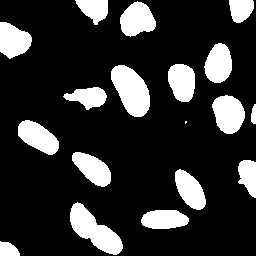}
  	\includegraphics[width=1\linewidth]{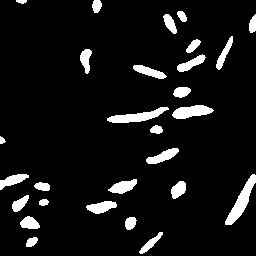}
	\centerline{\tiny (c) Unet with IELs}
	\end{minipage}
	}
	\subfigure{
	\begin{minipage}[b]{0.12\linewidth}
    \includegraphics[width=1\linewidth]{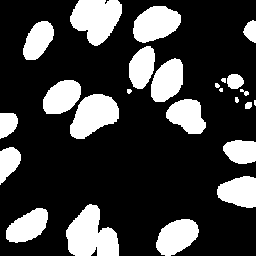}
  	\includegraphics[width=1\linewidth]{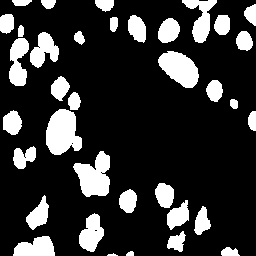}
  	\includegraphics[width=1\linewidth]{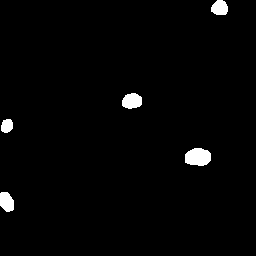}
  	\includegraphics[width=1\linewidth]{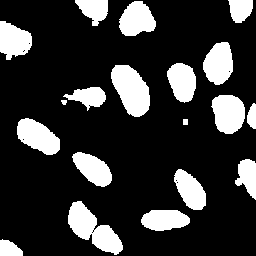}
  	\includegraphics[width=1\linewidth]{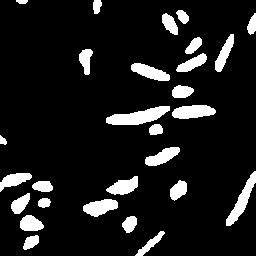}
	\centerline{\tiny (d) Ground truth}
	\end{minipage}
	}
	 \caption{Comparison between Unet \cite{ronneberger2015u} and Unet with heat-diffusion IELs on 2018 Data Science Bowl dataset with noisy labels. (a) and (d) are the original images and the corresponding ground truth, respectively. (b) and (c) exhibit the segmentation results of the original Unet and Unet with heat-diffusion IELs, respectively.} \label{dsb2018_fig}
\end{figure*}

\subsubsection{DeepLabV3+ with Heat-diffusion IELs on Cityscapes}
In addition to the datasets for medical images, we also evaluate the performance of our heat-diffusion IELs on images obtained from real-world scenarios. The dataset we employ for the evaluation is the Cityscapes dataset \cite{cordts2016cityscapes} which contains 2795 train and 500 validation images and the corresponding neural networks we test on this dataset are the DeepLabV3+ \cite{chen2018encoder} and HRNetV2-W48 \cite{sun2019high}. The way we add noise is identical to the previous experiments except that the size of noise windows is $5\times 5$, which is also quite small compared to the image size $(1024\times 2048)$.

For DeepLabV3+, we use ResNet-101 as the backbone and the results are presented in Figure~\ref{deeplab}, where the mean of class-wise intersection over union (mIoU) is adopted as the evaluation metric. These results indicate that for normal labels, DeepLabV3+ and its IELs counterpart show similar performance while for noisy labels DeepLabV3+ with IELs outperforms the original DeepLabV3+. This demonstrate that IELs are effective in mitigating the impact of noise in labels. In addition, the detailed segmentation maps are provided in Figure~\ref{deeplab_fig} which indicates that DeepLabV3+ tends to overfit the noise in labels whereas DeepLabV3+ with heat-diffusion IELs can generate accurate segmentation map with minimal noise.

\begin{figure}
  \centering
  \includegraphics[width=8cm]{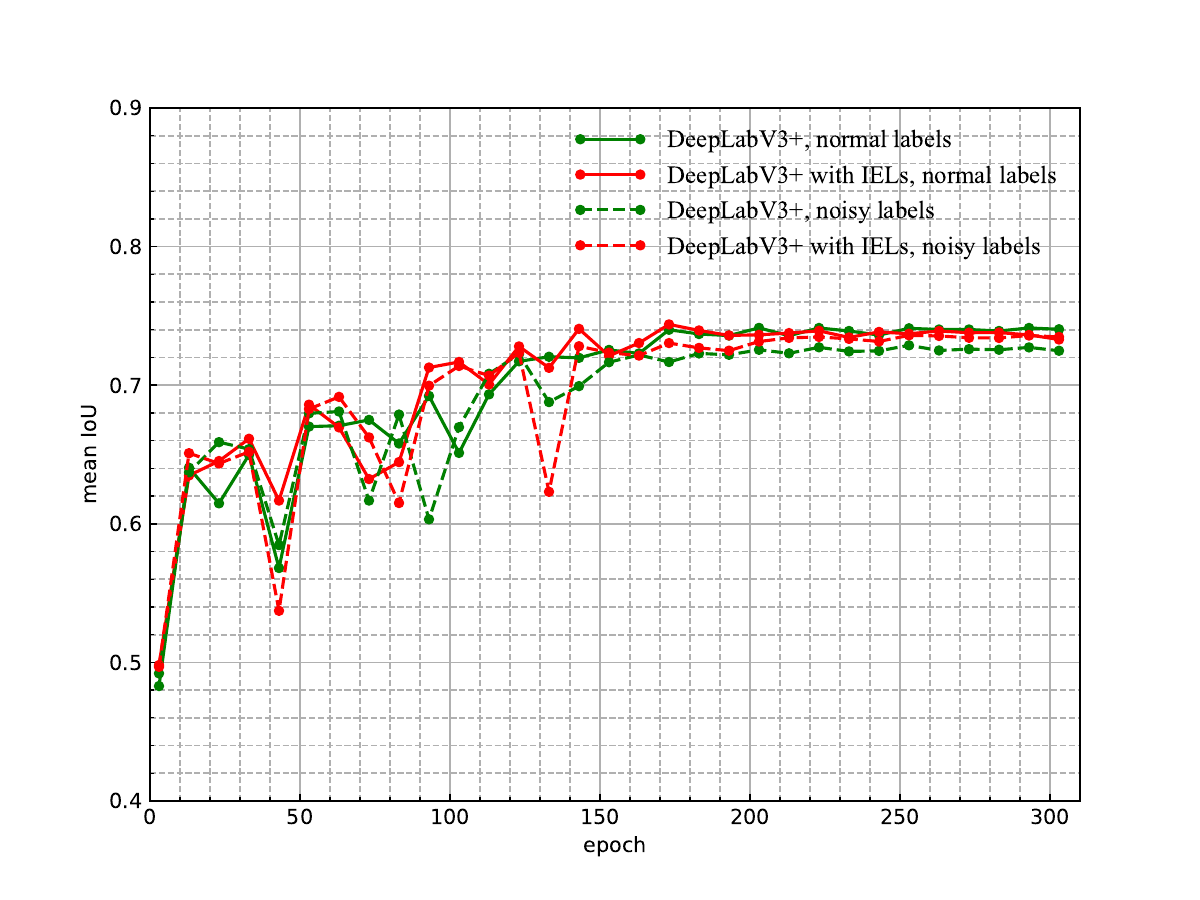}
 \caption{Mean IoU on the validation set of Cityscapes dataset based on DeepLabV3+ \cite{chen2018encoder} and DeepLabV3+ with heat-diffusion IELs. The results of DeepLabV3+ and DeepLabV3+ with heat-diffusion IELs are highlighted by green lines and red lines, respectively. The results on normal labels and noisy labels are marked by solid lines and dotted lines, respectively.} \label{deeplab}
\end{figure}

\begin{figure*}
	\centering
	\subfigure{
	\begin{minipage}[b]{0.2\linewidth}
	\includegraphics[width=1\linewidth]{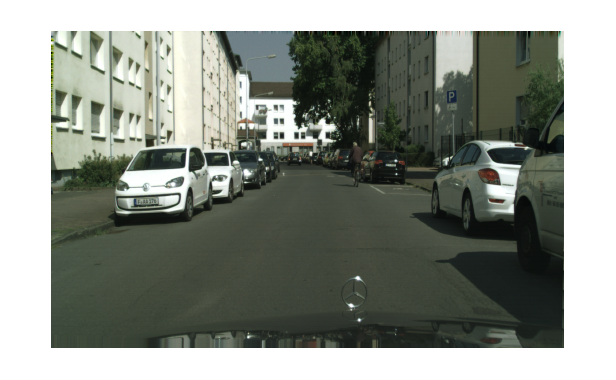}\\[-8pt]
	\includegraphics[width=1\linewidth]{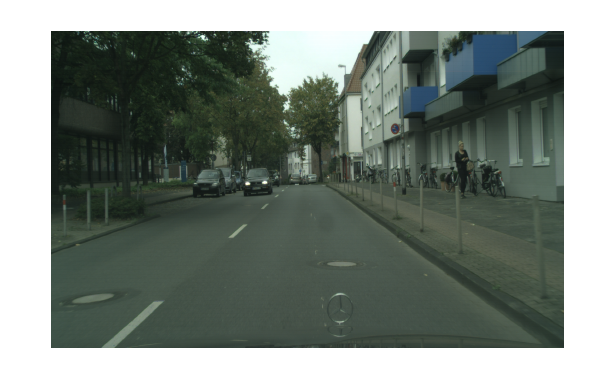}\\[-8pt]
	\includegraphics[width=1\linewidth]{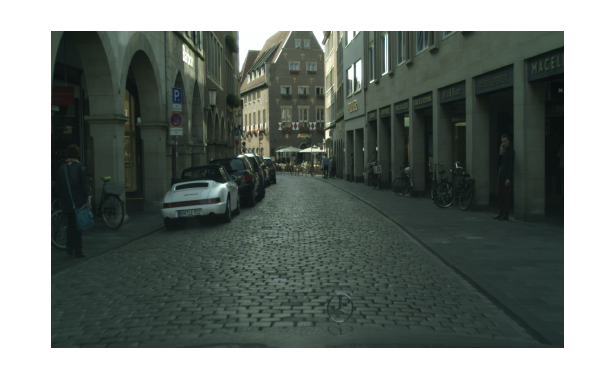}\\[-8pt]
	\includegraphics[width=1\linewidth]{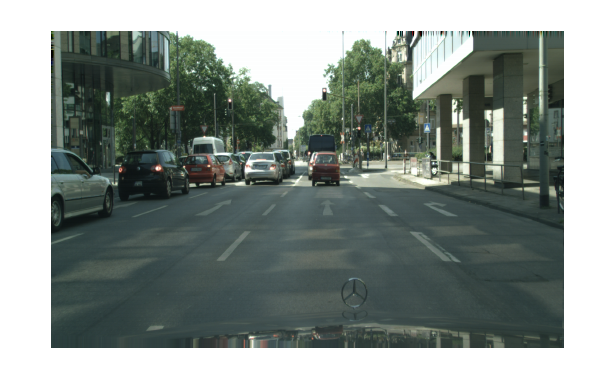}\\[-8pt]
  \includegraphics[width=1\linewidth]{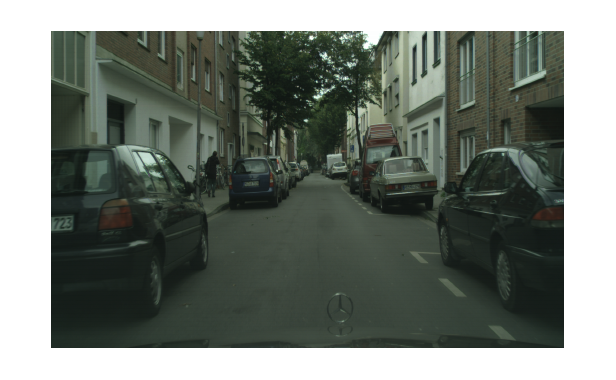}\\[-8pt]
	\centerline{\tiny (a) Images}
	\end{minipage}
	}
  \hspace{-5mm}
	\subfigure{
	\begin{minipage}[b]{0.2\linewidth}
    \includegraphics[width=1\linewidth]{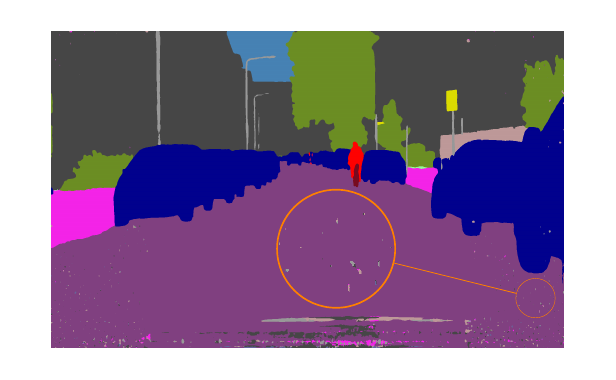}\\[-8pt]
  	\includegraphics[width=1\linewidth]{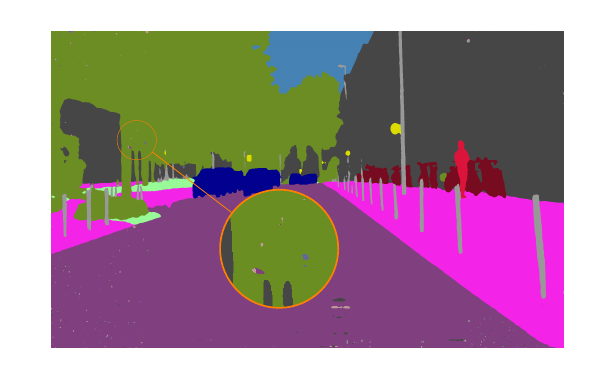}\\[-8pt]
  	\includegraphics[width=1\linewidth]{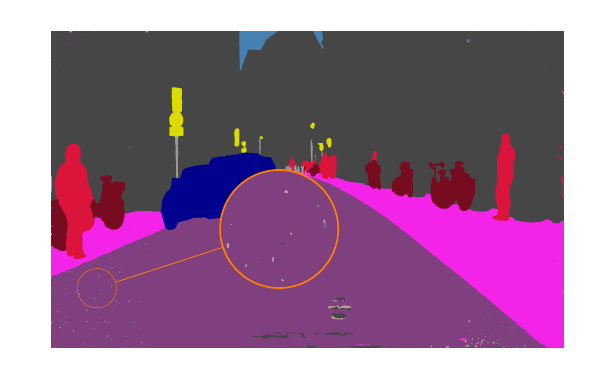}\\[-8pt]
  	\includegraphics[width=1\linewidth]{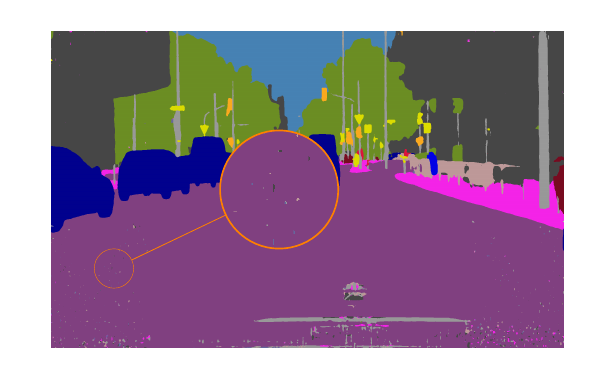}\\[-8pt]
    \includegraphics[width=1\linewidth]{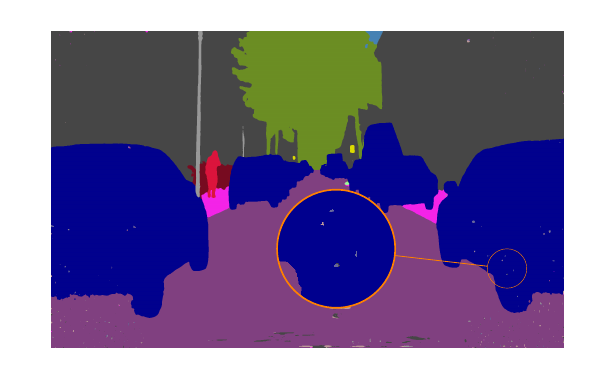}\\[-8pt]
	\centerline{\tiny (b) DeepLabV3+}
	\end{minipage}
	}
  \hspace{-5mm}
	\subfigure{
	\begin{minipage}[b]{0.2\linewidth}
    \includegraphics[width=1\linewidth]{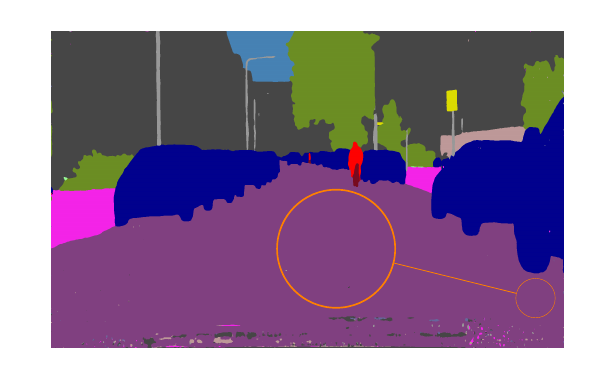}\\[-8pt]
  	\includegraphics[width=1\linewidth]{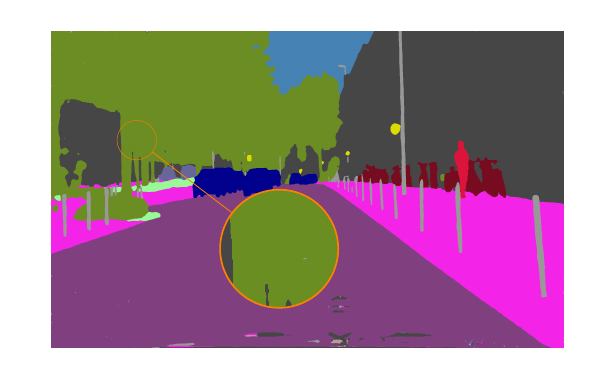}\\[-8pt]
  	\includegraphics[width=1\linewidth]{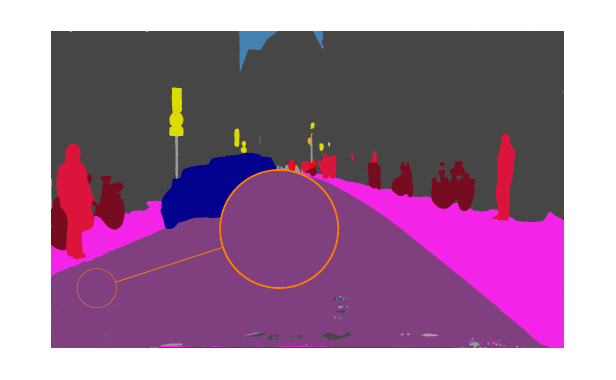}\\[-8pt]
  	\includegraphics[width=1\linewidth]{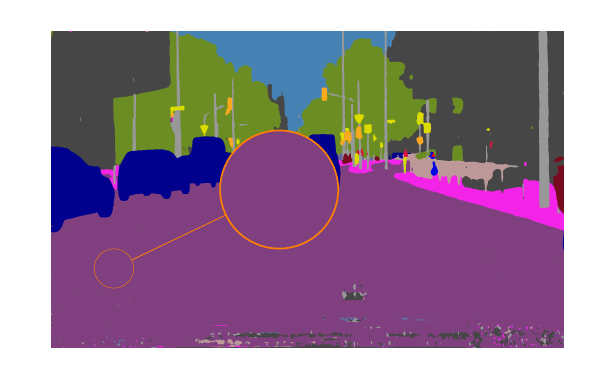}\\[-8pt]
    \includegraphics[width=1\linewidth]{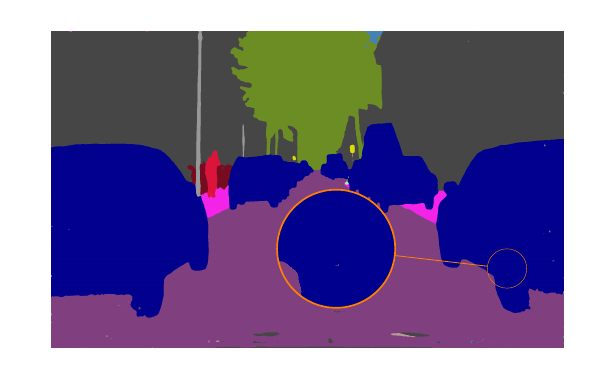}\\[-8pt]
	\centerline{\tiny (c) DeepLabV3+ with diffusion IELs}
	\end{minipage}
	}
  \hspace{-5mm}
	\subfigure{
	\begin{minipage}[b]{0.2\linewidth}
    \includegraphics[width=1\linewidth]{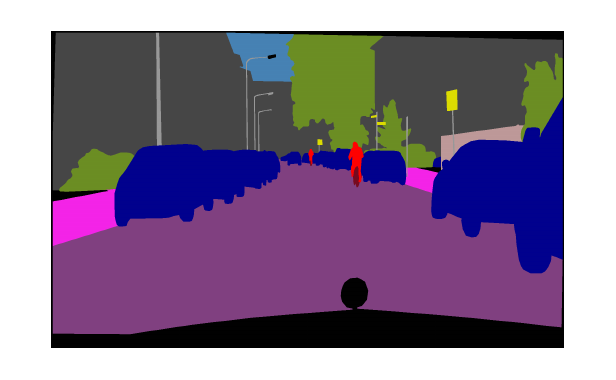}\\[-8pt]
  	\includegraphics[width=1\linewidth]{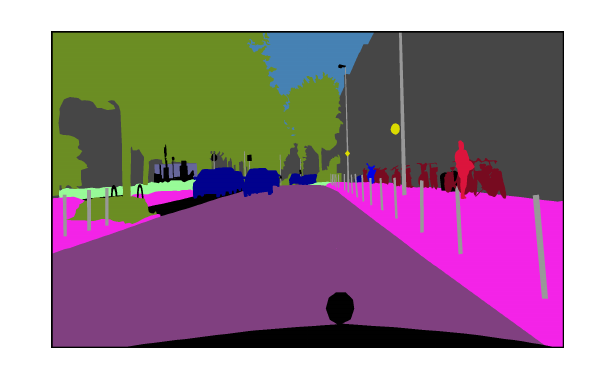}\\[-8pt]
  	\includegraphics[width=1\linewidth]{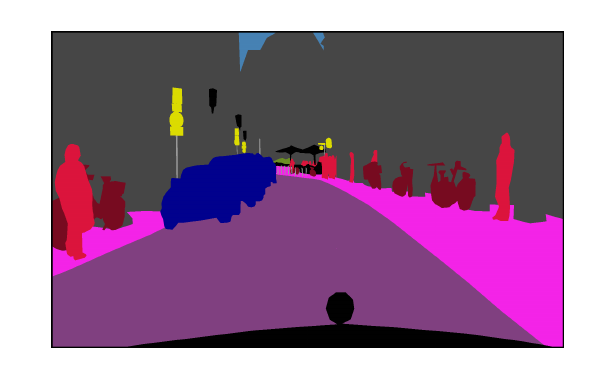}\\[-8pt]
  	\includegraphics[width=1\linewidth]{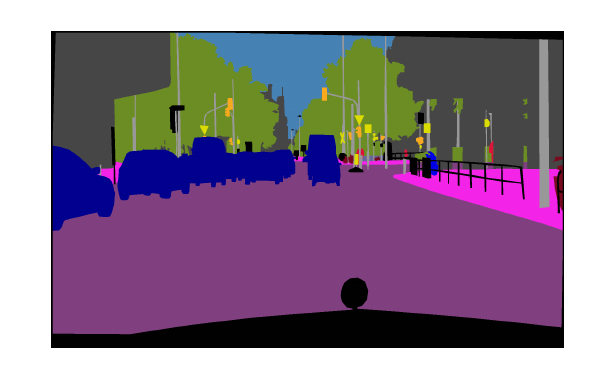}\\[-8pt]
    \includegraphics[width=1\linewidth]{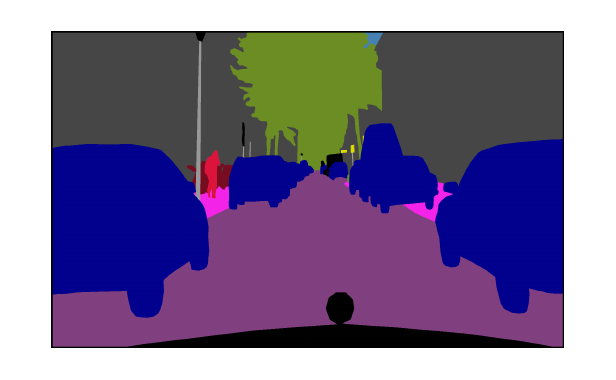}\\[-8pt]
	\centerline{\tiny (d) Ground truth}
	\end{minipage}
	}
	 \caption{Comparison between DeepLabV3+ \cite{chen2018encoder} and DeepLabV3+ with heat-diffusion IELs on Cityscapes dataset with noisy labels. (a) and (d) are the original images and the corresponding ground truth, respectively. (b) and (c) exhibit the segmentation results of the original DeepLabV3+ and DeepLabV3+ with heat-diffusion IELs, respectively. (The noise in (b) is not obvious since the noise size is quite small. Noise can be observed by zooming in on the four corners of these images.)} \label{deeplab_fig}
\end{figure*}

\subsubsection{HRNetV2-W48 with Heat-diffusion IELs on Cityscapes}
Furthermore, we test the HRNetV2-W48 and its IELs counterpart on the Cityscapes dataset. To expedite the training process, a pretrained model is utilized. The results are displayed in Figure~\ref{hrnet}. We use the same evaluation metric, mIoU, for this experiment as well. While HRNetV2-W48 achieves high mIoU compared to DeepLabV3+, its tendency to overfit on noisy labels is more prominent than that of DeepLabV3+. This is due to the fact that HRNetV2-W48 relies heavily on the high-resolution features which are more susceptible to noise. Nevertheless, our heat-diffusion IELs significantly mitigate this overfitting, as evident from the results. Additionally, Figure~\ref{hrnet_fig} provides detailed segmentation maps, which further demonstrate the effectiveness of our heat-diffusion IELs in handling noise in labels.

\begin{figure}[t!]
  \centering
  \includegraphics[width=8cm]{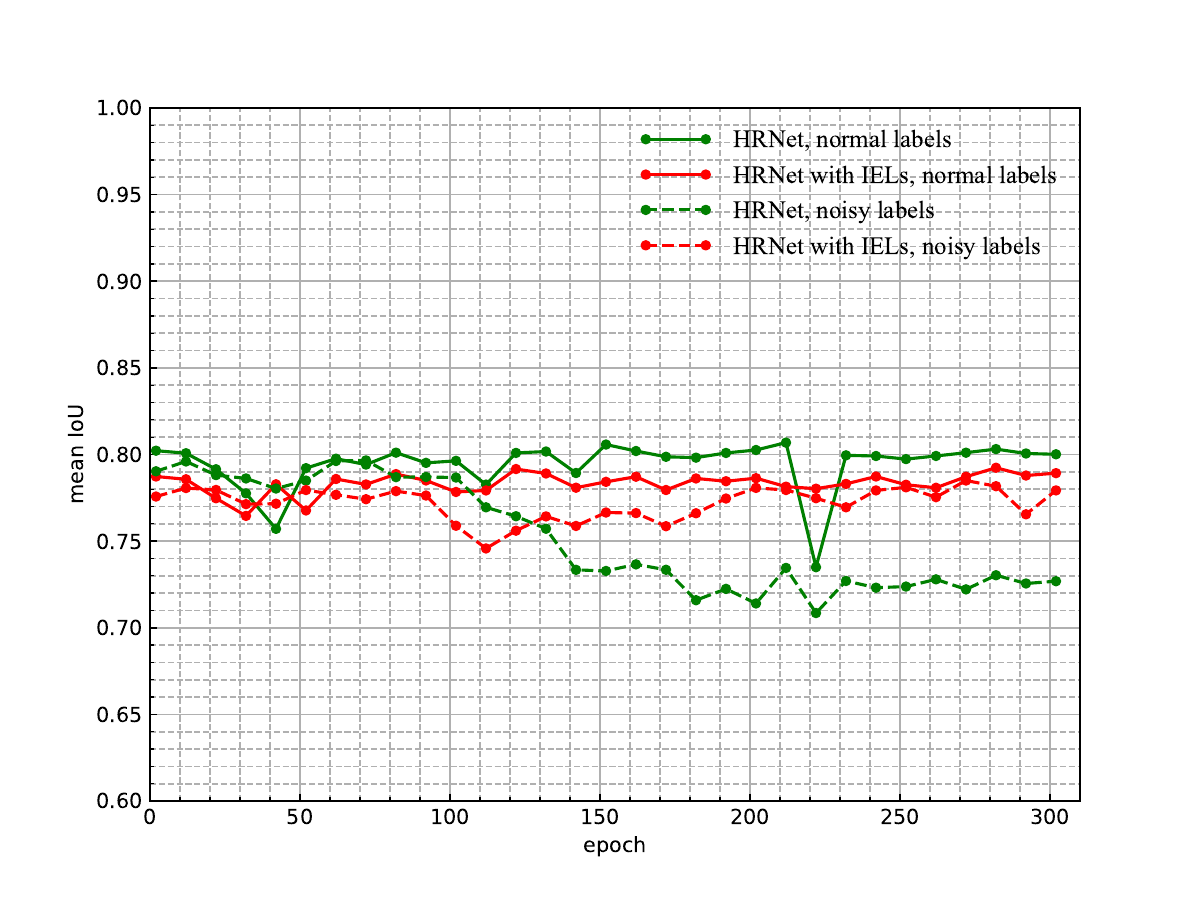}
 \caption{Mean IoU on the validation set of Cityscapes dataset based on HRNetV2-W48 \cite{sun2019high} and HRNetV2-W48 with heat-diffusion IELs. The results of HRNetV2-W48 and HRNetV2-W48 with heat-diffusion IELs are highlighted by green lines and red lines, respectively. The results on normal labels and noisy labels are marked by solid lines and dotted lines, respectively.} \label{hrnet}
\end{figure}

\begin{figure*}
	\centering
	\subfigure{
	\begin{minipage}[b]{0.2\linewidth}
	\includegraphics[width=1\linewidth]{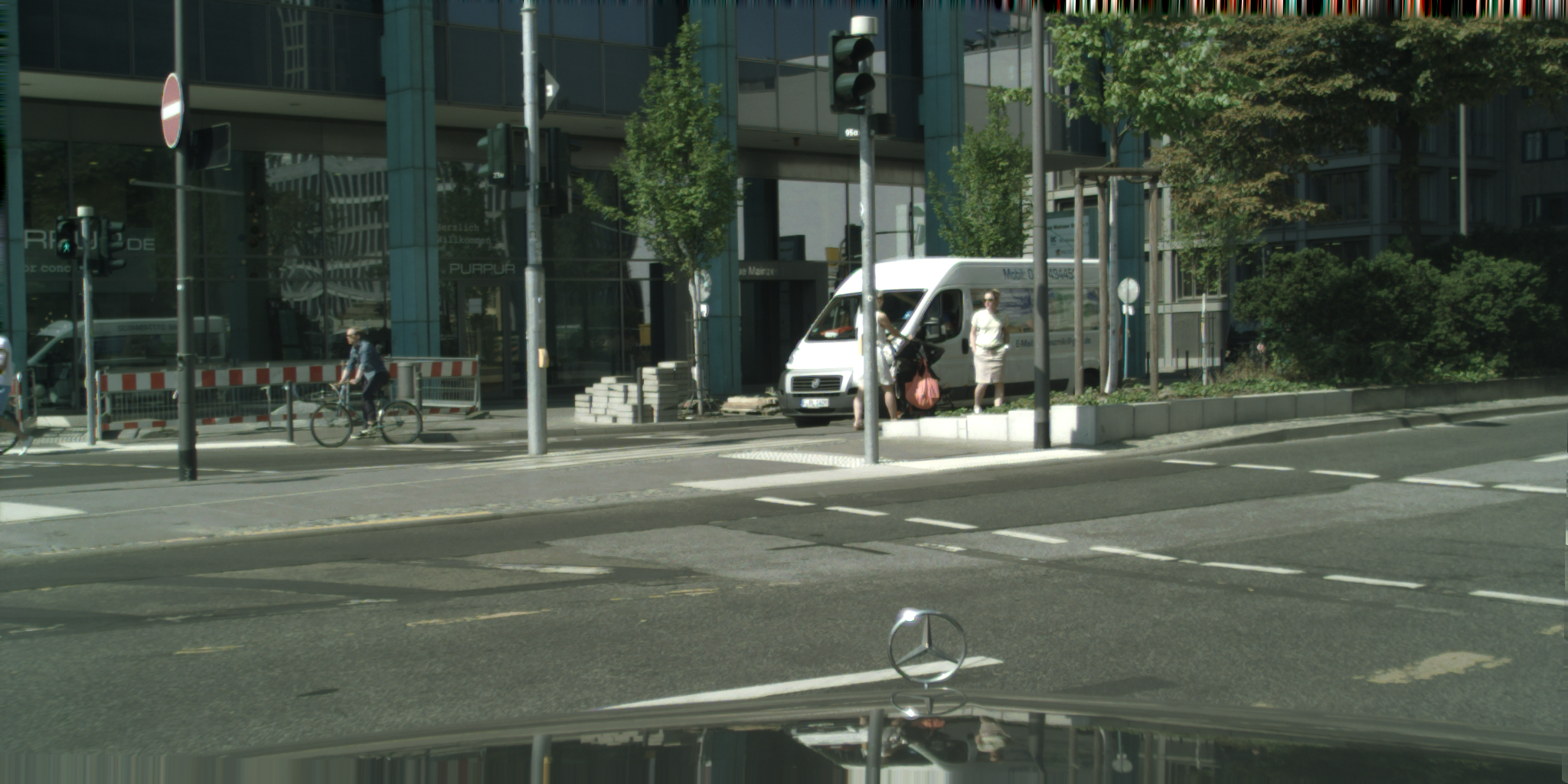}
	\includegraphics[width=1\linewidth]{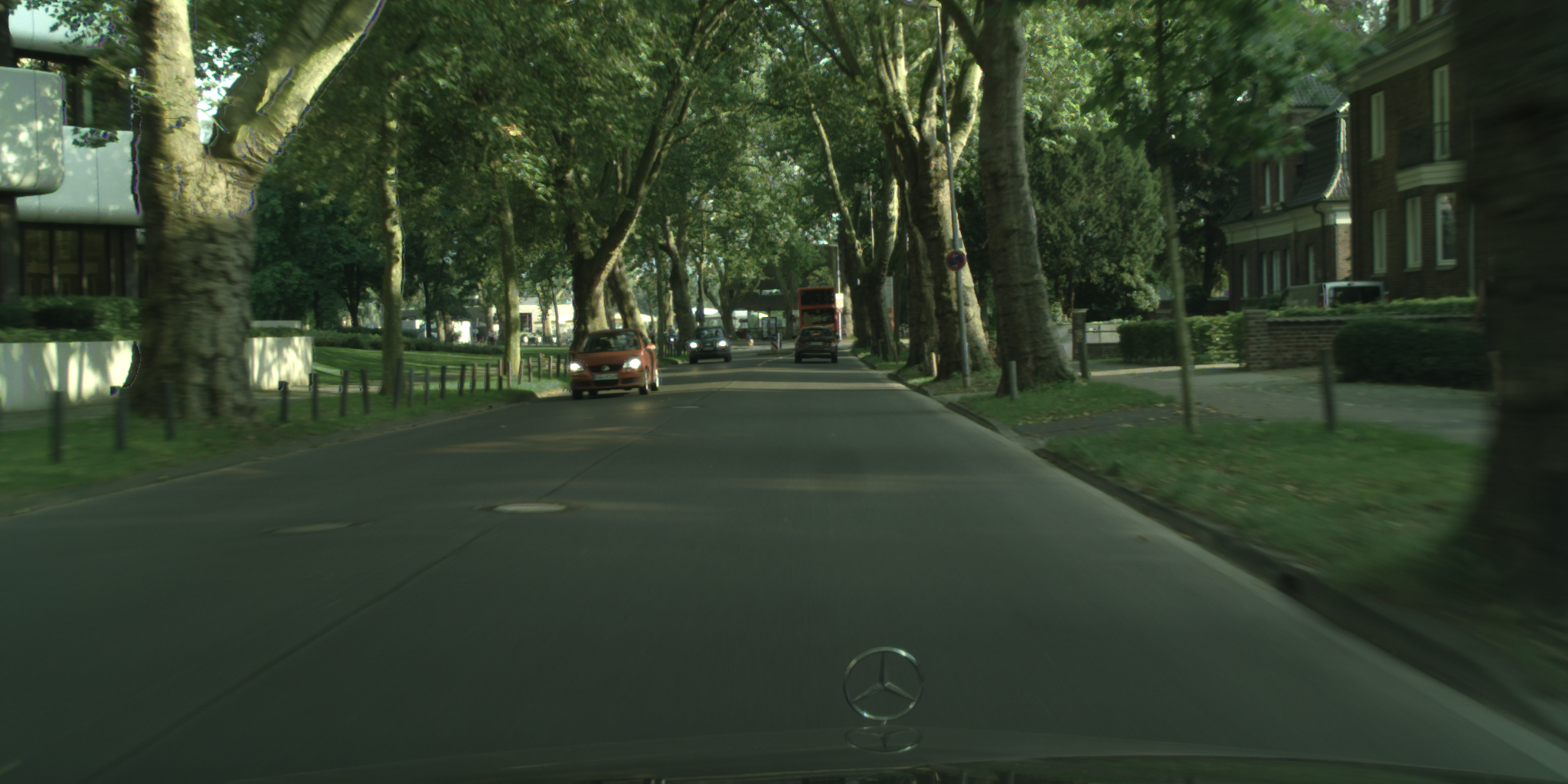}
	\includegraphics[width=1\linewidth]{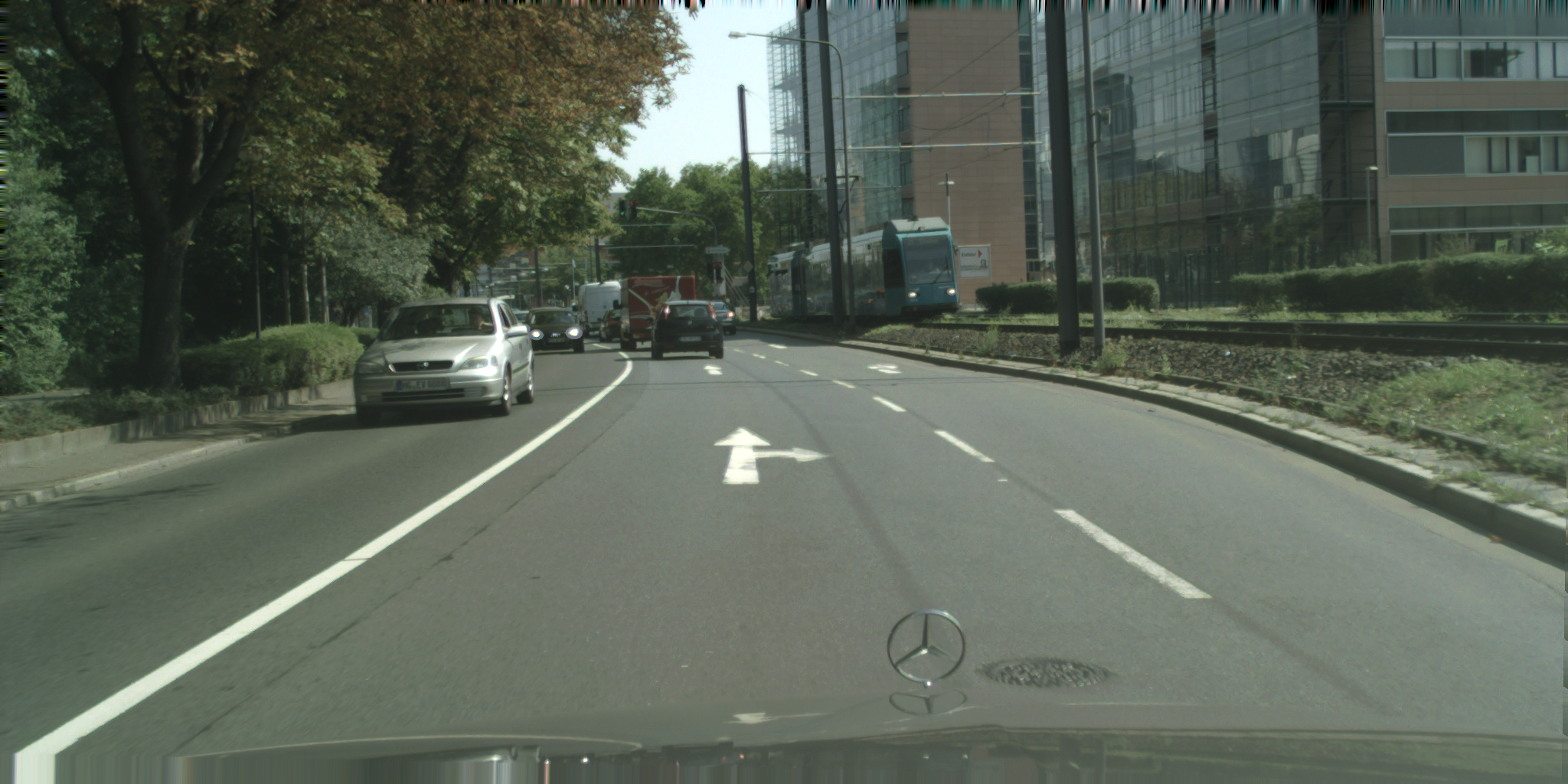}
	\includegraphics[width=1\linewidth]{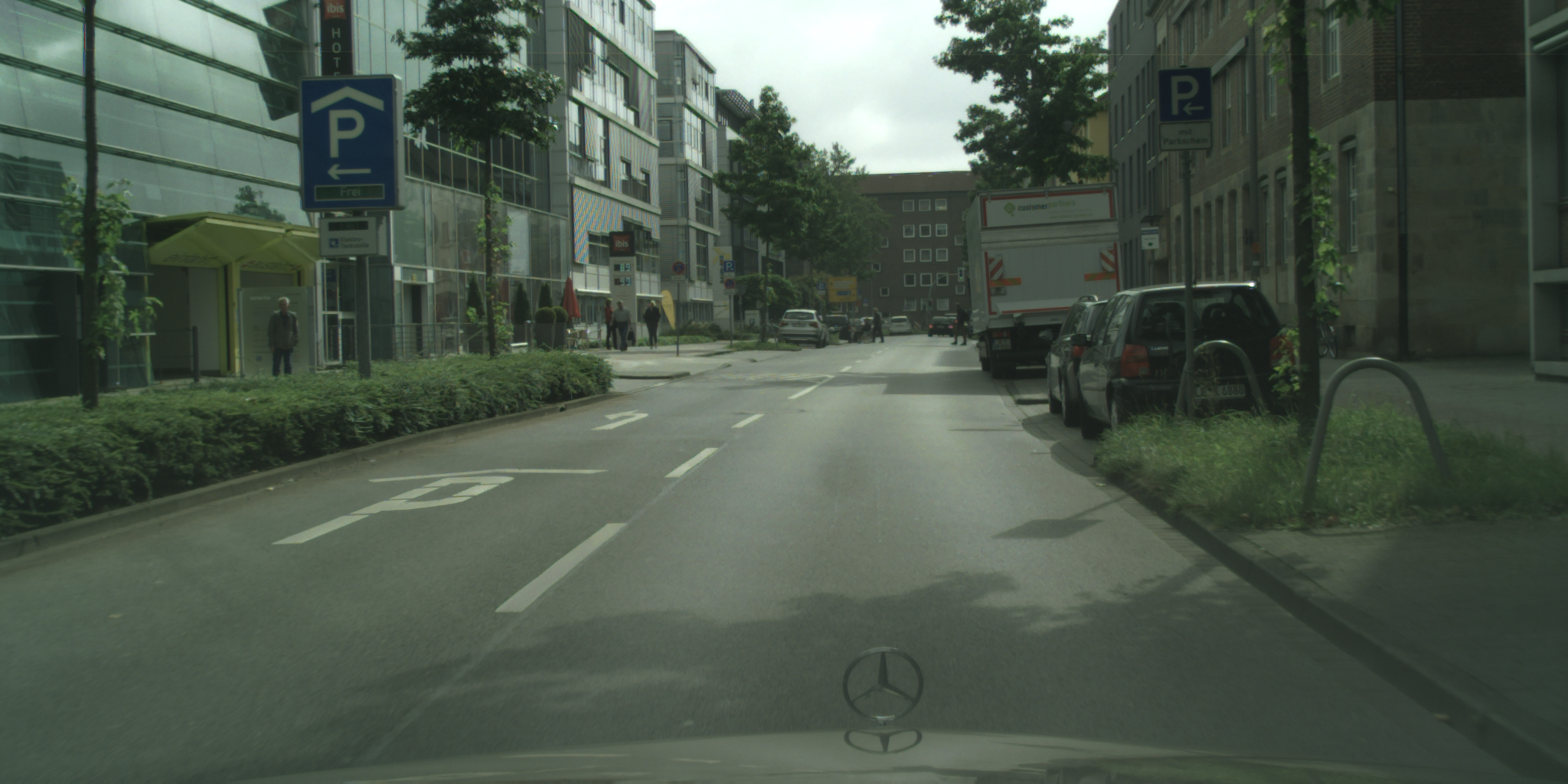}
  \includegraphics[width=1\linewidth]{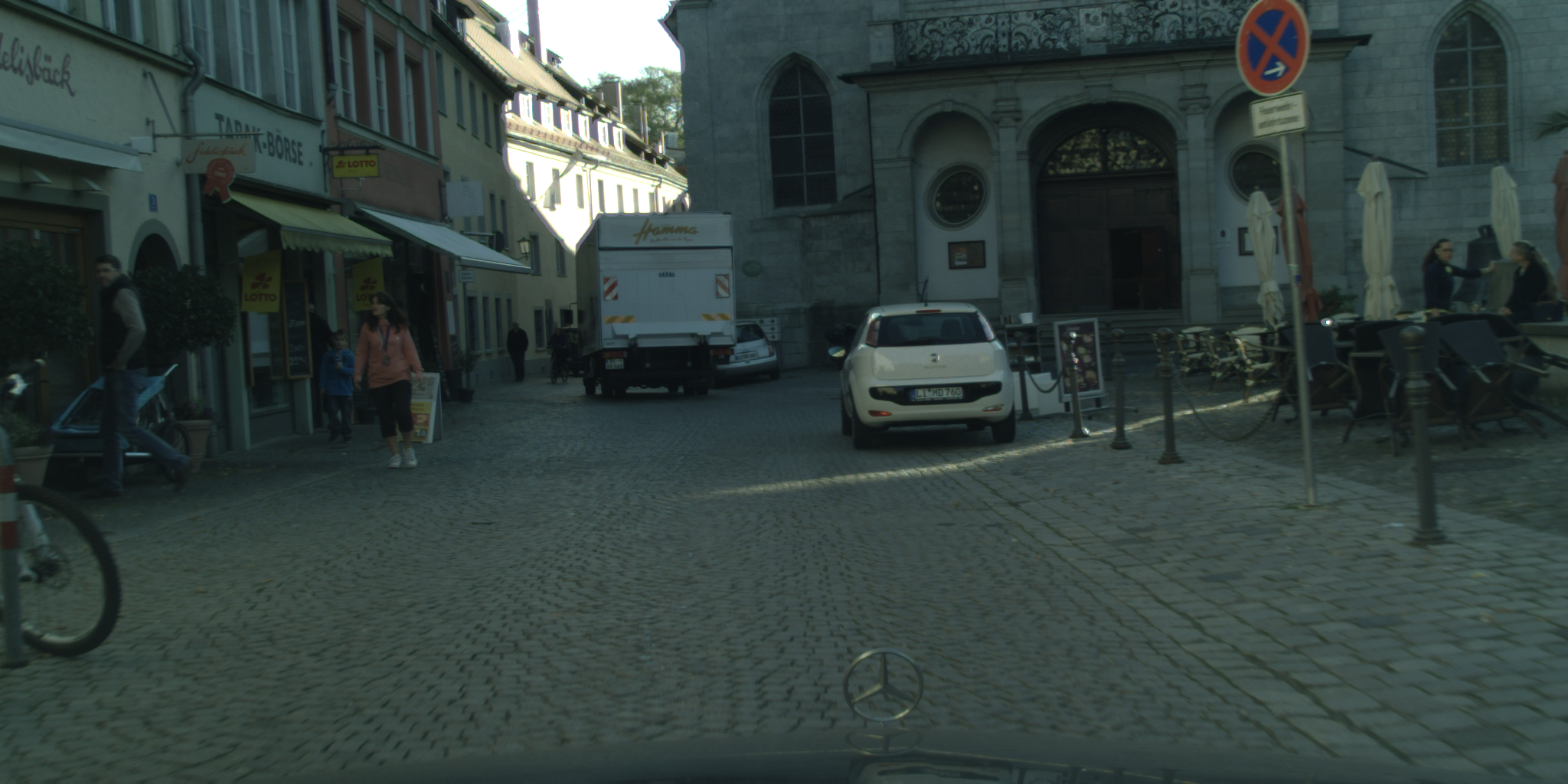}
	\centerline{\tiny (a) Images}
	\end{minipage}
	}
	\subfigure{
	\begin{minipage}[b]{0.2\linewidth}
    \includegraphics[width=1\linewidth]{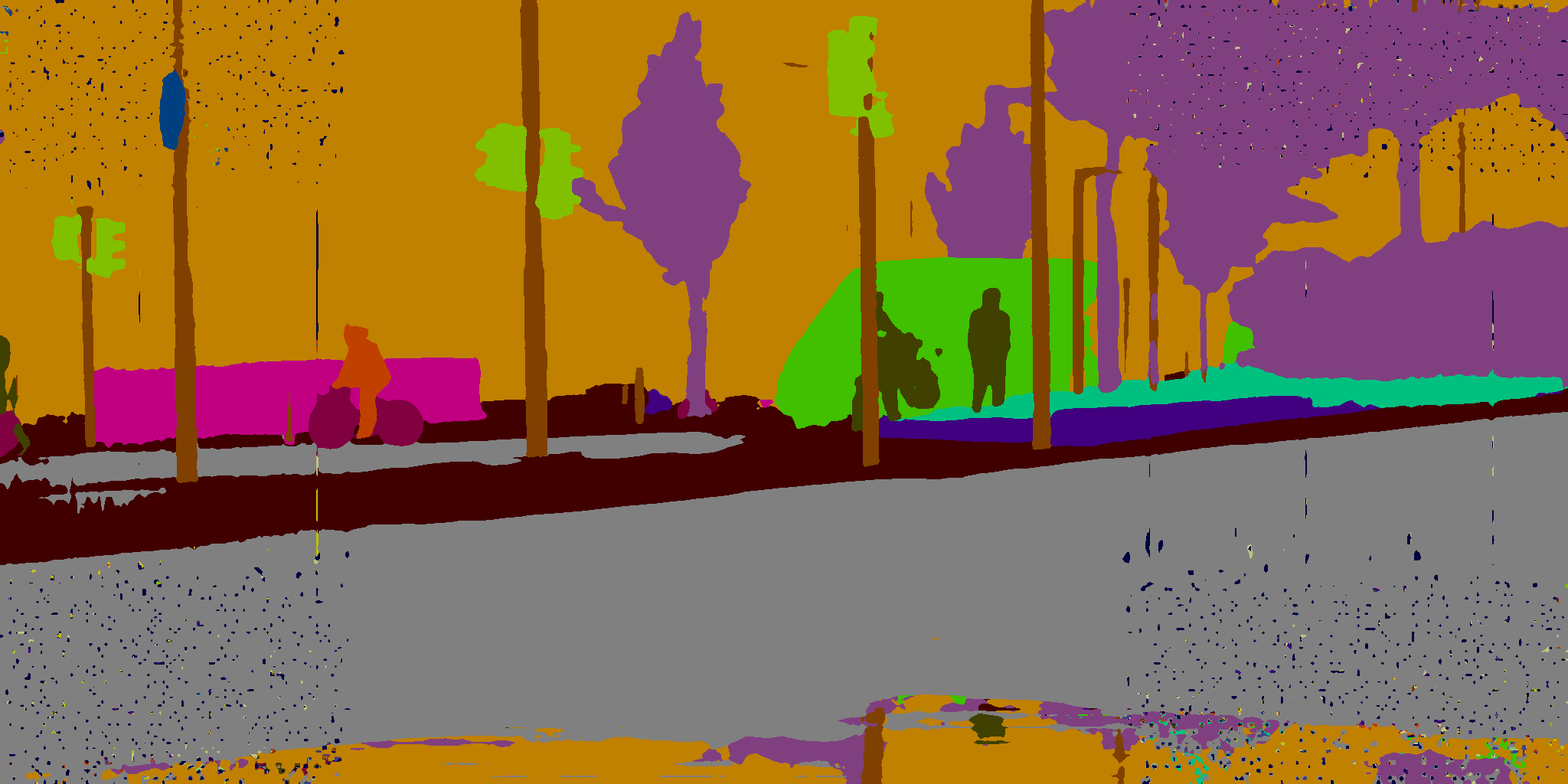}
  	\includegraphics[width=1\linewidth]{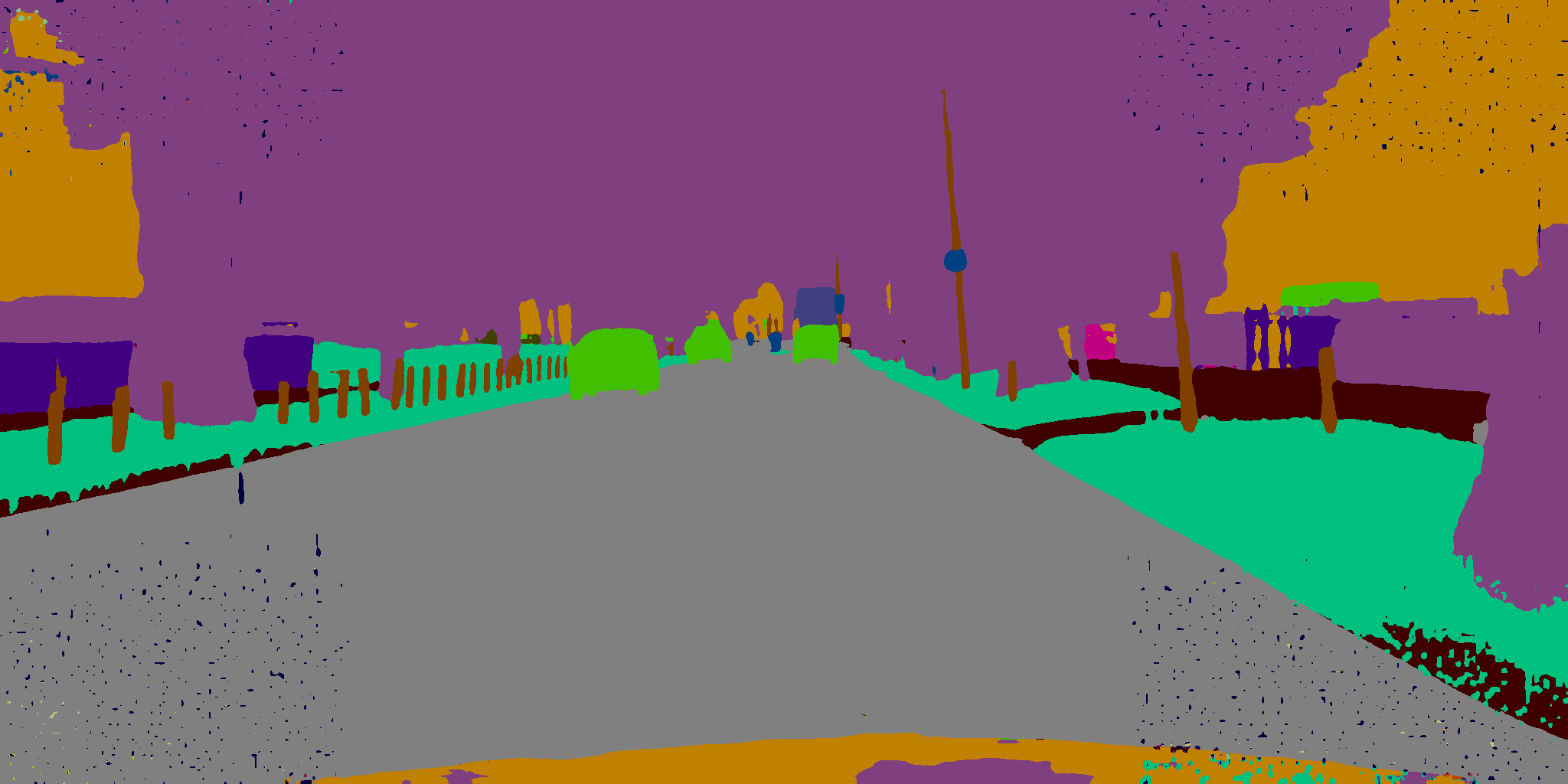}
  	\includegraphics[width=1\linewidth]{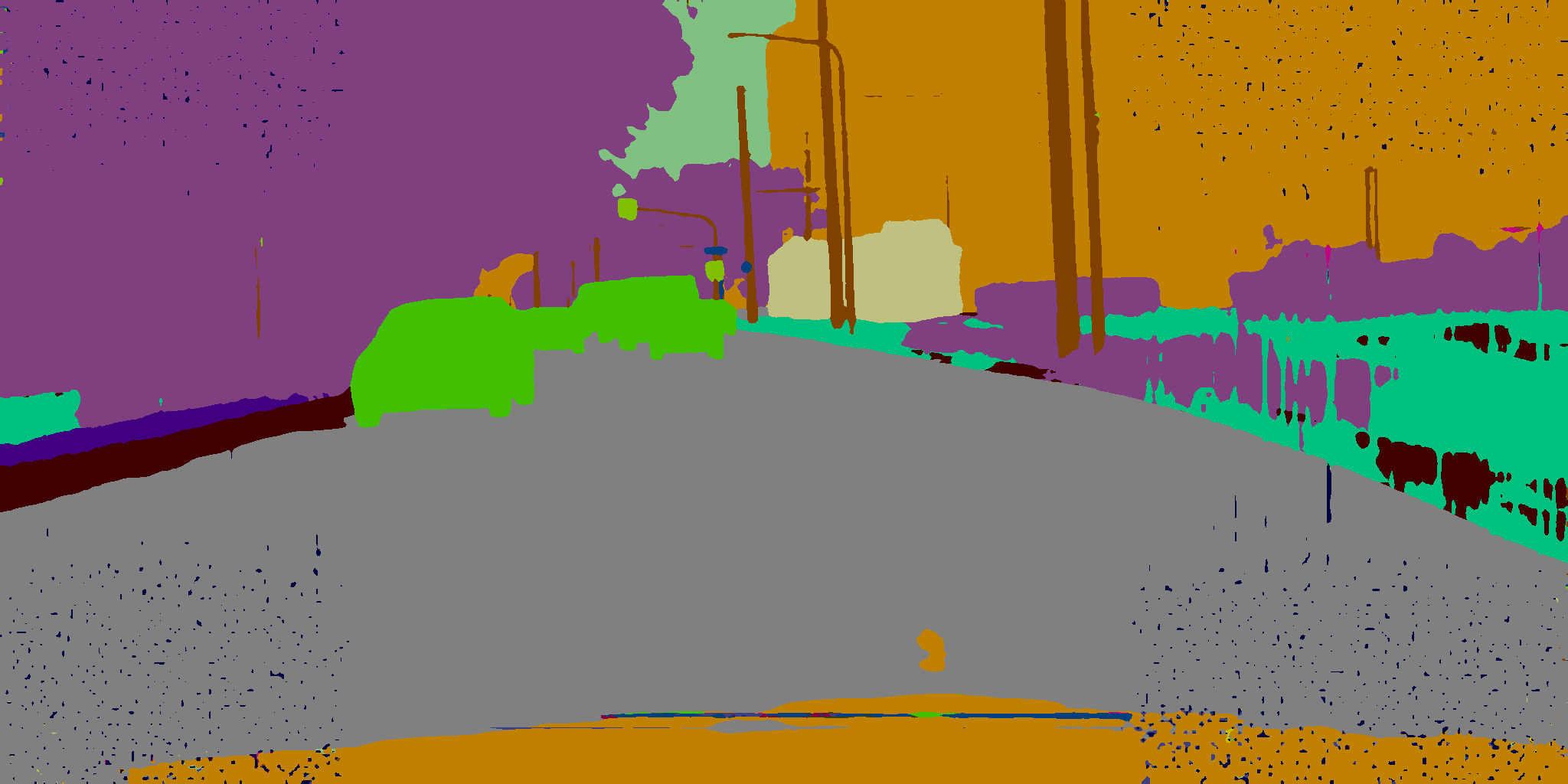}
  	\includegraphics[width=1\linewidth]{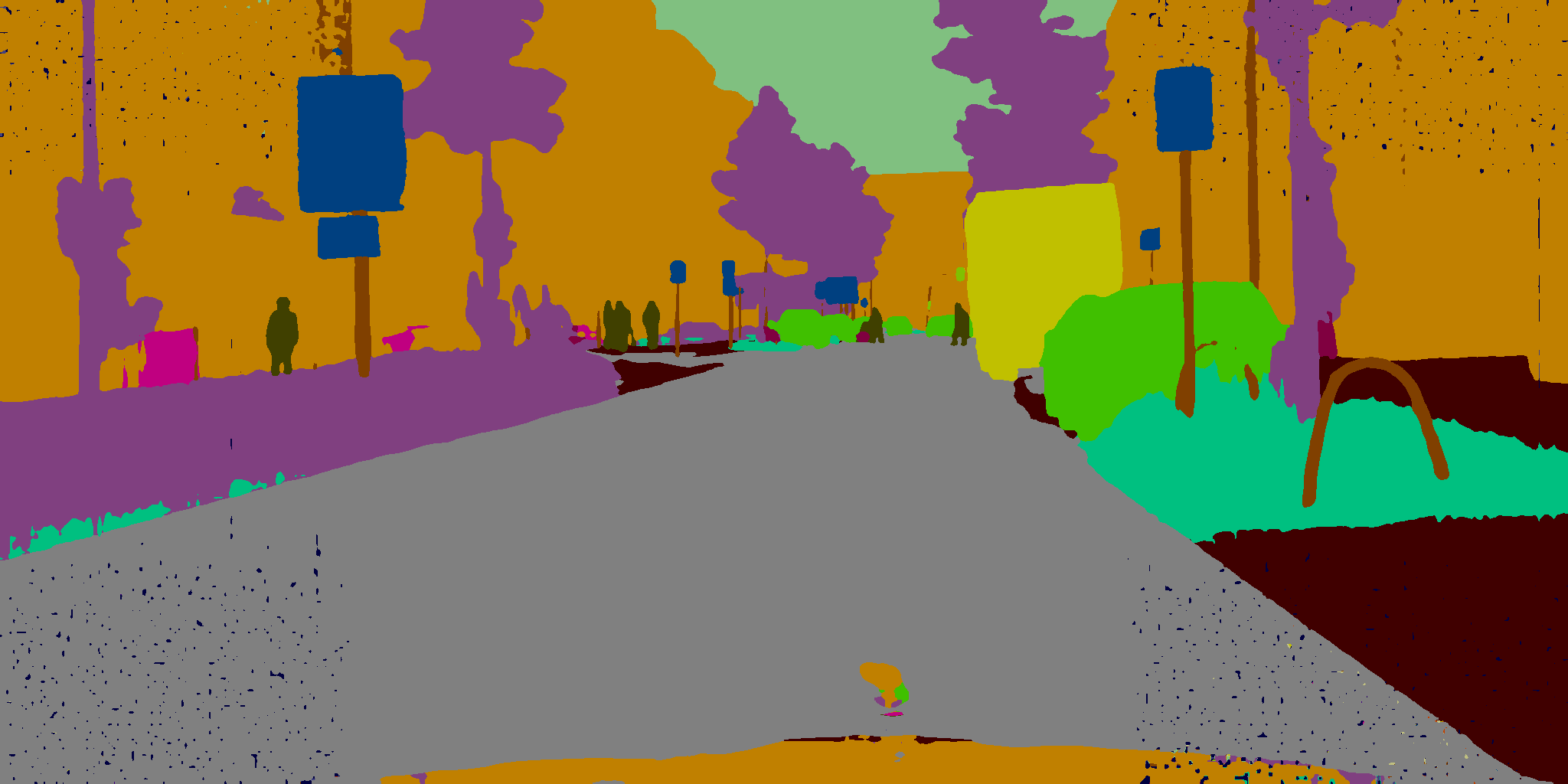}
    \includegraphics[width=1\linewidth]{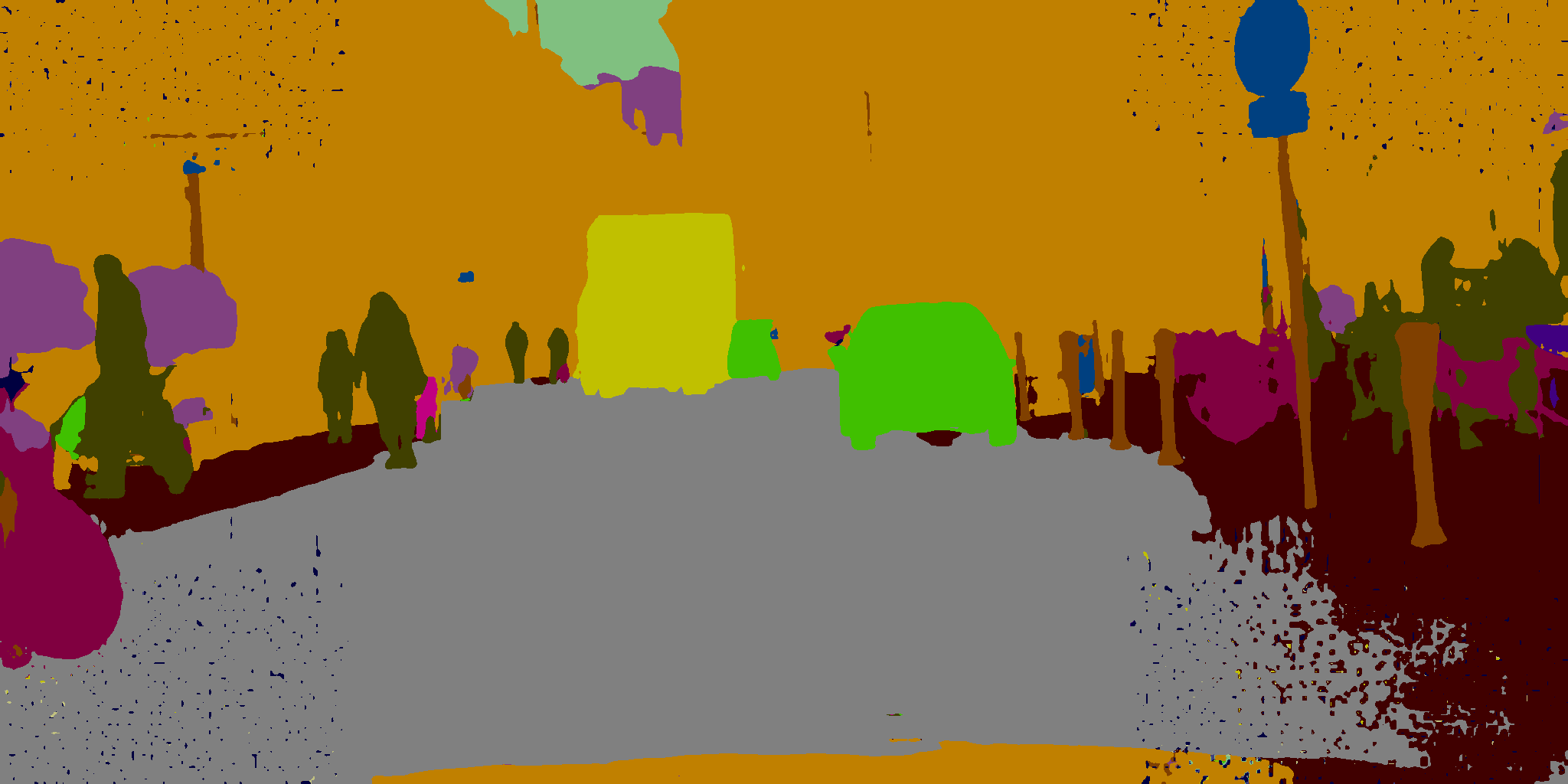}
	\centerline{\tiny (b) HRNetV2-W48}
	\end{minipage}
	}
	\subfigure{
	\begin{minipage}[b]{0.2\linewidth}
    \includegraphics[width=1\linewidth]{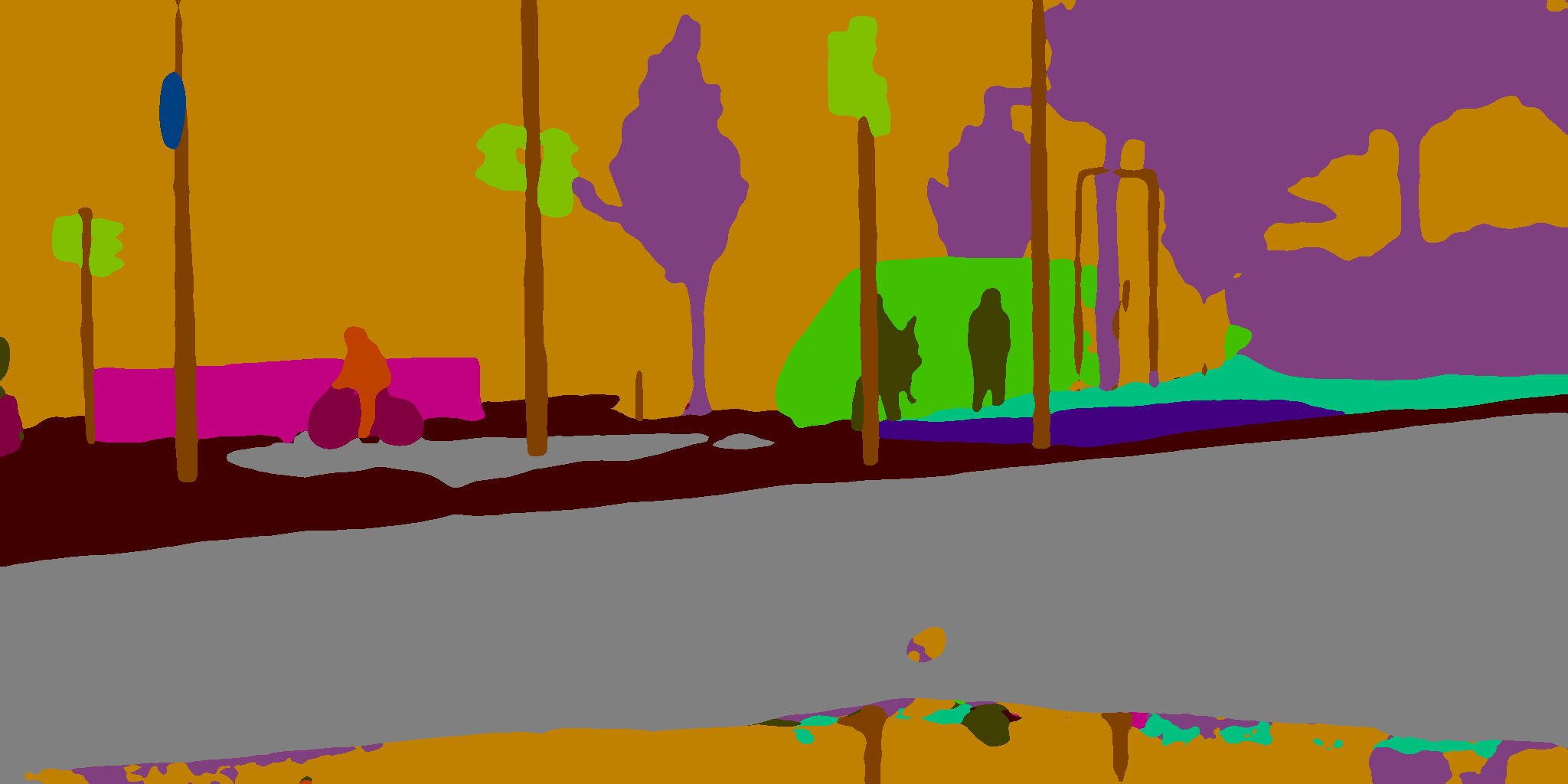}
  	\includegraphics[width=1\linewidth]{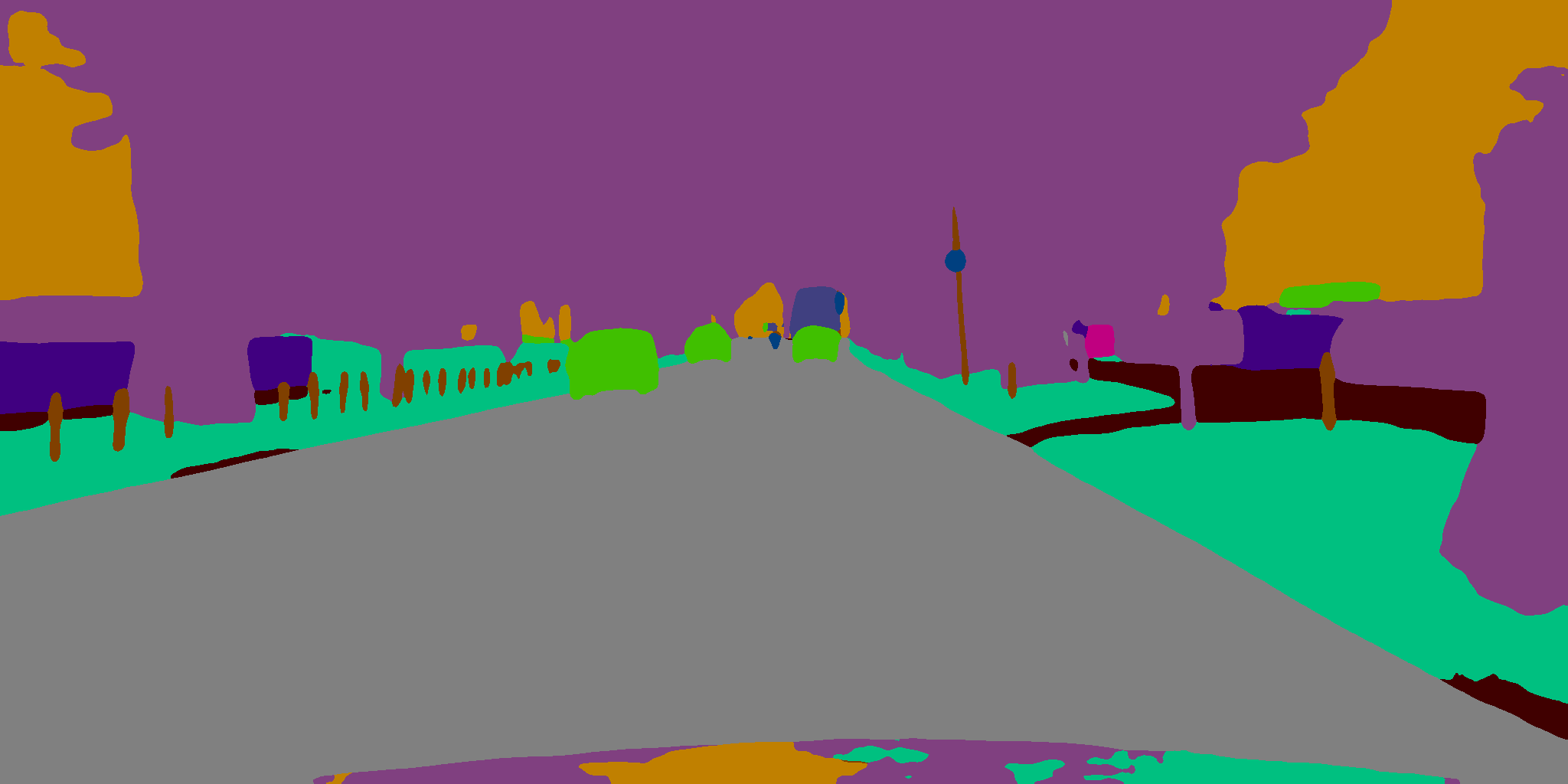}
  	\includegraphics[width=1\linewidth]{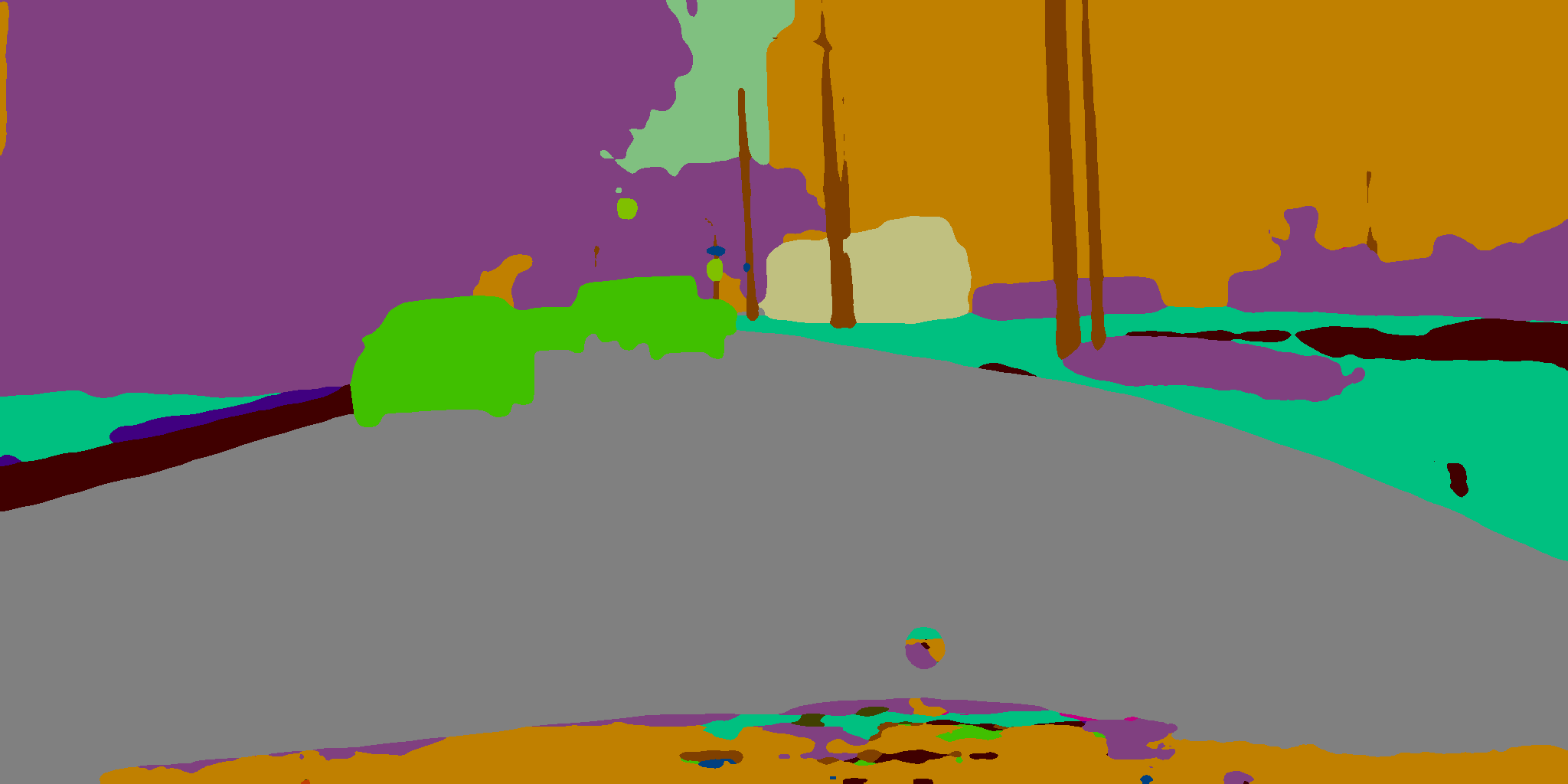}
  	\includegraphics[width=1\linewidth]{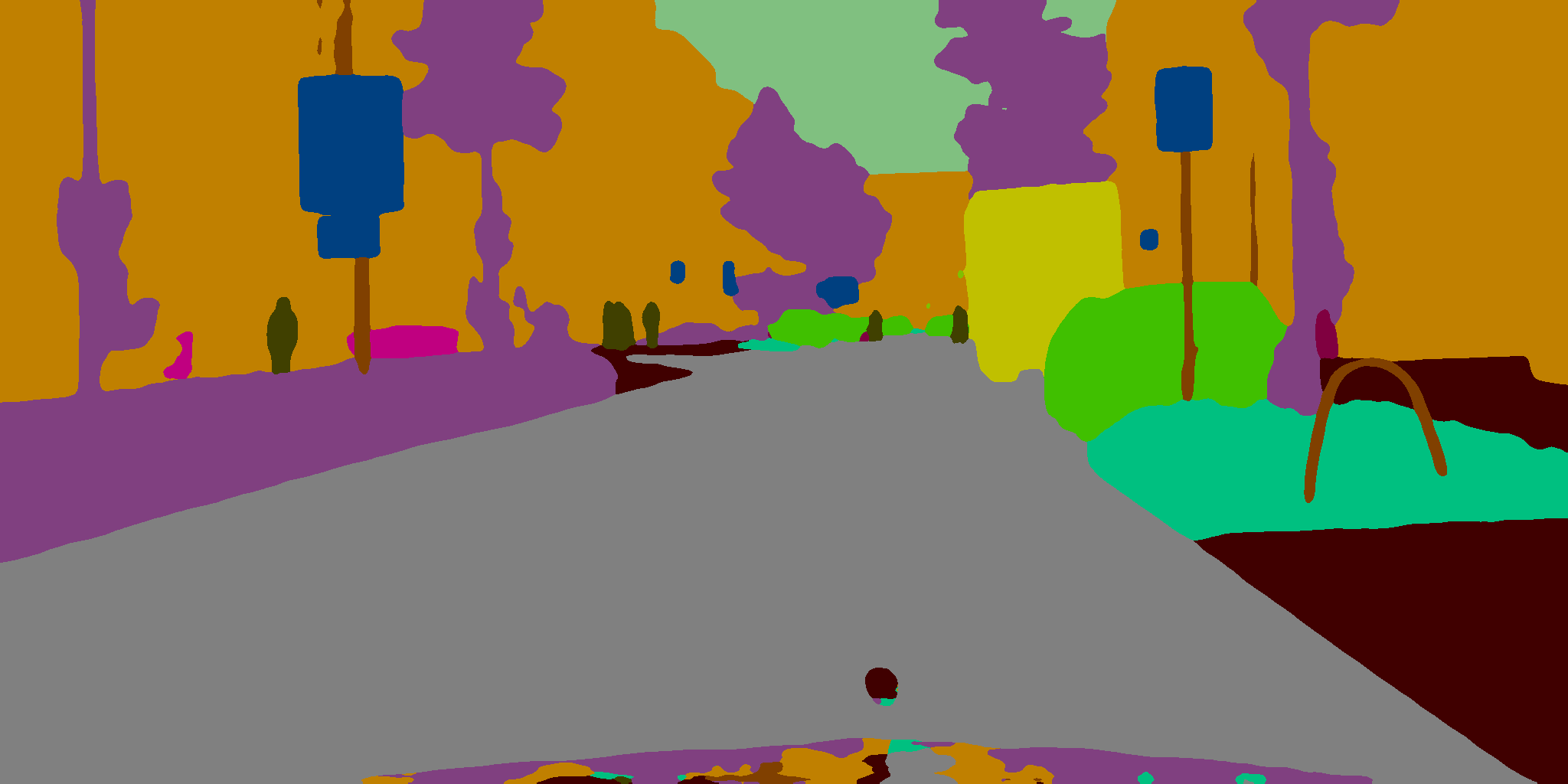}
    \includegraphics[width=1\linewidth]{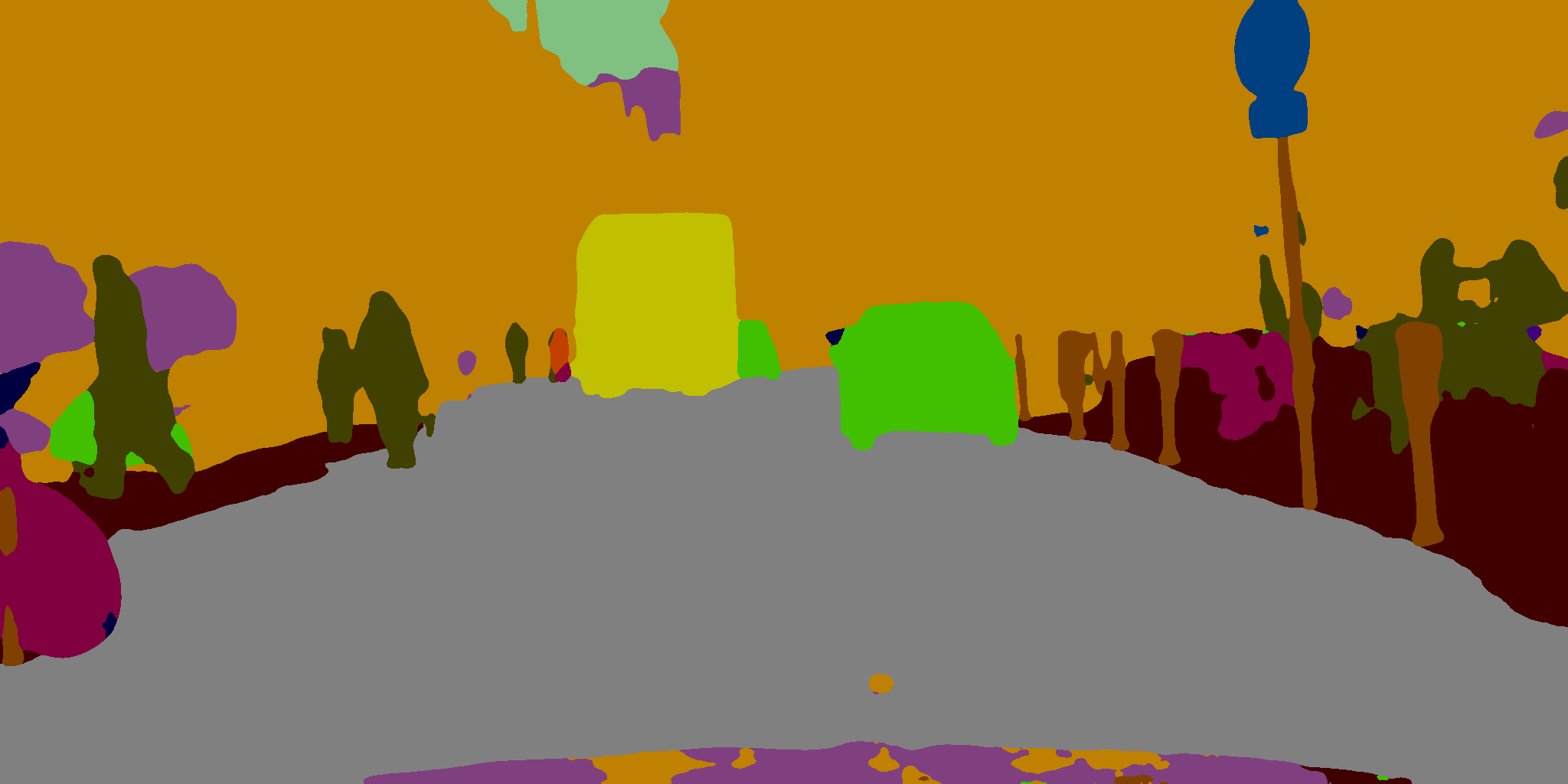}
	\centerline{\tiny (c) HRNetV2-W48 with diffusion IELs}
	\end{minipage}
	}
	\subfigure{
	\begin{minipage}[b]{0.2\linewidth}
    \includegraphics[width=1\linewidth]{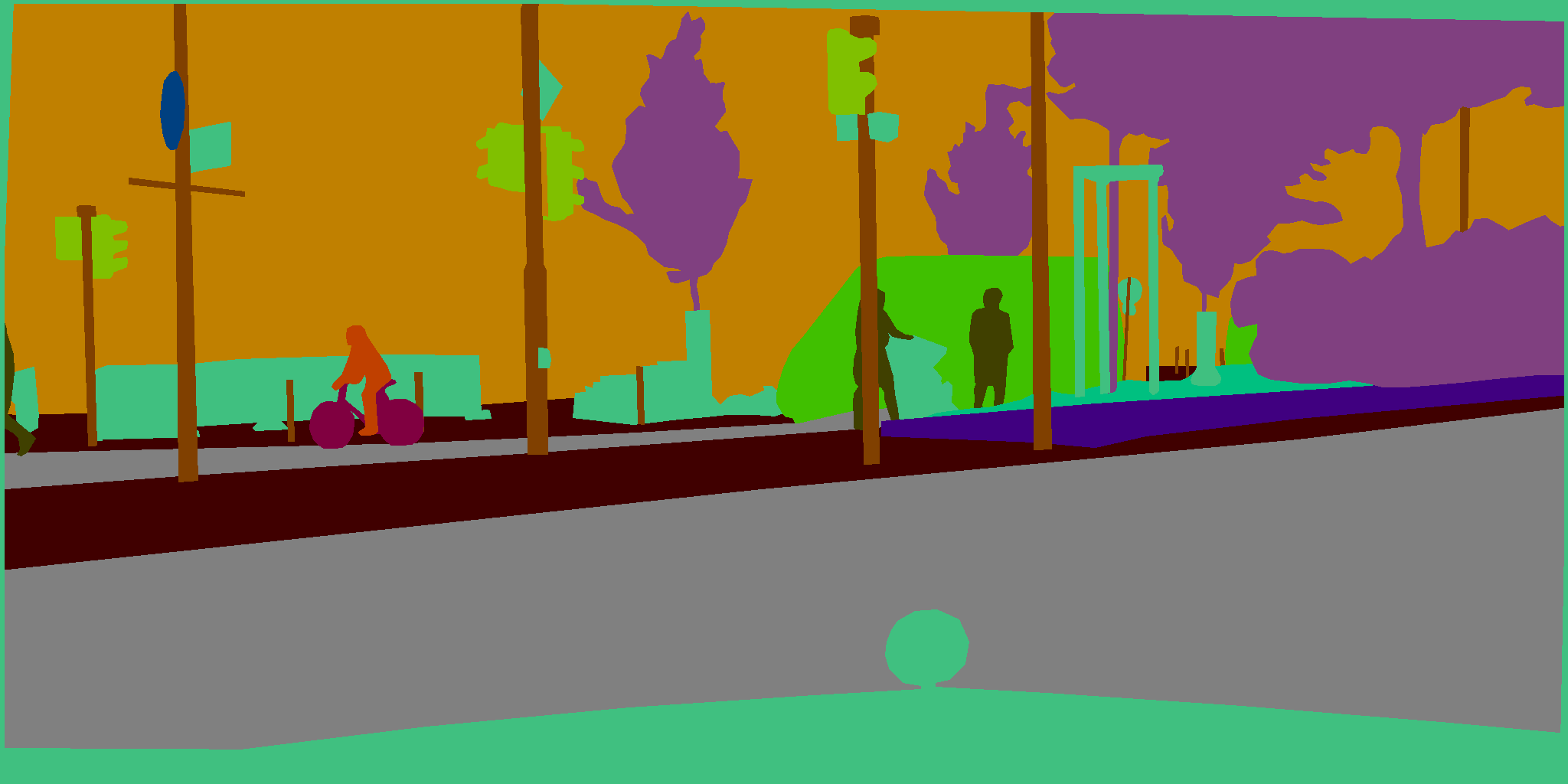}
  	\includegraphics[width=1\linewidth]{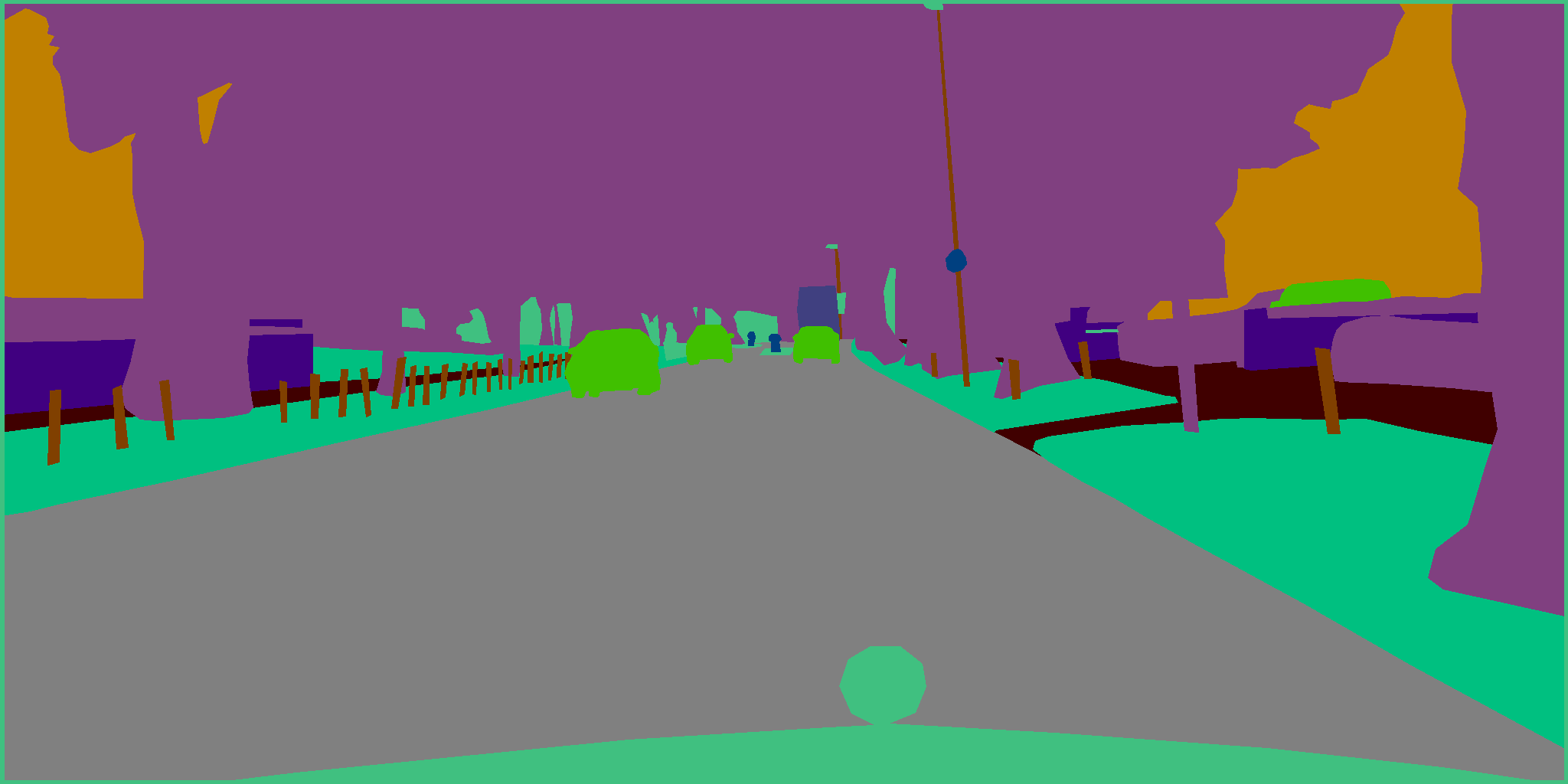}
  	\includegraphics[width=1\linewidth]{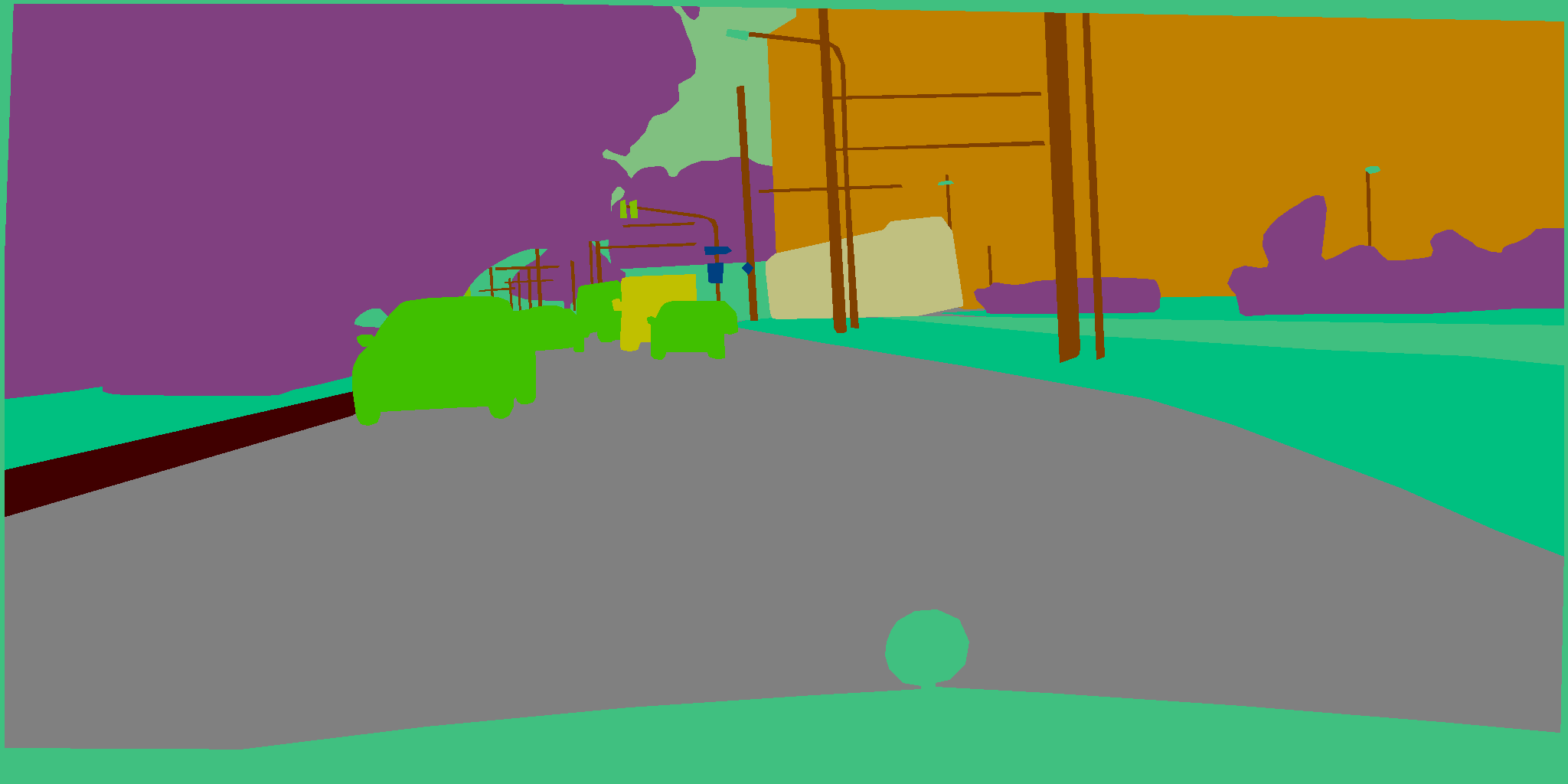}
  	\includegraphics[width=1\linewidth]{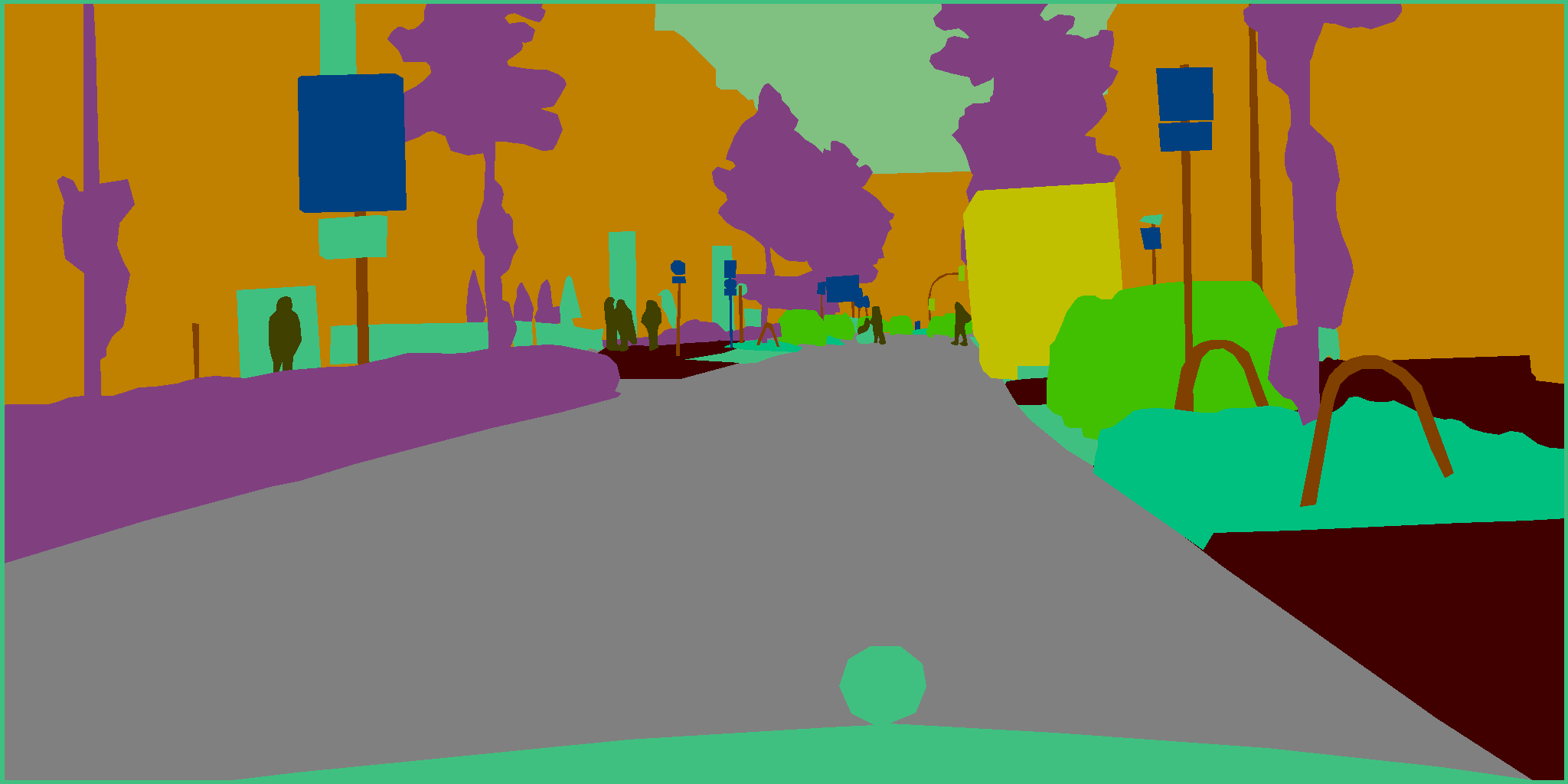}
    \includegraphics[width=1\linewidth]{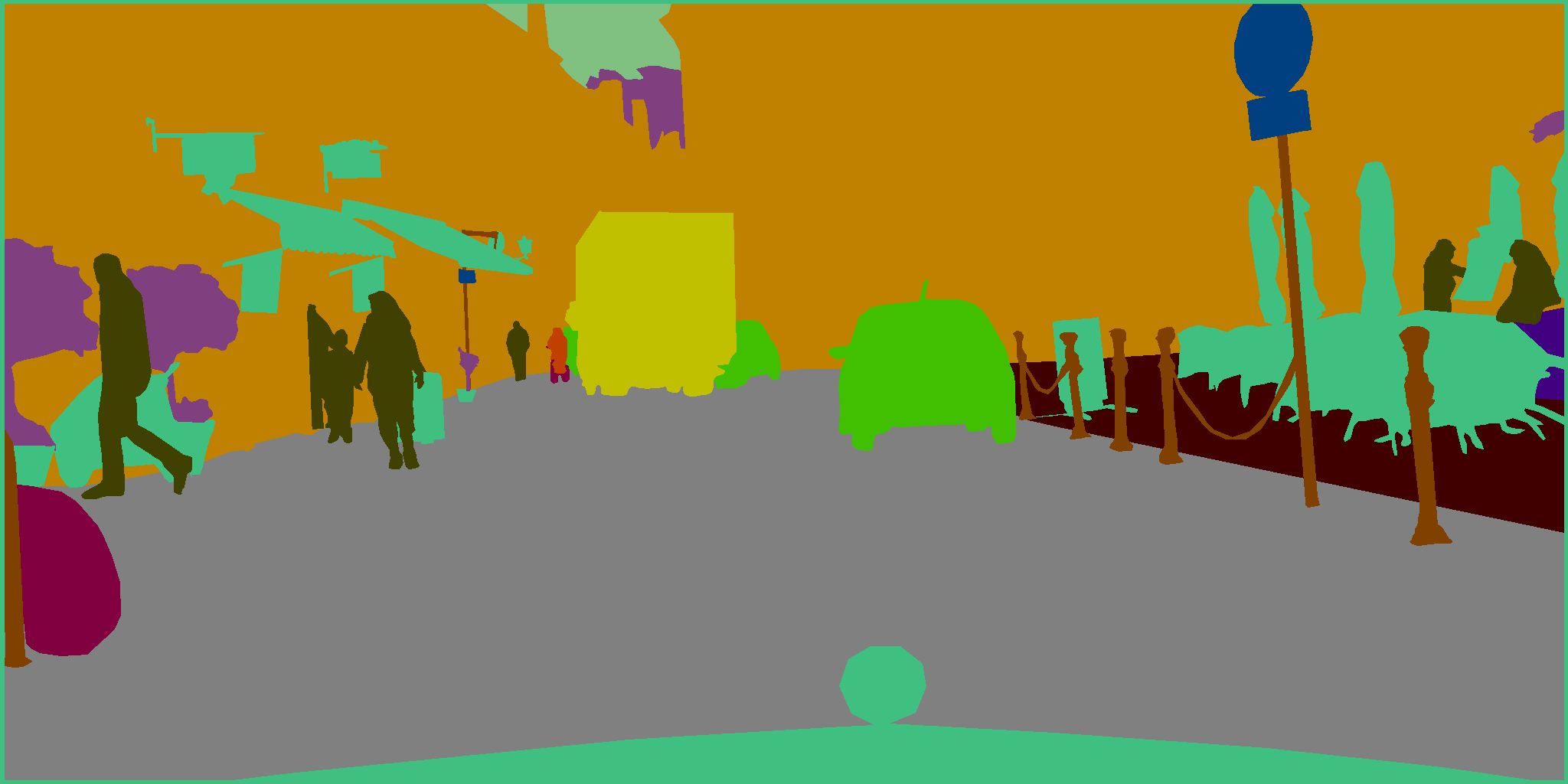}
	\centerline{\tiny (d) Ground truth}
	\end{minipage}
	}
	 \caption{Comparison between HRNetV2-W48 \cite{sun2019high} and HRNetV2-W48 with heat-diffusion IELs on Cityscapes dataset with noisy labels. (a) and (d) are the original images and the corresponding ground truth, respectively. (b) and (c) exhibit the segmentation results of the original HRNetV2-W48 and HRNetV2-W48 with heat-diffusion IELs, respectively.} \label{hrnet_fig}
\end{figure*}
\subsubsection{Unet with Heat-diffusion IELs on DRIVE dataset}
In addition to addressing the overfitting issue caused by artificially introduced synthetic noise, we also apply heat-diffusion IELs to tackle overfitting on datasets with clean labels. We utilize the Digital Retinal Images for Vessel Extraction (DRIVE) dataset \cite{1282003} for testing, which contains 20 images for training and 20 images for test. In our experiment, we employ an augmented training dataset of 80 images for training and the test images for validation. First, we test the original Unet on the DRIVE dataset. The dice scores on the validation set are recorded in Figure \ref{drive_dice}. And several segmentation results are displayed in Figure \ref{drive_unet}, which reflects the original Unet overfits the dense, small branches of blood vessels in the eyeball. Next, we test the Unet with heat-diffusion IELs of varying numbers of layers. The dice scores on the validation set are recorded in Figure \ref{drive_dice} and several corresponding segmentation results are showed in Figure \ref{drive_unet_iels}. Figure \ref{drive_unet_iels} demonstrates that as the number of IELs increases, overfitting on the dense small branches of blood vessels is alleviated. The best dice score is achieved with 10 IELs, as indicated in Figure \ref{drive_dice}. In our experiment, we only test the heat-diffusion IELs, which can affect all blood vessels. Designing other IELs that specifically focus on small branches could potentially lead to even better results for improving the dice scores.

\begin{figure}[t!]
  \centering
  \includegraphics[width=8cm]{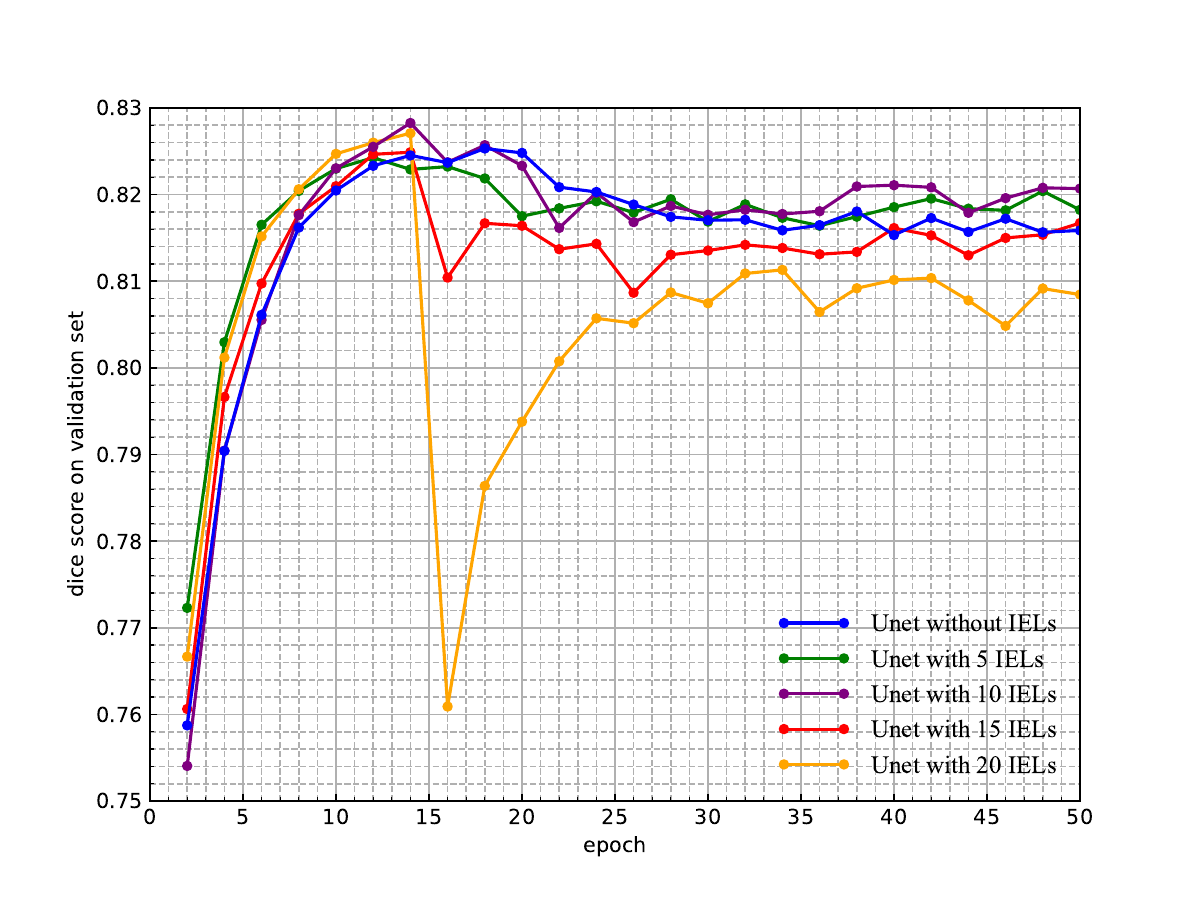}
 \caption{Dice scores on the validation set of DRIVE dataset based on Unet and Unet with heat-diffusion IELs of different numbers. The results of them are highlighted by lines of different colors.} \label{drive_dice}
\end{figure}

\begin{figure*}
  \label{drive_unet}
	\centering
	\subfigure{
	\begin{minipage}[b]{0.18\linewidth}
  \includegraphics[width=1.1\linewidth]{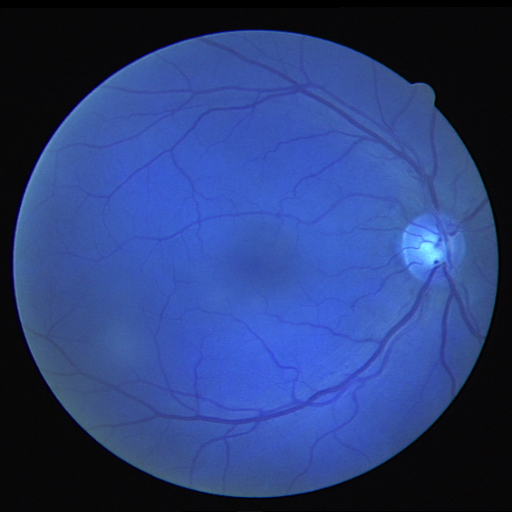}
	\includegraphics[width=1.1\linewidth]{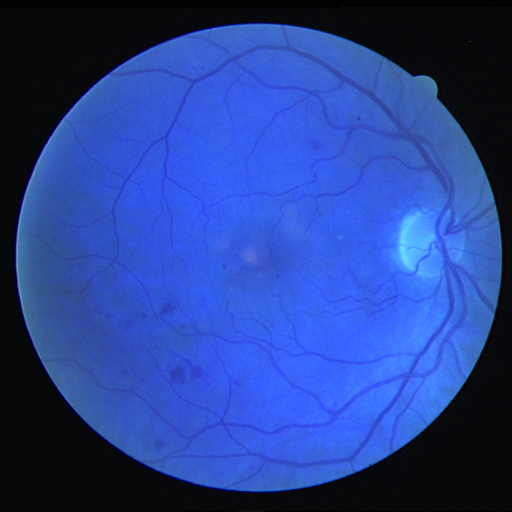}
	\includegraphics[width=1.1\linewidth]{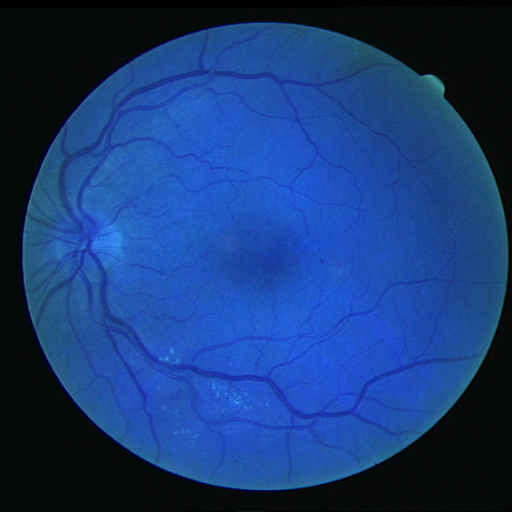}
	\centerline{\tiny (a) Images}
	\end{minipage}
	}
	\subfigure{
	\begin{minipage}[b]{0.18\linewidth}
    \includegraphics[width=1.1\linewidth]{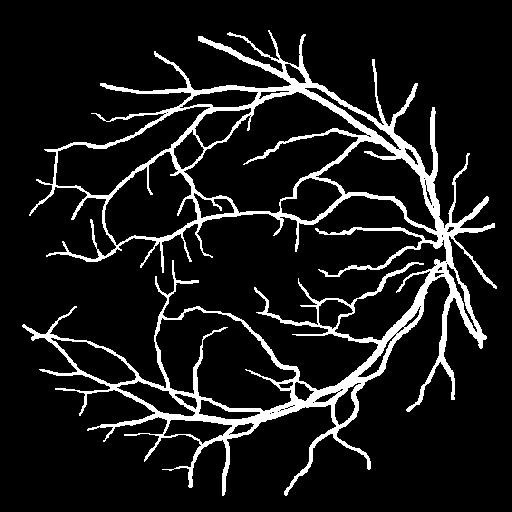}
  	\includegraphics[width=1.1\linewidth]{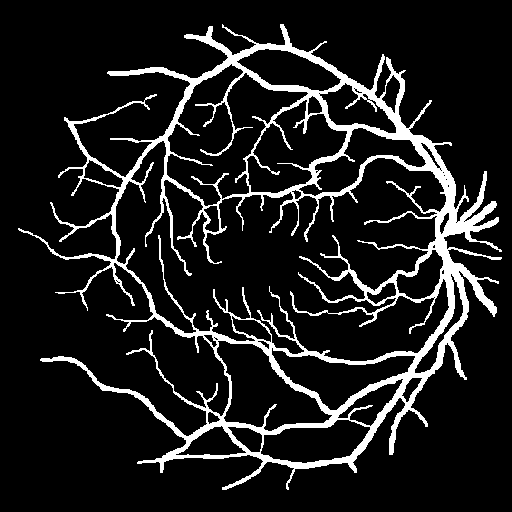}
  	\includegraphics[width=1.1\linewidth]{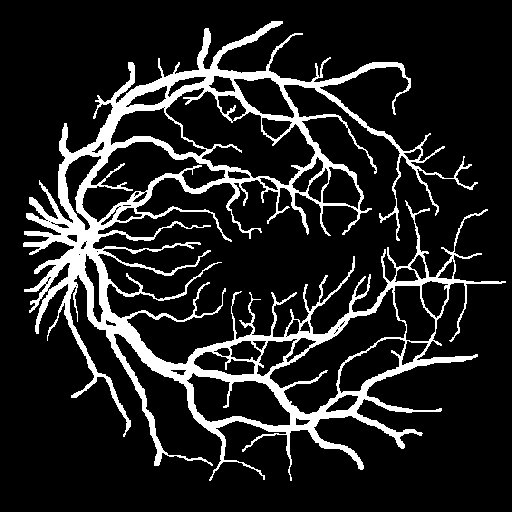}
	\centerline{\tiny (b) Ground truth}
	\end{minipage}
	}
	\subfigure{
	\begin{minipage}[b]{0.18\linewidth}
    \includegraphics[width=1.1\linewidth]{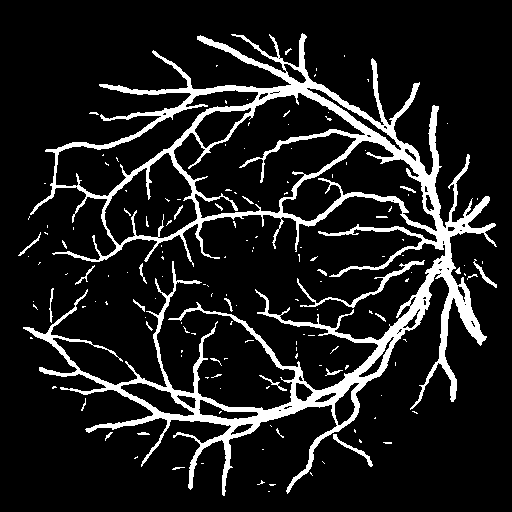}
  	\includegraphics[width=1.1\linewidth]{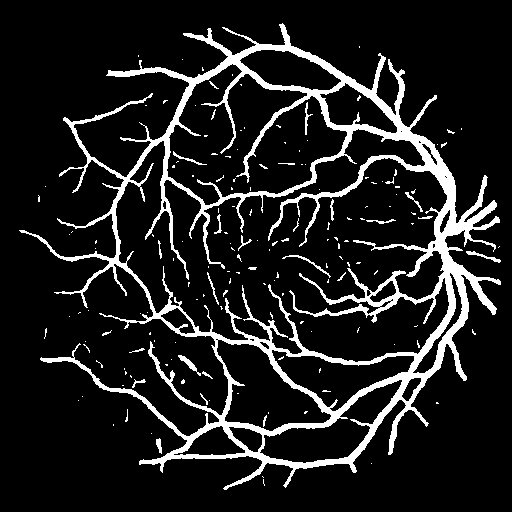}
  	\includegraphics[width=1.1\linewidth]{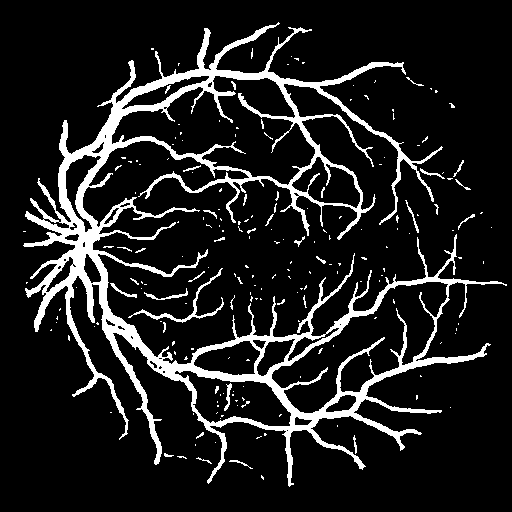}
	\centerline{\tiny (c) Unet}
	\end{minipage}
	}
	 \caption{Segmentation results for DRIVE dataset based on the originial Unet.}
\end{figure*}

\begin{figure*}
  \label{drive_unet_iels}
	\centering
	\subfigure{
	\begin{minipage}[b]{0.15\linewidth}
    \includegraphics[width=1.1\linewidth]{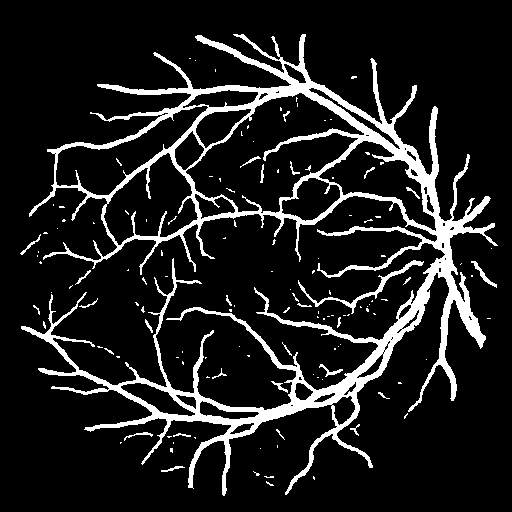}
  	\includegraphics[width=1.1\linewidth]{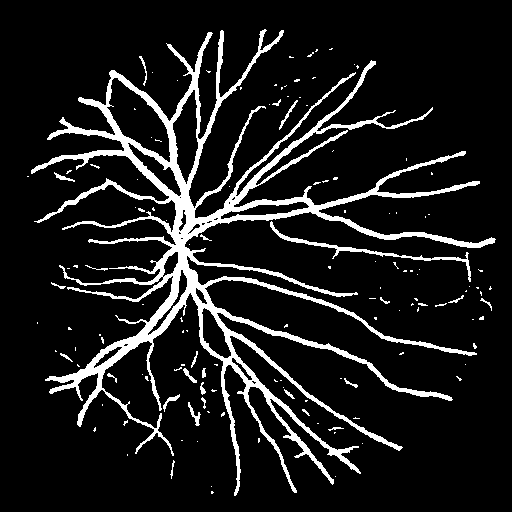}
  	\includegraphics[width=1.1\linewidth]{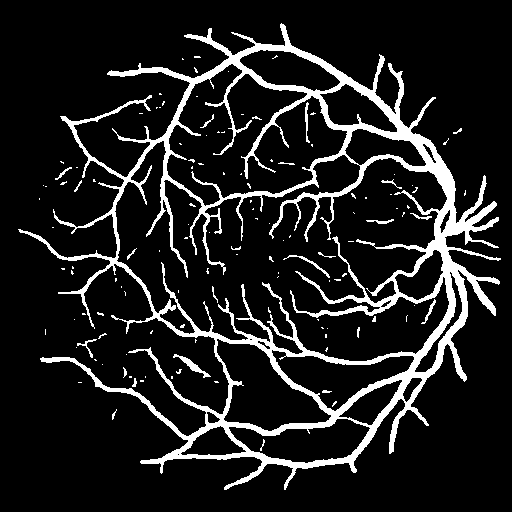}
  	\includegraphics[width=1.1\linewidth]{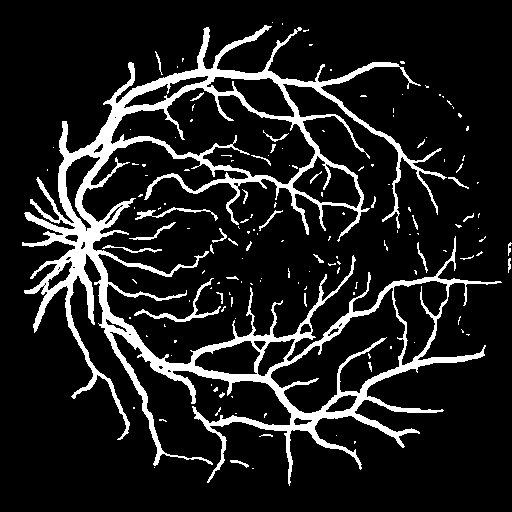}
	\centerline{ 5 IELs }
	\end{minipage}
	}
	\subfigure{
	\begin{minipage}[b]{0.15\linewidth}
    \includegraphics[width=1.1\linewidth]{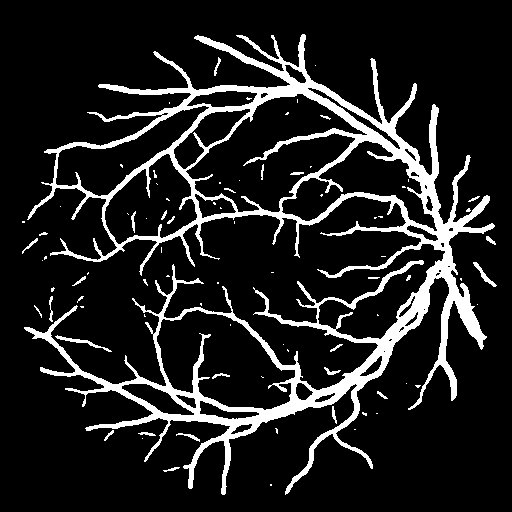}
  	\includegraphics[width=1.1\linewidth]{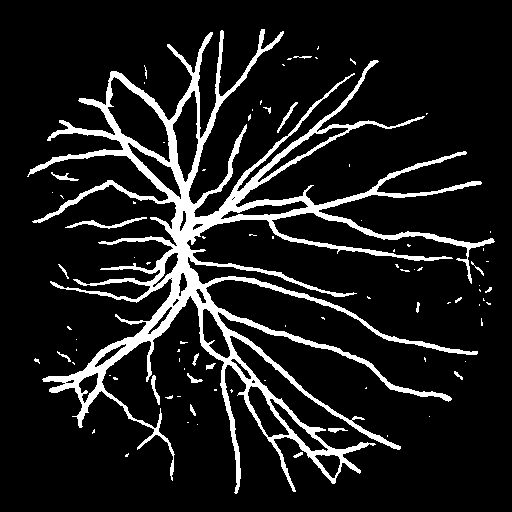}
  	\includegraphics[width=1.1\linewidth]{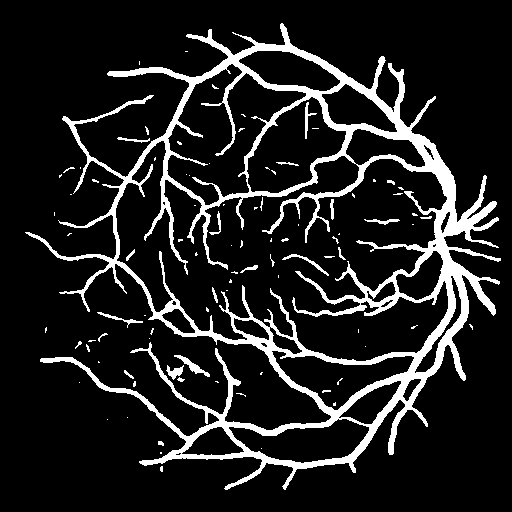}
  	\includegraphics[width=1.1\linewidth]{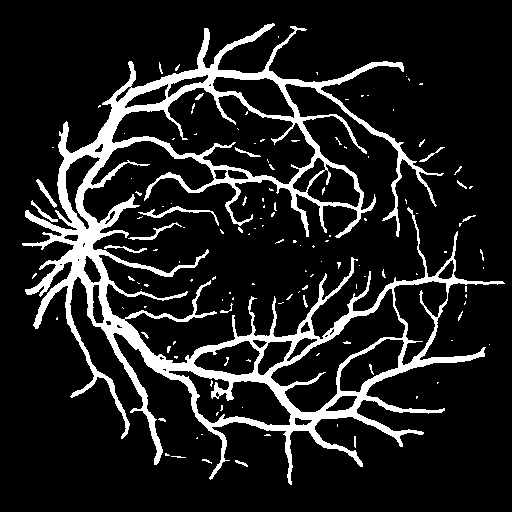}
	\centerline{ 10 IELs }
	\end{minipage}
	}
	\subfigure{
	\begin{minipage}[b]{0.15\linewidth}
    \includegraphics[width=1.1\linewidth]{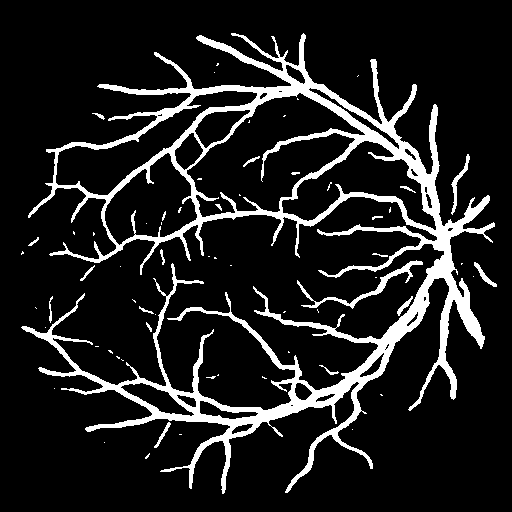}
  	\includegraphics[width=1.1\linewidth]{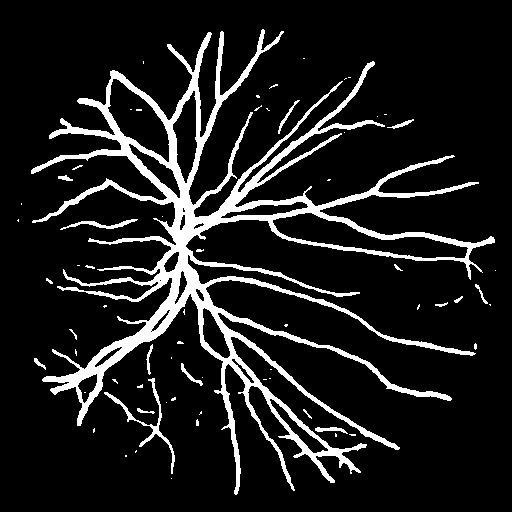}
  	\includegraphics[width=1.1\linewidth]{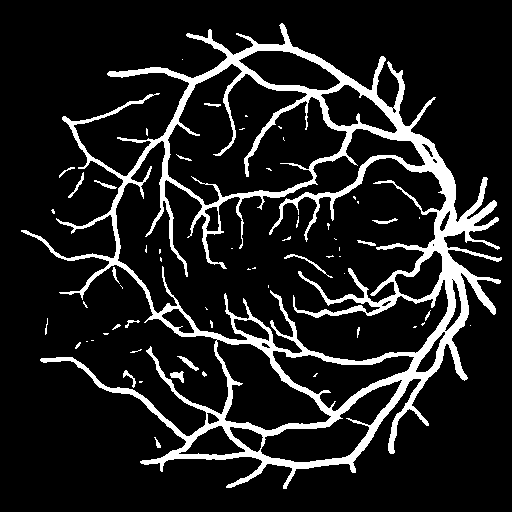}
  	\includegraphics[width=1.1\linewidth]{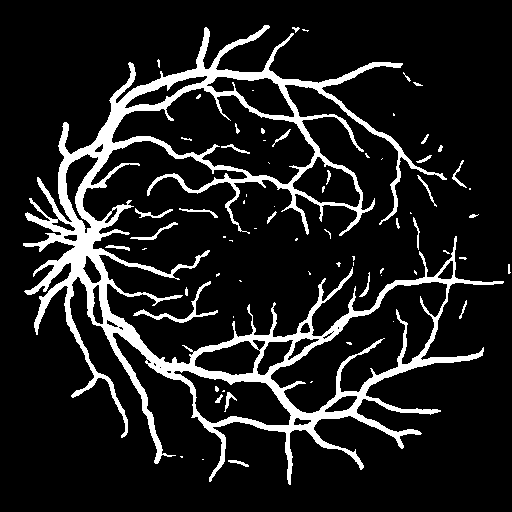}
	\centerline{ 15 IELs }
	\end{minipage}
	}
	\subfigure{
	\begin{minipage}[b]{0.15\linewidth}
    \includegraphics[width=1.1\linewidth]{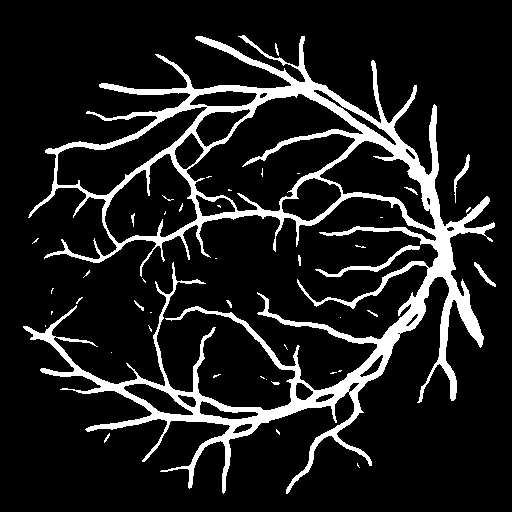}
  	\includegraphics[width=1.1\linewidth]{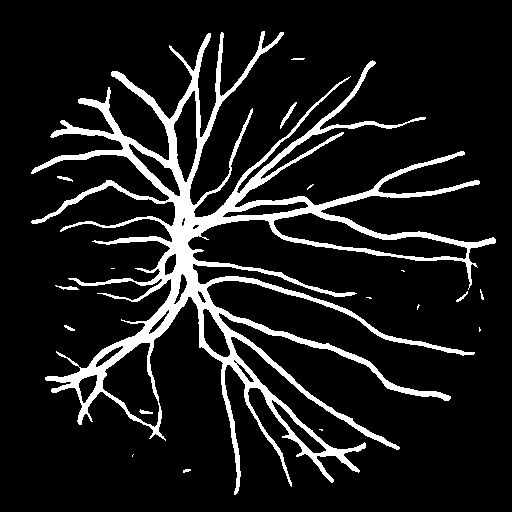}
  	\includegraphics[width=1.1\linewidth]{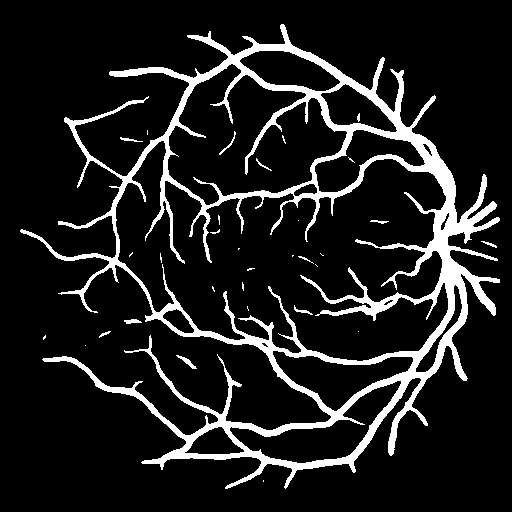}
  	\includegraphics[width=1.1\linewidth]{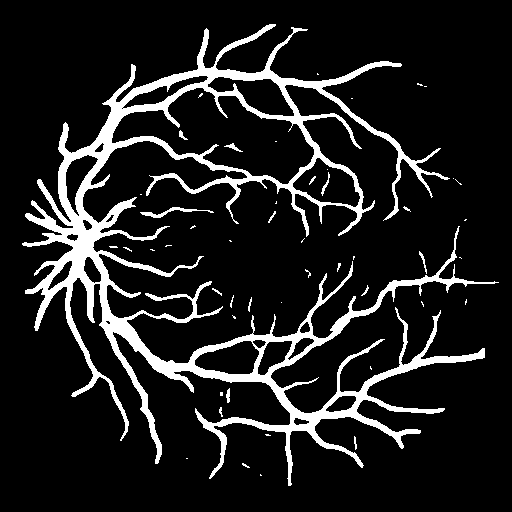}
	\centerline{ 20 IELs }
	\end{minipage}
	}
	 \caption{Segmentation results for DRIVE dataset based on the heat-diffusion IELs with different number of layers.}
\end{figure*}

{\subsubsection{Comparison between IELs and the pre- and post- processing}
In this part, we compare the performance of heat-diffusion IELs with pre- and post-processing techniques that are applied to labels and predictions using forward evolution. The results are presented in Figure \ref{fig: process dice}. It can be observed that preprocess on labels can only improve slight performance and the postprocessing on predictions even decreases the accuracy, while IELs can achieve much better results. Some preprocessed labels are shown in Figure \ref{fig: label preprocessing} and some predictions after postprocess are shown in Figure \ref{fig: label postprocessing}. It is evident that both pre- and post-processing methods cannot effectively eliminate noise while preserve the integrity of the original correct components, which can greatly influence the neural network's performance. 

\begin{figure}
    \centering
    \includegraphics[width=8cm]{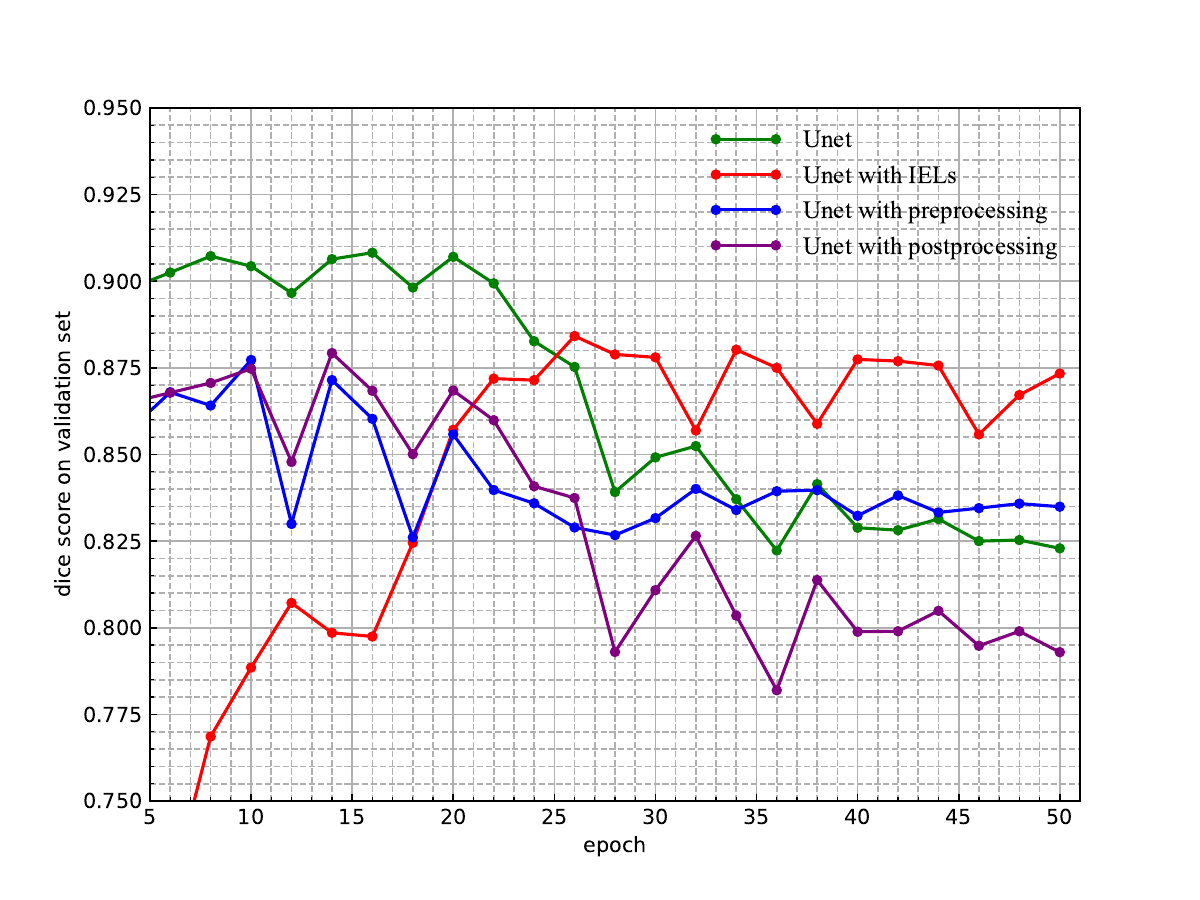}
   \caption{Dice scores of preprocessing, postprocessing and heat-diffusion IELs on the validation set of 2018 Data Science Bowl dataset.} \label{fig: process dice}
\end{figure}

\begin{figure*}
	\centering
	\subfigure{
	\begin{minipage}[b]{0.12\linewidth}
	\includegraphics[width=1\linewidth]{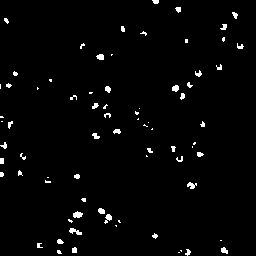}
	\includegraphics[width=1\linewidth]{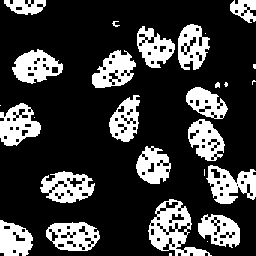}
	\includegraphics[width=1\linewidth]{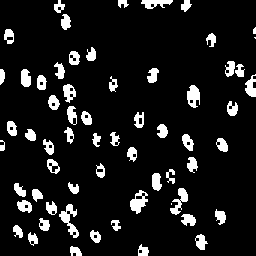}
	\centerline{\tiny (a) Noisy labels}
	\end{minipage}
	}
	\subfigure{
	\begin{minipage}[b]{0.12\linewidth}
	\includegraphics[width=1\linewidth]{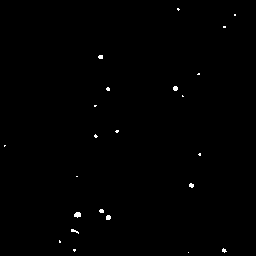}
	\includegraphics[width=1\linewidth]{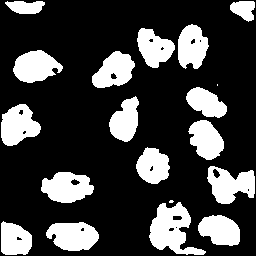}
	\includegraphics[width=1\linewidth]{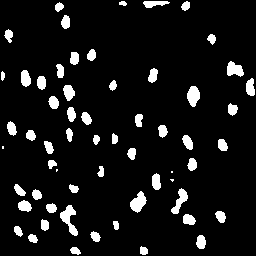}
	\centerline{\tiny (b) Preprocessed labels}
	\end{minipage}
	}
	\subfigure{
	\begin{minipage}[b]{0.12\linewidth}
	\includegraphics[width=1\linewidth]{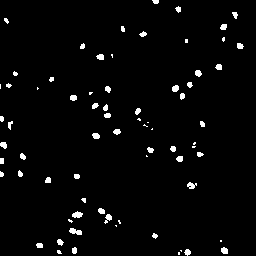}
	\includegraphics[width=1\linewidth]{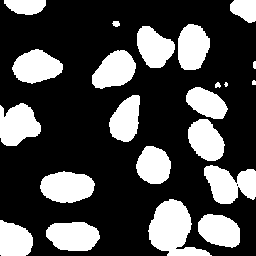}
	\includegraphics[width=1\linewidth]{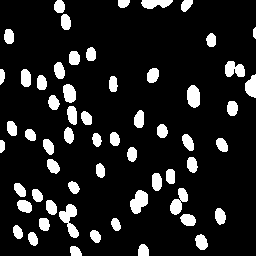}
	\centerline{\tiny (c) Ground truth}
	\end{minipage}
	}
	 \caption{Results of preprocessed labels on 2018 Data Science Bowl dataset.} \label{fig: label preprocessing}
\end{figure*}

\begin{figure*}
	\centering
	\subfigure{
	\begin{minipage}[b]{0.12\linewidth}
	\includegraphics[width=1\linewidth]{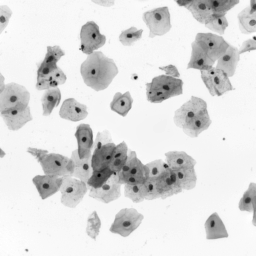}
	\includegraphics[width=1\linewidth]{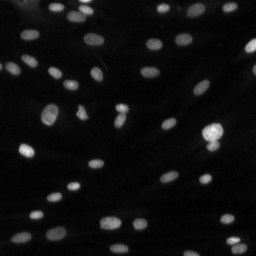}
	\includegraphics[width=1\linewidth]{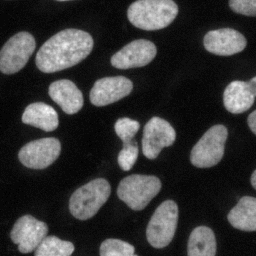}
	\centerline{\tiny (a) Images}
	\end{minipage}
	}
	\subfigure{
	\begin{minipage}[b]{0.12\linewidth}
	\includegraphics[width=1\linewidth]{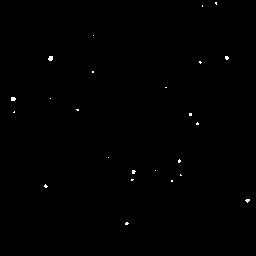}
	\includegraphics[width=1\linewidth]{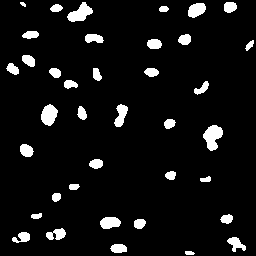}
	\includegraphics[width=1\linewidth]{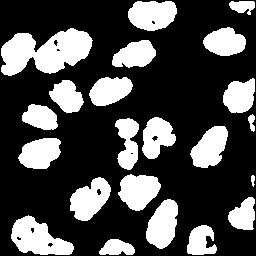}
	\centerline{\tiny (b) Postprocessed results}
	\end{minipage}
	}
	\subfigure{
	\begin{minipage}[b]{0.12\linewidth}
	\includegraphics[width=1\linewidth]{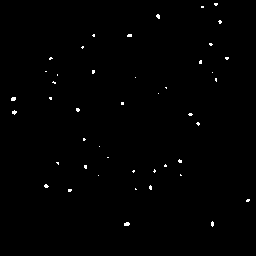}
	\includegraphics[width=1\linewidth]{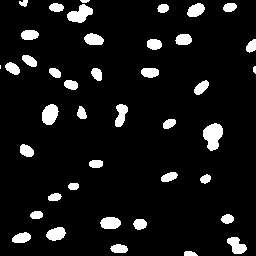}
	\includegraphics[width=1\linewidth]{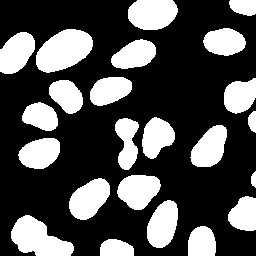}
	\centerline{\tiny (c) IELs}
	\end{minipage}
	}
	\subfigure{
	\begin{minipage}[b]{0.12\linewidth}
        \includegraphics[width=1\linewidth]{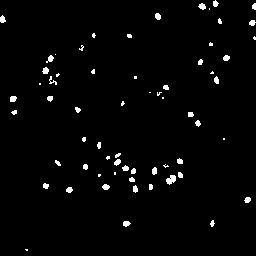}
        \includegraphics[width=1\linewidth]{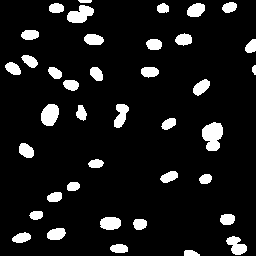}
        \includegraphics[width=1\linewidth]{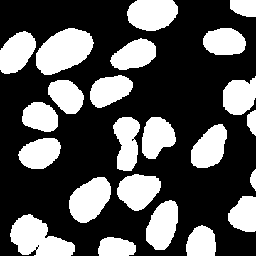}
	\centerline{\tiny (a) Ground truth}
	\end{minipage}
	}    
	 \caption{Comparison between postprocessing and heat-diffusion IELs on 2018 Data Science Bowl dataset with noisy labels.} \label{fig: label postprocessing}
\end{figure*}

\subsubsection{Comparison between IELs and other regularization methods}
In this part, we compare the performance of heat-diffusion IELs and other regularization methods including replacing IELs with forward evolution layers, loss-inserting method and $L^2$ norm regularization for weights in neural networks.
The comparison results clearly demonstrate the superiority of IELs in preserving the accuracy of neural networks while effectively managing noisy labels.

The definition of forward evolution layers (FELs) are similar to the IELs. Given an input $U$, the output of a forward evolution layer is
\begin{equation}
    L(U) = U + F_h(U)\Delta t.
  \end{equation}
For FELs, they are activated during both the training and prediction processes. The tests are conducted on a U-Net model using the DSB2018 dataset with noisy labels. The size of noise windows are $3\times 3$, and the noise percentage is 20\%. The results are presented in Figure \ref{fig: fels_loss} and Figure \ref{fig: fels}. It can be observed that noise remains apparent in the outputs for noisy labels, indicating that FELs are ineffective in enhancing the neural network's robustness to noise. 

\begin{figure}
    \centering
    \includegraphics[width=8cm]{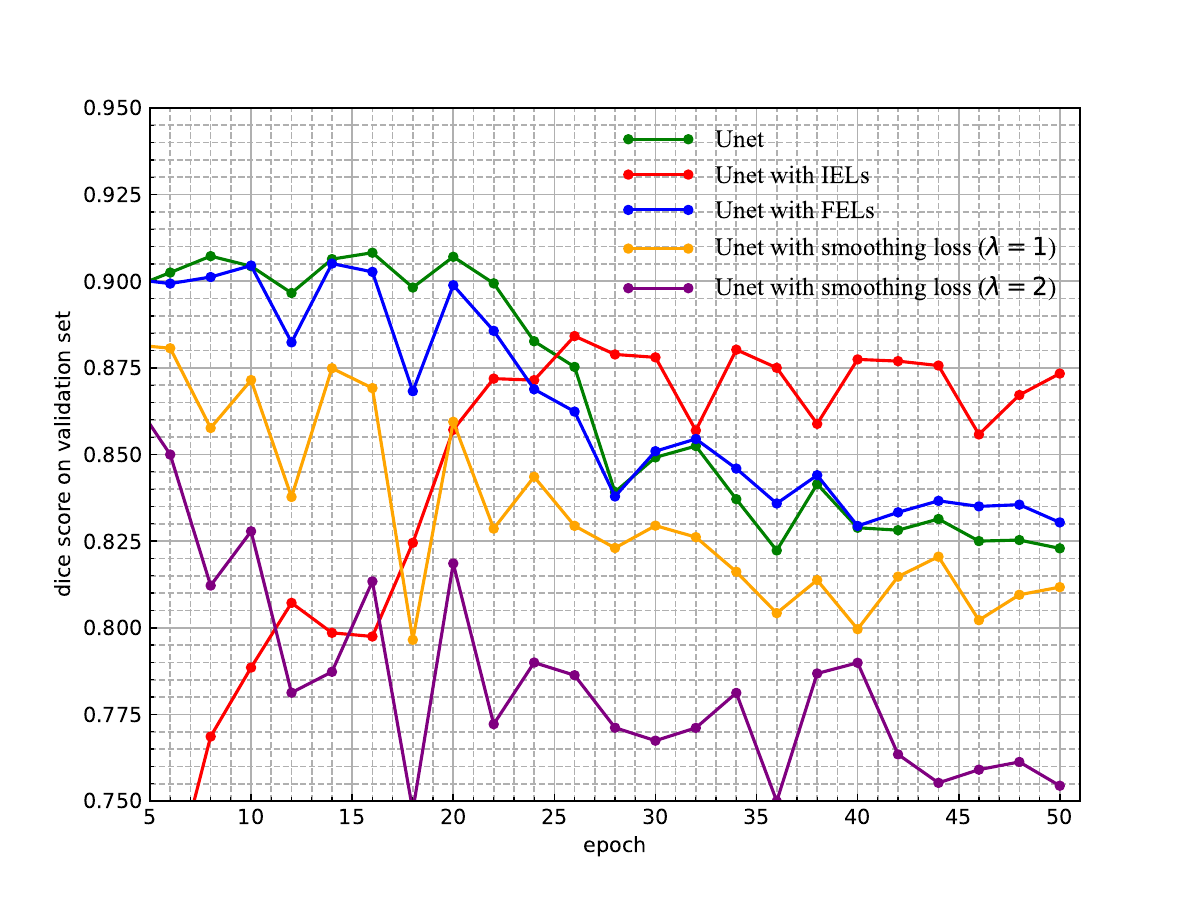}
   \caption{Dice scores of smoothing loss, heat-diffusion FELs and heat-diffusion IELs on the validation set of 2018 Data Science Bowl dataset.} \label{fig: fels_loss}
\end{figure}

\begin{figure*}
	\centering
	\subfigure{
	\begin{minipage}[b]{0.12\linewidth}
	\includegraphics[width=1\linewidth]{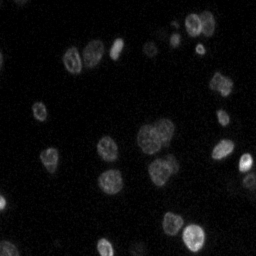}
	\includegraphics[width=1\linewidth]{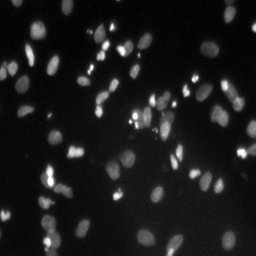}
	\includegraphics[width=1\linewidth]{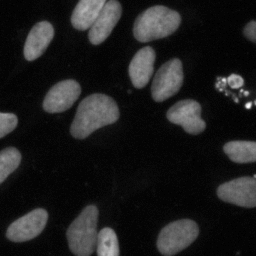}
    \includegraphics[width=1\linewidth]{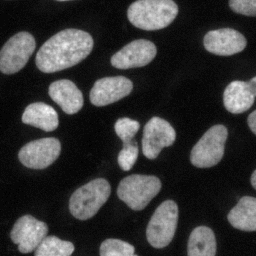}
	\centerline{\tiny (a) Images}
	\end{minipage}
	} 
	\subfigure{
	\begin{minipage}[b]{0.12\linewidth}
	\includegraphics[width=1\linewidth]{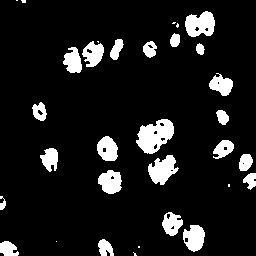}
	\includegraphics[width=1\linewidth]{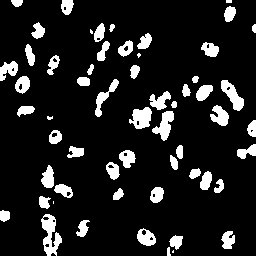}
	\includegraphics[width=1\linewidth]{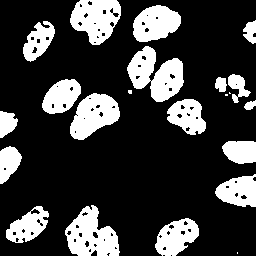}
    \includegraphics[width=1\linewidth]{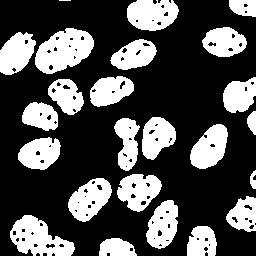}
	\centerline{\tiny (b) FELs}
	\end{minipage}
	}
	\subfigure{
	\begin{minipage}[b]{0.12\linewidth}
	\includegraphics[width=1\linewidth]{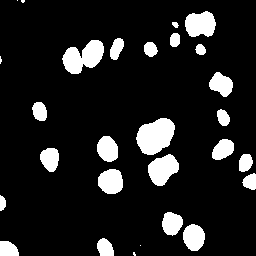}
	\includegraphics[width=1\linewidth]{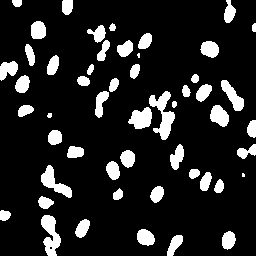}
	\includegraphics[width=1\linewidth]{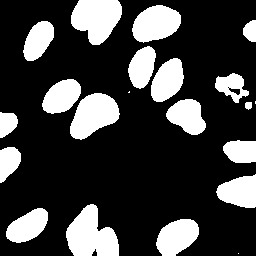}
    \includegraphics[width=1\linewidth]{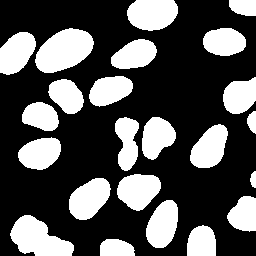}
	\centerline{\tiny (c) IELs}
	\end{minipage}
	}
	\subfigure{
	\begin{minipage}[b]{0.12\linewidth}
	\includegraphics[width=1\linewidth]{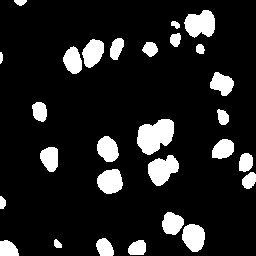}
	\includegraphics[width=1\linewidth]{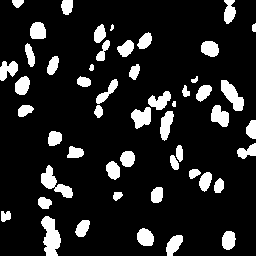}
	\includegraphics[width=1\linewidth]{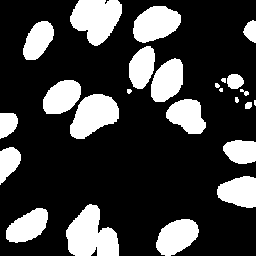}
    \includegraphics[width=1\linewidth]{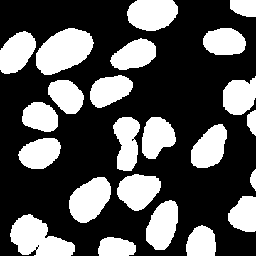}
	\centerline{\tiny (d) Ground truth}
	\end{minipage}
	}
	 \caption{Comparison between heat-diffusion FELs and heat-diffusion IELs on 2018 Data Science Bowl dataset with noisy labels.} \label{fig: fels}
\end{figure*}

The compared loss-inserting regularization method is based on the $L^2$ norm of the gradients of outputs, where the original loss is modified as follows 
\begin{equation}
    \label{loss_regularization}
    \sum_{i=1}^{N} (Loss^*(U_i,M_i) + \lambda\int_\Omega |\nabla U_i(x)|^2 dx),
\end{equation}
where $Loss^*(\cdot,\cdot)$ denotes the original loss function, and $U$'s and $M$'s represent the outputs of neural networks and labels, respectively.
For this regularization method, we set $\lambda=1, 2$ and test it on Unet for DSB2018 dataset with noisy labels. The size of noise windows are $3\times 3$, and the noise percentage is 20\%. The results are displayed in Figure \ref{fig: fels_loss} and Figure \ref{fig: loss regularization}. It can be observed that this regularization method has smoothing effect on neural network's outputs but noise can still be found in some images. Moreover, the term $|\nabla U_i|^2$ can also lead to over-smoothing of the correct parts of the output, resulting in lower accuracy compared to our IELs approach. Increasing the value of $\lambda$ further decreases accuracy significantly.

\begin{figure*}
	\centering
	\subfigure{
	\begin{minipage}[b]{0.12\linewidth}
	\includegraphics[width=1\linewidth]{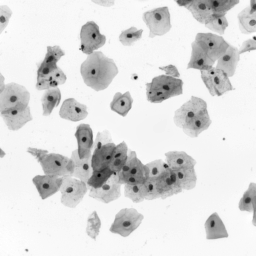}
	\includegraphics[width=1\linewidth]{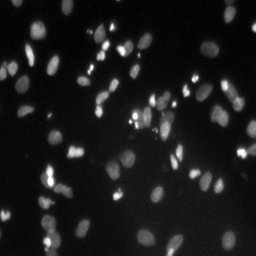}
	\includegraphics[width=1\linewidth]{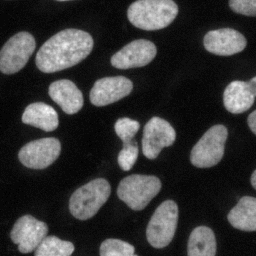}
    \includegraphics[width=1\linewidth]{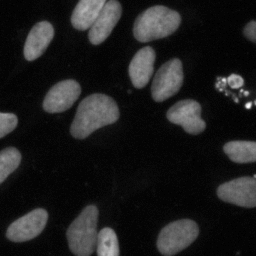}
	\centerline{\tiny (a) Images}
	\end{minipage}
	}
	\subfigure{
	\begin{minipage}[b]{0.12\linewidth}
	\includegraphics[width=1\linewidth]{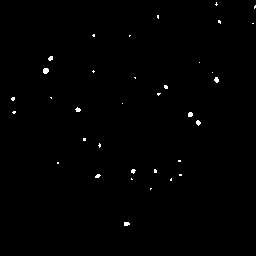}
	\includegraphics[width=1\linewidth]{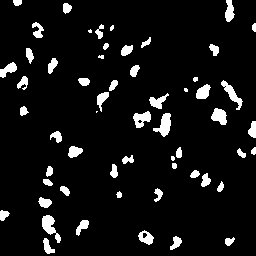}
	\includegraphics[width=1\linewidth]{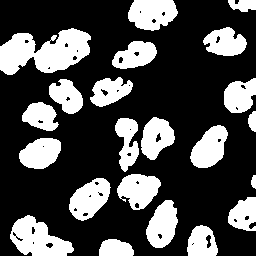}
    \includegraphics[width=1\linewidth]{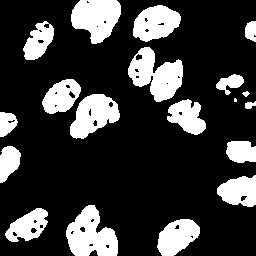}
	\centerline{\tiny (b) $\lambda=1$}
	\end{minipage}
	}
	\subfigure{
	\begin{minipage}[b]{0.12\linewidth}
	\includegraphics[width=1\linewidth]{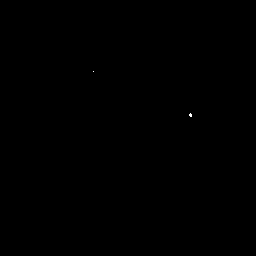}
	\includegraphics[width=1\linewidth]{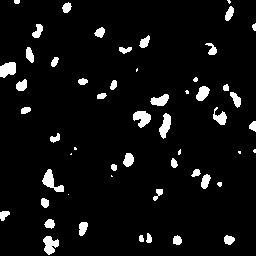}
	\includegraphics[width=1\linewidth]{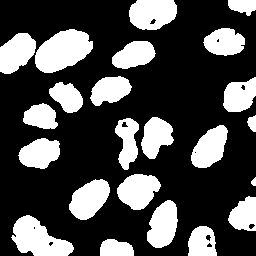}
    \includegraphics[width=1\linewidth]{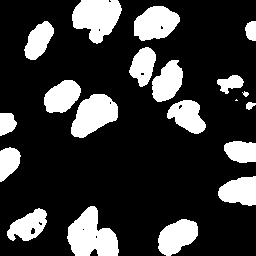}
	\centerline{\tiny (c) $\lambda=2$}
	\end{minipage}
	}    
	\subfigure{
	\begin{minipage}[b]{0.12\linewidth}
	\includegraphics[width=1\linewidth]{postprocess_iels1.png}
	\includegraphics[width=1\linewidth]{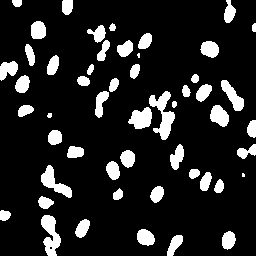}
	\includegraphics[width=1\linewidth]{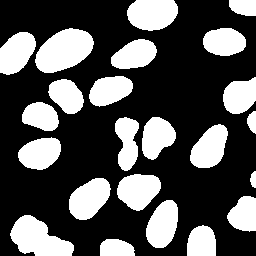}
    \includegraphics[width=1\linewidth]{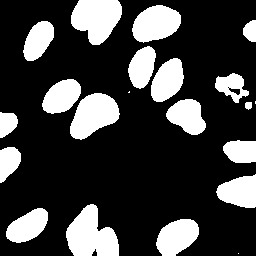}
	\centerline{\tiny (d) IELs}
	\end{minipage}
	}
	\subfigure{
	\begin{minipage}[b]{0.12\linewidth}
	\includegraphics[width=1\linewidth]{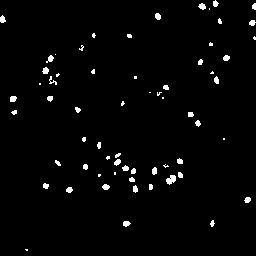}
	\includegraphics[width=1\linewidth]{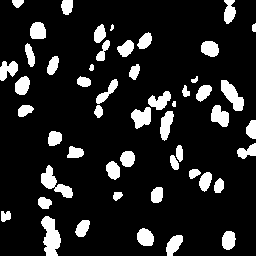}
	\includegraphics[width=1\linewidth]{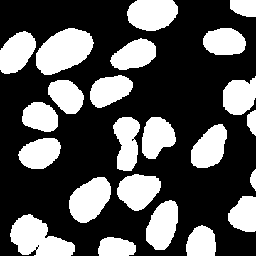}
    \includegraphics[width=1\linewidth]{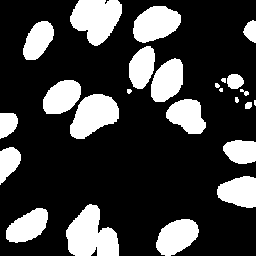}
	\centerline{\tiny (e) Ground truth}
	\end{minipage}
	}
	 \caption{Comparison between loss-inserting method and heat-diffusion IELs on 2018 Data Science Bowl dataset with noisy labels. (a) and (e) are the original images and the corresponding ground truth, respectively. (b) and (c) exhibit the segmentation results of the loss regularization with $\lambda=1$ and $\lambda=2$, respectively. (d) shows the results of our heat-diffusion IELs.} \label{fig: loss regularization}
\end{figure*}

$L^2$ norm regularization for weights modifies the original loss as follows:
\begin{equation}
    \sum_{i=1}^{N} Loss^*(U_i,M_i) + \sum_{j}\lambda |\omega_j|^2,
\end{equation}
where $\omega_j$'s are all the learnable parameters in the neural network. For this regularization method, we set $\lambda=0.1, 1$ and test it on Unet for DRIVE dataset with noisy labels. The size of noise windows are $1\times 1$, and the noise percentage is 50\%. The results are exhibited in Figure \ref{fig: l2 dice} and Figure \ref{fig: l2 weights}. It can be observed that our IELs method outperforms this regularization method in handling noisy labels. In particular, strong $L^2$ norm regularization for weights tends to significantly compromise accuracy.

\begin{figure}
    \centering
    \includegraphics[width=8cm]{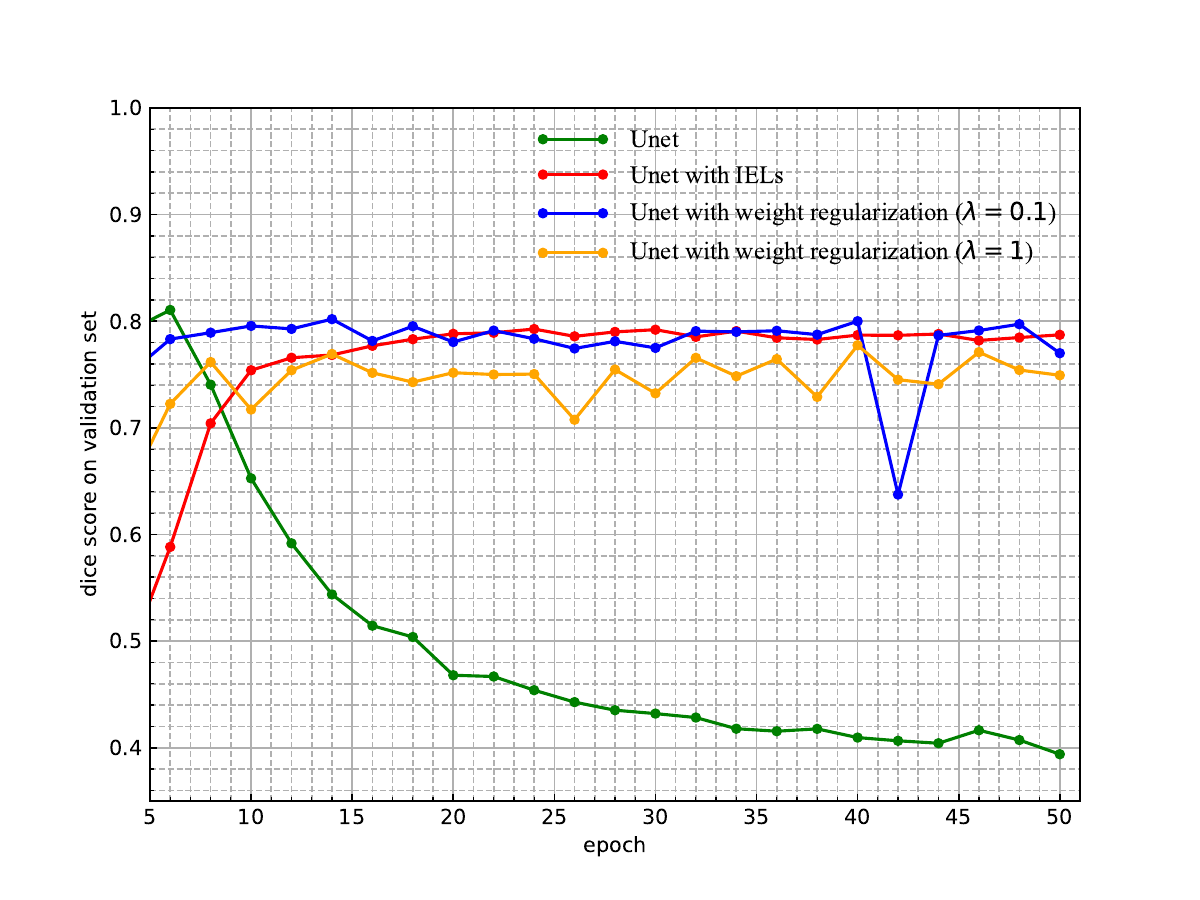}
   \caption{Dice scores of weight regularization and heat-diffusion IELs on the validation set of 2018 Data Science Bowl dataset.} \label{fig: l2 dice}
\end{figure}

\begin{figure*}
	\centering
	\subfigure{
	\begin{minipage}[b]{0.12\linewidth}
	\includegraphics[width=1\linewidth]{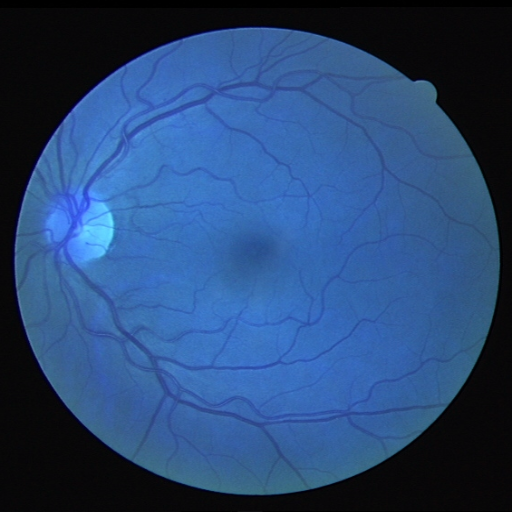}
	\includegraphics[width=1\linewidth]{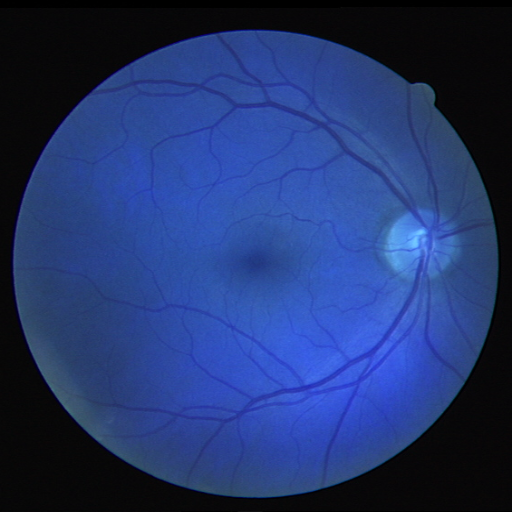}
    \includegraphics[width=1\linewidth]{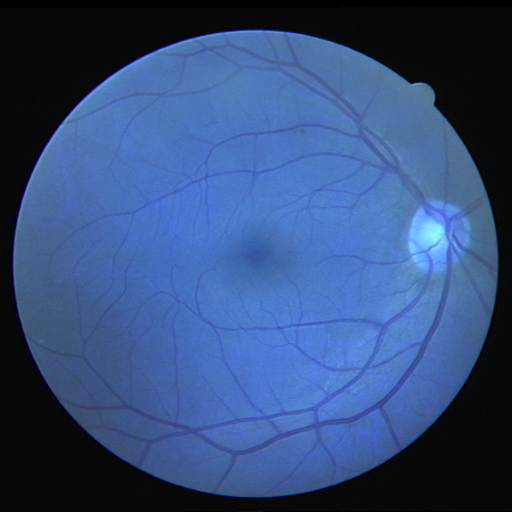}
	\centerline{\tiny (a) Images}
	\end{minipage}
}
	\subfigure{
	\begin{minipage}[b]{0.12\linewidth}
	\includegraphics[width=1\linewidth]{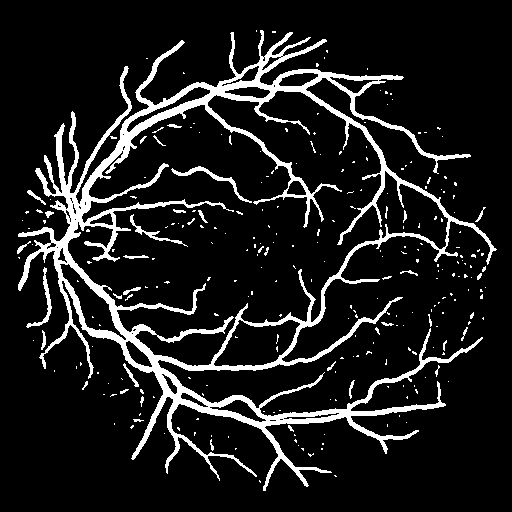}
	\includegraphics[width=1\linewidth]{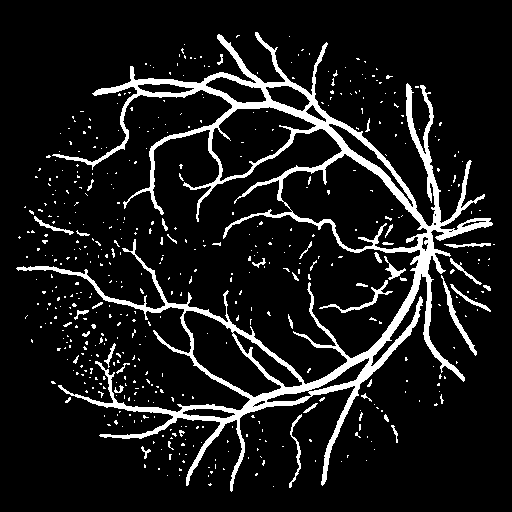}
    \includegraphics[width=1\linewidth]{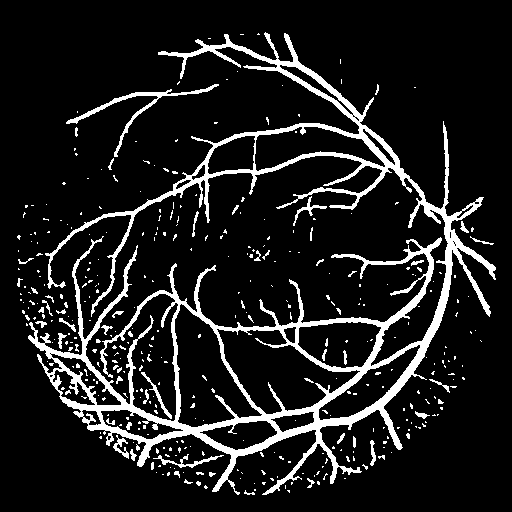}
	\centerline{\tiny (b) $\lambda=0.1$}
	\end{minipage}
	}
	\subfigure{
	\begin{minipage}[b]{0.12\linewidth}
	\includegraphics[width=1\linewidth]{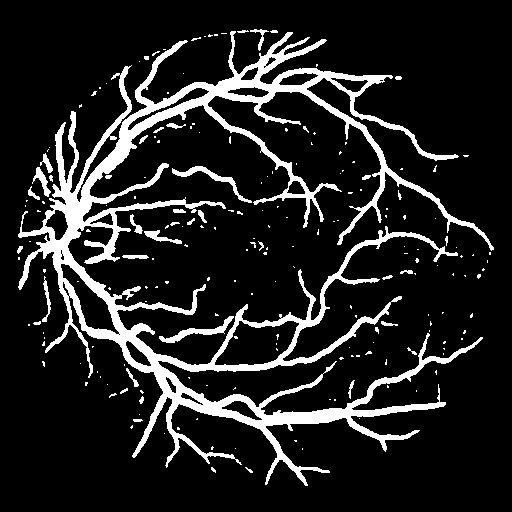}
	\includegraphics[width=1\linewidth]{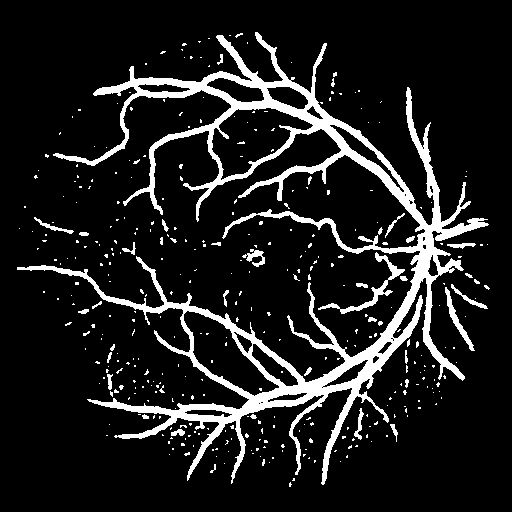}
    \includegraphics[width=1\linewidth]{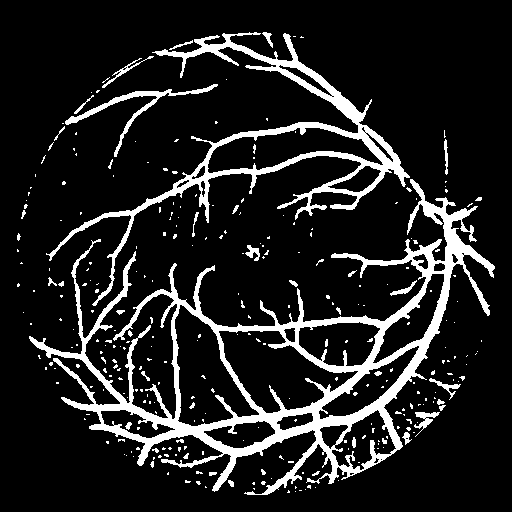}
	\centerline{\tiny (c) $\lambda=1$}
	\end{minipage}
	}
	\subfigure{
	\begin{minipage}[b]{0.12\linewidth}
	\includegraphics[width=1\linewidth]{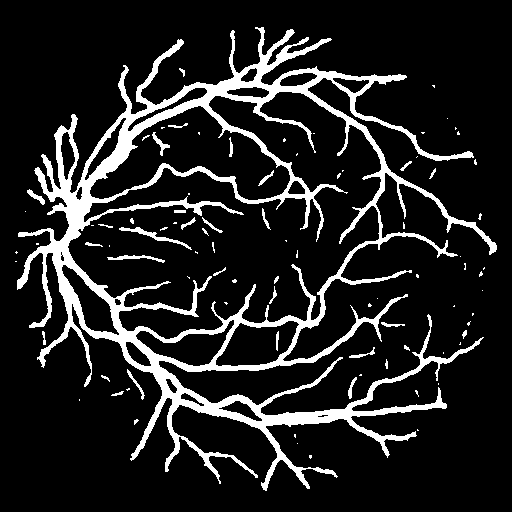}
	\includegraphics[width=1\linewidth]{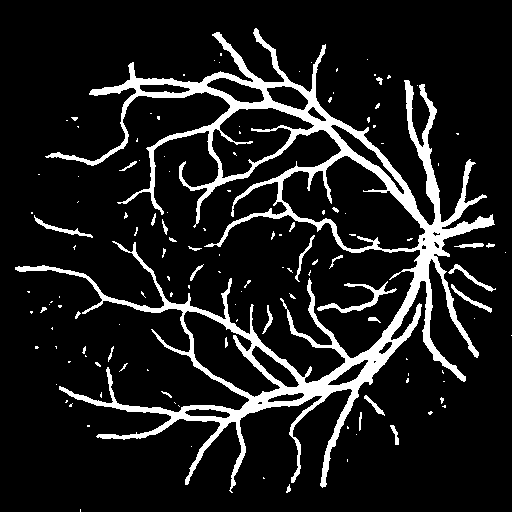}
    \includegraphics[width=1\linewidth]{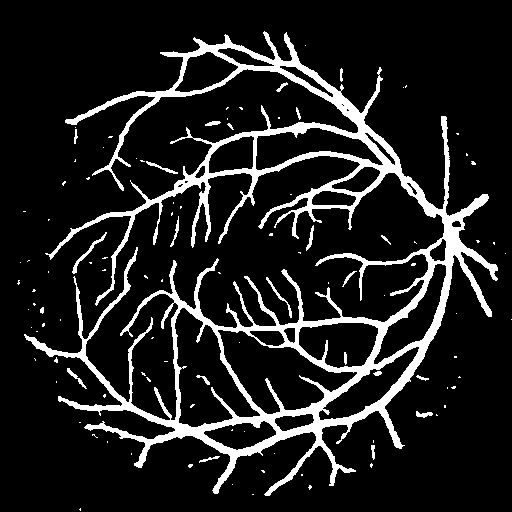}
	\centerline{\tiny (d) IELs}
	\end{minipage}
	}
	\subfigure{
	\begin{minipage}[b]{0.12\linewidth}
	\includegraphics[width=1\linewidth]{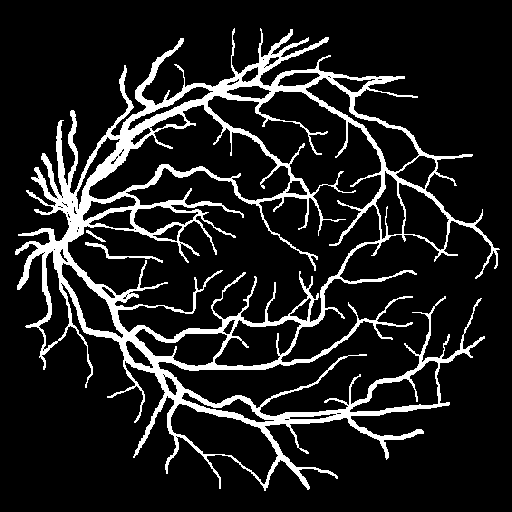}
	\includegraphics[width=1\linewidth]{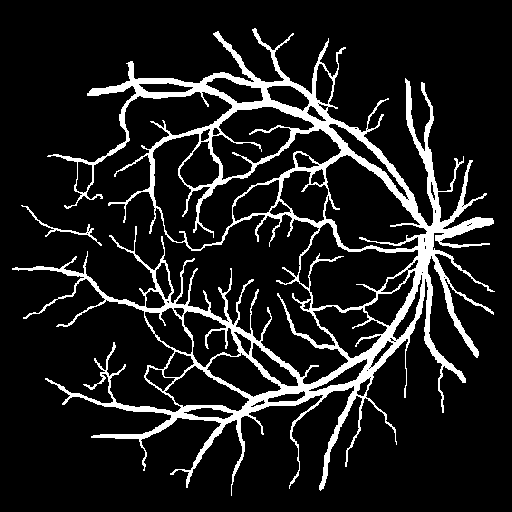}
    \includegraphics[width=1\linewidth]{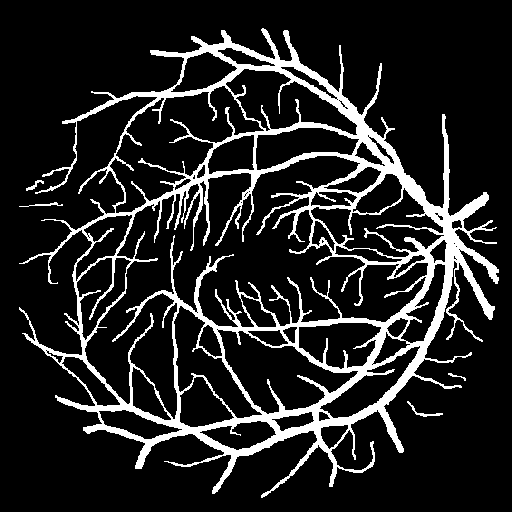}
	\centerline{\tiny (e) Ground truth}
	\end{minipage}
	}    
	 \caption{Comparison between the weight regularization and heat-diffusion IELs on DRIVE dataset with noisy labels. (a) and (e) are the original images and the corresponding ground truth, respectively. (b) and (c) exhibit the segmentation results of the $L^2$ regularization method on weights with $\lambda=0.1$ and $\lambda=1$, respectively. (d) shows the results of our heat-diffusion IELs.} \label{fig: l2 weights}
\end{figure*}
}

\section{Curve-motion IELs}
\label{curve-motion IELs}
In this section, we design IELs based on level set-based curve-motion equations, which are typically non-energy PDE models. These developed IELs can be adapted to possess the capability of regularizing the outputs of neural networks \modified{by imposing convexity on the segmentation}.
\subsection{Derivation of the Curve-motion IELs}
A level-set based curve-motion equation can be formulated as follows:
\begin{equation}
  \label{level-set}
  \left\{
  \begin{array}{ll}
  \phi_t = V(x, t)|\nabla \phi|, &  t\in[0,T]\\
    \phi(x, 0) = \phi_0, & \forall x\in\Omega\\
    \frac{\partial \phi}{\partial x}(x,t) = 0, & \forall x \in \partial\Omega,\quad t\in[0,T]
\end{array}
\right.
\end{equation}
Here, the curve at time $t$ is \modified{implicitly} represented by $\{x: \phi(x,t)=0\}$, which we denote as $\Gamma_t$. And the object of interest is represented by the region $\{x: \phi(x,t)\ge 0\}$. The function $V(x, t)$ gives the speed of $\Gamma_t$ in its outward normal direction.

Based on equation \eqref{level-set}, we can design the curve-motion IELs as follows
\begin{equation}
  L(U^n) = U^n+ \Delta t V^n|\nabla_h U^n|,
\end{equation}
where $\nabla_h$ is the finite difference approximation of gradient operator and the value $|\nabla_h U|$ at pixel $(ih,jh)$
\begin{equation}
  \label{F_nabla}
  \begin{split}
      (|\nabla_h U|)_{i,j} &= \frac{\sqrt{(U_{i+1,j}-U_{i-1,j})^2+(U_{i,j+1}-U_{i,j-1})^2}}{2h},\\
      & i=1,2,\cdots,M_1, j=1,2,\cdots,M_2.
  \end{split}
\end{equation}

For some specific segmentation tasks, the shapes of segmentation maps should be convex. Curve-motion IELs enable us to design evolutions that penalize the concave boundaries of neural network outputs by selecting an appropriate $V(x, t)$. A straightforward approach to achieving this is by setting
\begin{equation}
  \label{speed function}
  V(x, t)  = \begin{cases}
    -1, & x \in R_d^{C}\\
    0, & \text{else}
  \end{cases},
\end{equation}
where $C$ is the set of all concave boundaries of the outputs and $R_d^{C}$ is defined as $\{x : \text{distance}(x, C) \le d \}$.

When dealing with a segmentation map, allowing the inward concave boundaries to evolve outward results in the map becoming less concave. Conversely, the speed function defined in \eqref{speed function} prompts concave boundaries to evolve inward, thereby amplifying the error caused by concavity.

To derive the exact expression of $V(x, t)$, it is necessary to identify the concave curves. To accomplish this, we introduce the convex shape condition \cite{luo2020convex,liu2020convex}. Given an object $\mathbb{D}\subset \Omega\subset \mathbb{R}^2$, let $f(x)$ be the indicative function of $\mathbb{D}$, i.e.
\begin{equation}
  f(x) =
  \begin{cases}
  1, & x \in \mathbb{D} \\
  0, & \text{else.}
\end{cases}
\end{equation}
Denote $g_r(x)$ be a kernel function with support being a sphere $\mathbb{B}_r = \{x: ||x||_2 \le r \} \in \Omega$, namely,
\begin{equation}
  g_r(x)=
  \begin{cases}
\frac{1}{|\mathbb{B}_r|}, &  x \in \mathbb{B}_r \\
  0,                      & \text{else.}
\end{cases}
\end{equation}
Then the connected components of $\mathbb{D}$ are all convex if
\begin{equation}
  \label{convex condition}
  (1 - f(x))(g_r * (1 - 2f)(x)) \ge 0, \quad \forall \mathbb{B}_r \in \Omega, \forall x \in \Omega,
\end{equation}
where $*$ denotes the convolution operator.

Locating all the concave curves is a non-trivial task in discrete setting. However, by leveraging the convex shape condition, we understand that \eqref{convex condition} should hold for all  $r \ge 0$ with $\mathbb{B}_r \in \Omega$ if $\mathbb{D}$ is convex. In other words, for any given set $K = \{r_0, r_1, \cdots, r_k\}$, if a point fails to satisfy \eqref{convex condition} for any $r \in K$, it implies the presence of concave curves in the vicinity. Consequently, we can identify all the points that do not satisfy \eqref{convex condition} for all $r \in K$ to approximate the set $C$.
Based on this, we employ curve-motion IELs to impose convexity on the segmentation outputs of neural networks. The workflow of the designed curve-motion is displayed in Algorithm \ref{curve-motion iel}.


\begin{algorithm}[H]
  \caption{Curve-motion IELs for Convex Shape Regularization}
	\label{curve-motion iel}
  \KwIn{$\Delta t, N, d, K = \{r_0, r_1, \cdots, r_k\}$}
  \textbf{Initialize $U^0$ as the output of the neural network}

  \For{$n = 0, 1, \cdots, N-1$}
{
      Set
      $$
      f(x) =
        \begin{cases}
        1, & \text{if } U^n(x) > 0 \\
        0, & \text{else.}
        \end{cases}
      $$

    Find the approximate set of concave boundaries
          $$
    				C := \bigcup_{r\in K} \{x: (1 - f(x))(g_r * (1 - 2f)(x)) < 0 \}.
    			$$

    Define
      $$
      V^n  = \begin{cases}
        -1, & x \in R_d^{C}\\
        0, & \text{else,}
      \end{cases}
      $$
      with $R_d^{C} := \{x : \text{distance}(x, C) \le d \}$\;

    Update
            $$
            U^{n+1} = U^n+ \Delta t V^n|\nabla_h U^n|
            $$
}
\KwOut{$U^{N}$}
\end{algorithm}

\subsection{Experiments for the curve-motion IELs}
In this part, we evaluate our curve-motion IELs on Unet and the Retinal Fundus Glaucoma Challenge (REFUGE) dataset \cite{orlando2020refuge}, which provides 1200 fundus images with ground truth segmentations of Optic Disc and Cup. The dataset is divided into training set, validation set and test set and each set contains 400 images. The image size is  $2056 \times 2124$. To expedite computational processes, we resize all the images to $512 \times 512$.

The segmentation maps of Optic Disc should be convex but the blood vessels near it can affect segmentation results and lead to concave maps. In our experiment, we will amplify the concavity of the segments for Optic Disc through curve-motion IELs during training. $K, d, \Delta t$ and the number of IELs are chosen as $\{5, 10, 15\}$, 3, 0.1 and 20, respectively.

The neural network is trained on the training dataset, and its performance is evaluated on the validation set. We conducted training for both Unet and Unet with IELs for 50 epochs, with a learning rate of 0.0002. The training loss and the dice score on the validation set are illustrated in Figure \ref{refuge_loss} and Figure \ref{refuge_dice}, respectively. To demonstrate the effectiveness of the curve-motion IELs, we provide visuals of some inputs and the corresponding outputs of IELs during the training process in Figure \ref{refuge_fig1}, which clearly show that our curve-motion IELs exacerbate the concavity of segmentation boundaries. Furthermore, the segmentation results of the original Unet and Unet with IELs after 50 epochs are presented in Figure \ref{refuge_fig2}. It is noteworthy that the utilization of our curve-motion IELs in Unet results in a notable tendency for the Optic Disc segments to exhibit greater convexity. This observation demonstrates the capacity of our curve-motion IELs to influence the shape of segmentation maps, promoting convexity in this specific context.

\begin{figure}
  \centering
  \includegraphics[width=6cm,height=4cm]{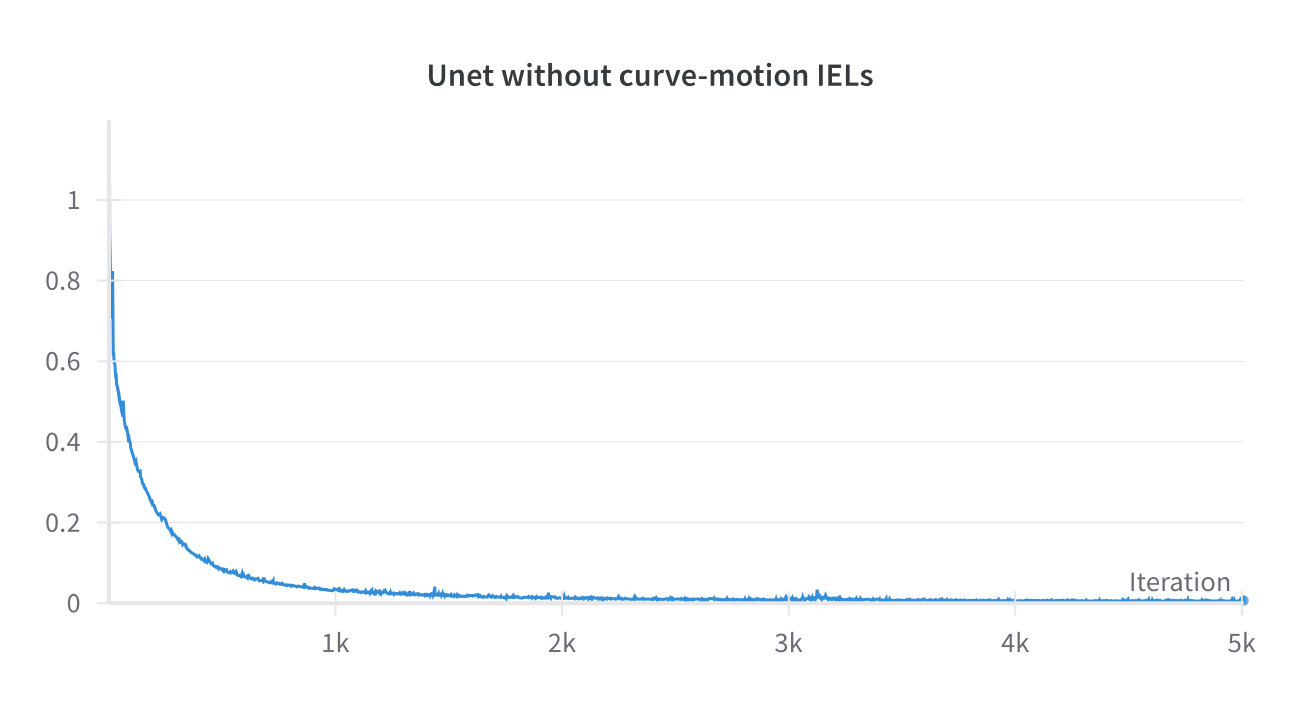} \hspace{1cm}
  \includegraphics[width=6cm,height=4cm]{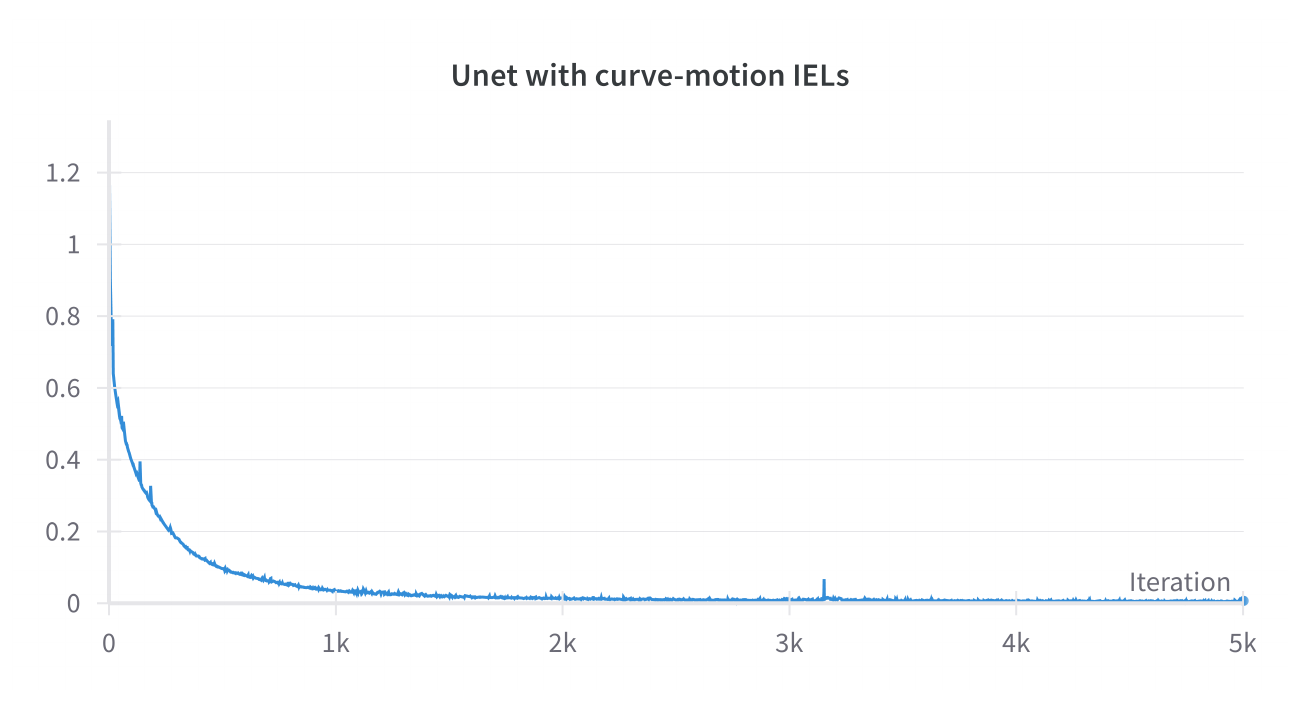}
 \caption{Training loss of Unet \cite{ronneberger2015u} and Unet with curve-motion IELs.} \label{refuge_loss}
\end{figure}

\begin{figure}
  \centering
  \includegraphics[width=8cm]{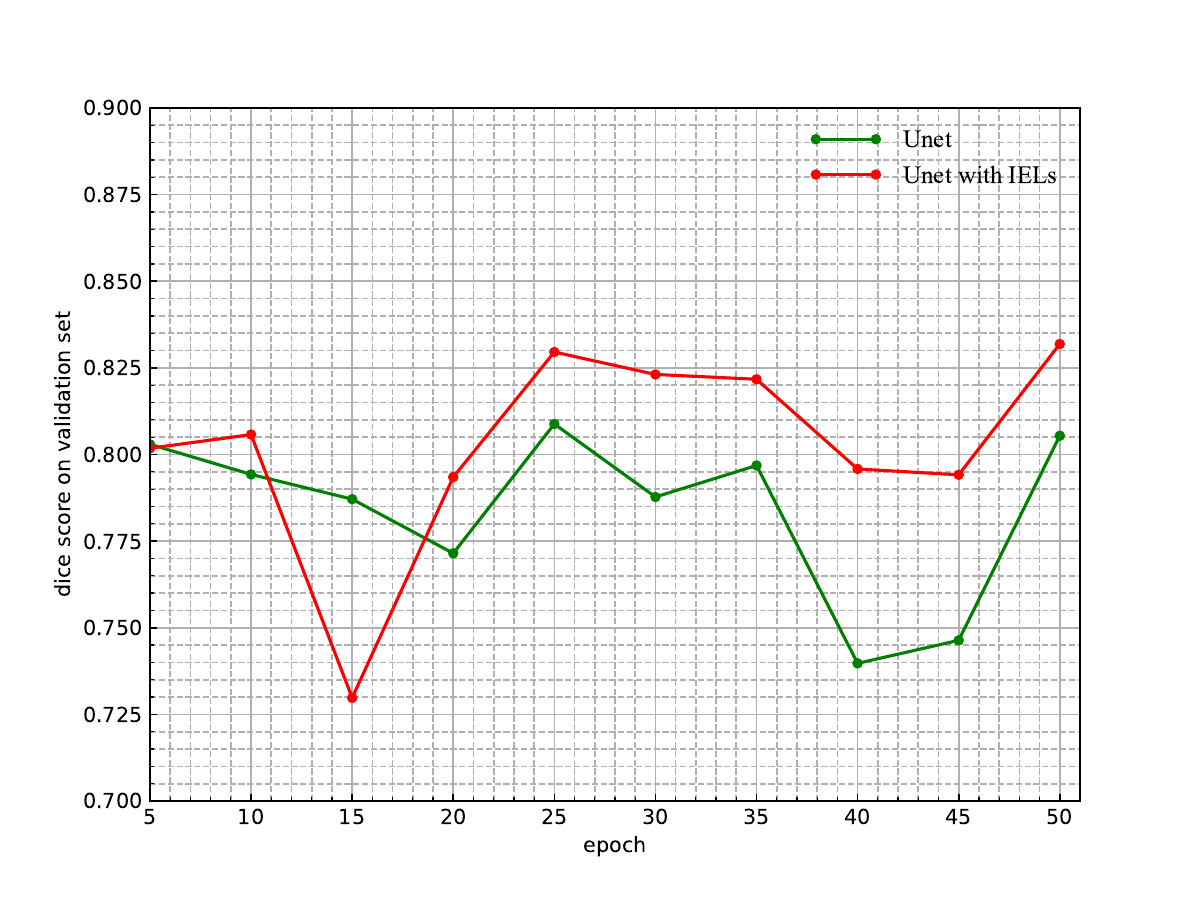}
 \caption{Dice score on the validation set of REFUGE dataset based on Unet \cite{ronneberger2015u} and Unet with curve-motion IELs. The results of Unet and Unet with heat-diffusion IELs are highlighted by green lines and red lines, respectively.} \label{refuge_dice}
\end{figure}

\begin{figure*}
	\centering
	\subfigure{
	\begin{minipage}[b]{0.12\linewidth}
    \includegraphics[width=1\linewidth]{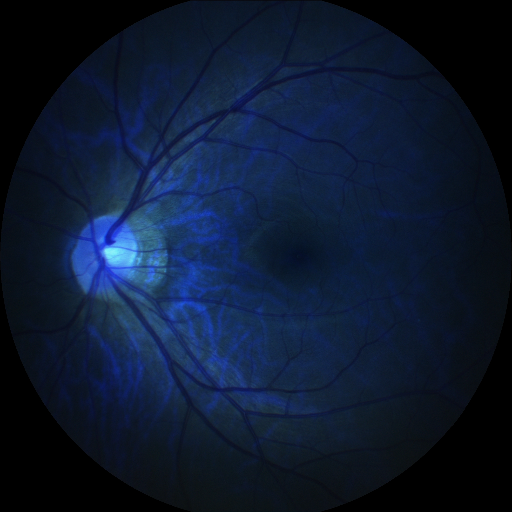}
  	\includegraphics[width=1\linewidth]{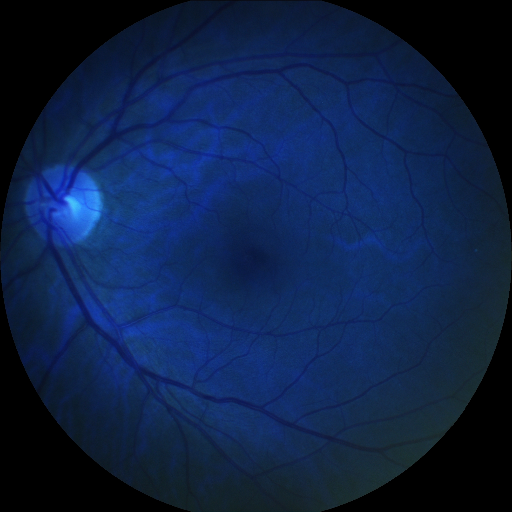}
  	\includegraphics[width=1\linewidth]{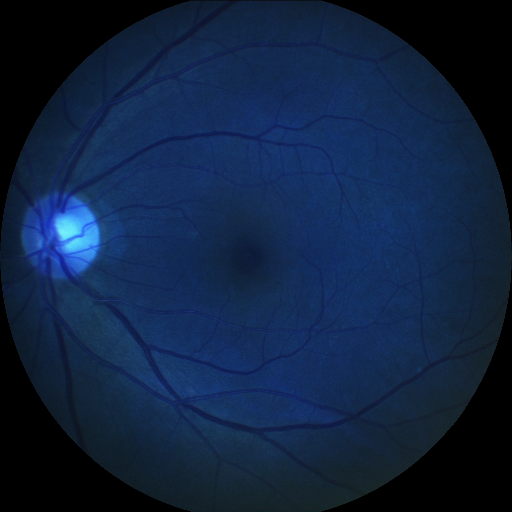}
	\centerline{\tiny (a) Images}
	\end{minipage}
	}
	\subfigure{
	\begin{minipage}[b]{0.12\linewidth}
    \includegraphics[width=1\linewidth]{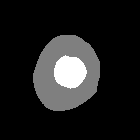}
  	\includegraphics[width=1\linewidth]{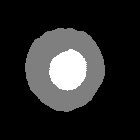}
  	\includegraphics[width=1\linewidth]{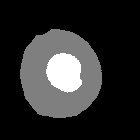}
	\centerline{\tiny (b) Inputs}
	\end{minipage}
	}
	\subfigure{
	\begin{minipage}[b]{0.12\linewidth}
    \includegraphics[width=1\linewidth]{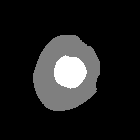}
  	\includegraphics[width=1\linewidth]{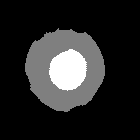}
  	\includegraphics[width=1\linewidth]{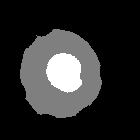}
	\centerline{\tiny (c) Outputs}
	\end{minipage}
	}
	\subfigure{
	\begin{minipage}[b]{0.12\linewidth}
    \includegraphics[width=1\linewidth]{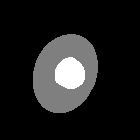}
  	\includegraphics[width=1\linewidth]{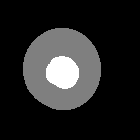}
  	\includegraphics[width=1\linewidth]{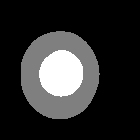}
	\centerline{\tiny (d) Ground truth}
	\end{minipage}
	}
	 \caption{Comparison between inputs and outputs of the curve-motion IELs during training. (a) and (d) are the original images and the corresponding ground truth, respectively. (b) and (c) indicate the inputs and outputs of the curve-motion IELs, respectively.} \label{refuge_fig1}
\end{figure*}

\begin{figure*}
	\centering
	\subfigure{
	\begin{minipage}[b]{0.12\linewidth}
    \includegraphics[width=1\linewidth]{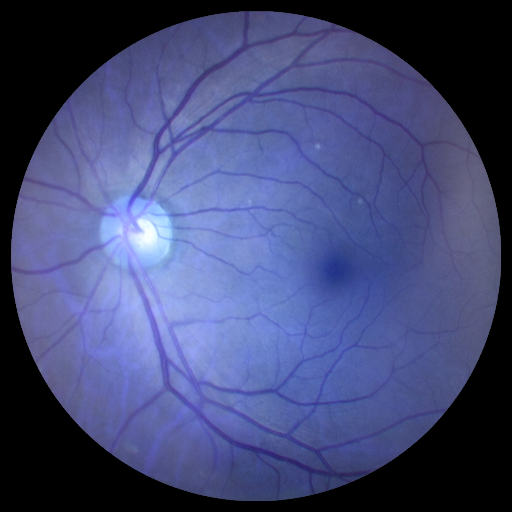}
  	\includegraphics[width=1\linewidth]{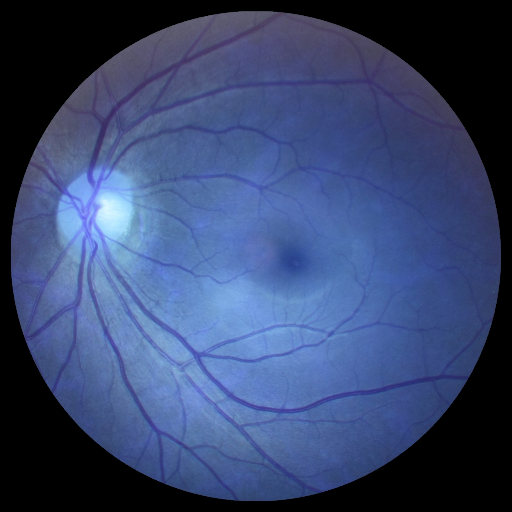}
  	\includegraphics[width=1\linewidth]{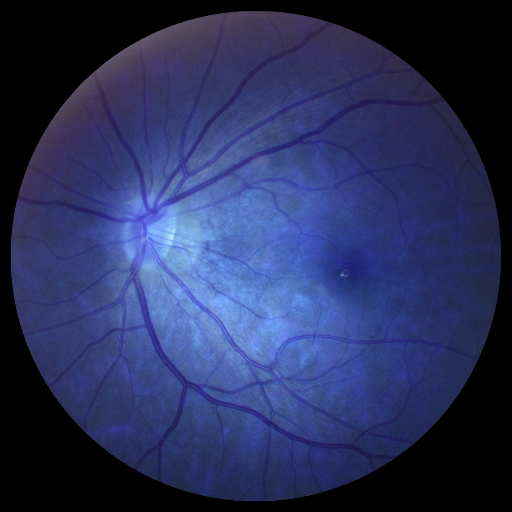}
  	\includegraphics[width=1\linewidth]{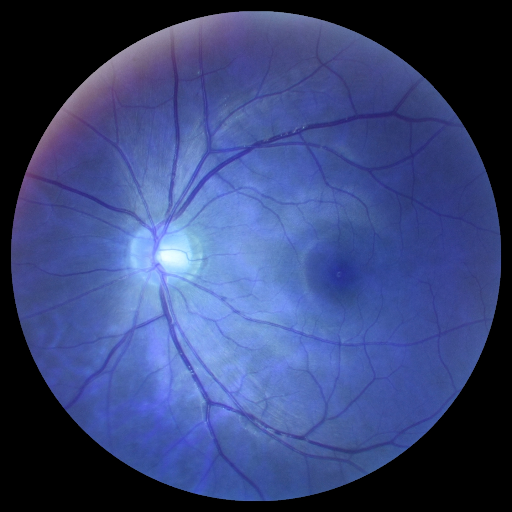}
  	\includegraphics[width=1\linewidth]{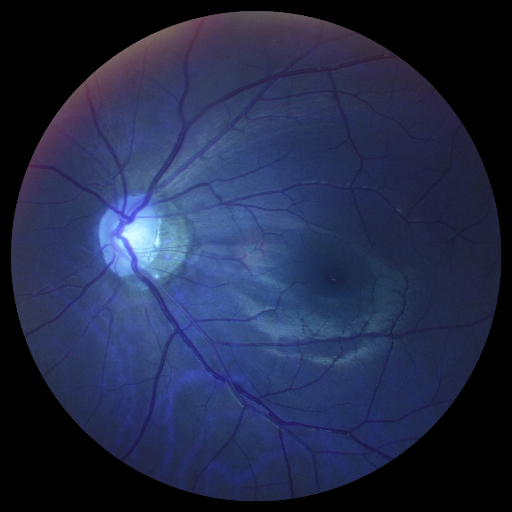}
	\centerline{\tiny (a) Images}
	\end{minipage}
	}
	\subfigure{
	\begin{minipage}[b]{0.12\linewidth}
    \includegraphics[width=1\linewidth]{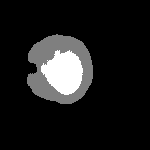}
  	\includegraphics[width=1\linewidth]{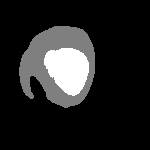}
  	\includegraphics[width=1\linewidth]{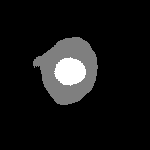}
  	\includegraphics[width=1\linewidth]{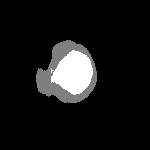}
  	\includegraphics[width=1\linewidth]{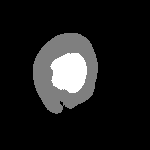}
	\centerline{\tiny (b) Unet}
	\end{minipage}
	}
	\subfigure{
	\begin{minipage}[b]{0.12\linewidth}
    \includegraphics[width=1\linewidth]{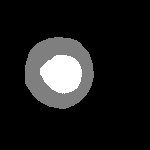}
  	\includegraphics[width=1\linewidth]{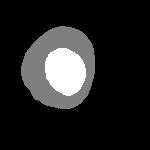}
  	\includegraphics[width=1\linewidth]{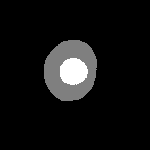}
  	\includegraphics[width=1\linewidth]{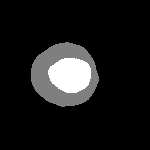}
  	\includegraphics[width=1\linewidth]{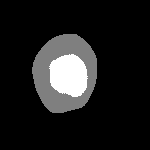}
	\centerline{\tiny (c) Unet with IELs}
	\end{minipage}
	}
	\subfigure{
	\begin{minipage}[b]{0.12\linewidth}
    \includegraphics[width=1\linewidth]{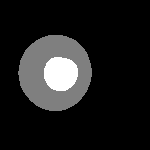}
  	\includegraphics[width=1\linewidth]{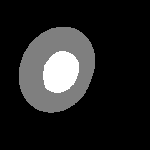}
  	\includegraphics[width=1\linewidth]{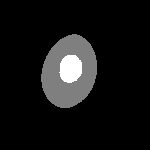}
  	\includegraphics[width=1\linewidth]{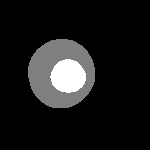}
  	\includegraphics[width=1\linewidth]{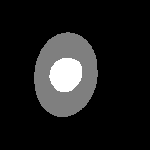}
	\centerline{\tiny (d) Ground truth}
	\end{minipage}
	}
	 \caption{Comparison between Unet \cite{ronneberger2015u} and Unet with curve-motion IELs on REFUGE dataset. (a) and (d) are the original images and the corresponding ground truth, respectively. (b) and (c) exhibit the segmentation results of the original Unet and Unet with curve-motion IELs, respectively.} \label{refuge_fig2}
\end{figure*}

\section{Theoretical Considerations for IELs}
\label{theoretical results}
In the following section, we elucidate the underlying principles governing the functionality of the IELs. We provide rigorous proof of the heat-diffusion IELs' effectiveness as regularizers, establishing their equivalence to a specific type of loss function. This demonstration offers a deeper understanding of their impact within the training process.
In addition, we also prove that training a neural network with appropriately designed IELs on data affected by label corruption is equivalent to training the neural network on clean labels under reasonable assumptions.

\subsection{Theoretical Results for the IELs as Regularizers}
In fact, the heat diffusion equation is often used in PDE-based models involving image denoising. The principle behind this is that heat diffusion equation is a gradient flow for minimizing the following energy
\begin{equation}
  \label{smooth_energy}
  E(u)=\int_\Omega |\nabla u|^2 \mbox{d}x
\end{equation}
More mathematically, we have the following energy decaying property for the solution $u$ of the heat diffusion equation,
\begin{equation}
  \int_\Omega |\nabla u(t+\Delta t,x)|^2 \mbox{d}x \le \int_\Omega |\nabla u(t,x)|^2 \mbox{d}x, \forall \Delta t>0.
\end{equation}
As discussed previously, the property will lead the solution $u$ to become smoother and smoother. As for our heat-diffusion IELs, we can also prove that it has a similar energy decaying property between its input and output. The theorem is organized as follows.
\begin{theorem}
\label{energy_decay_iels}
Define a discrete domain $\Omega_h: (h,2h,\cdots,M_1h)\times(h,2h,\cdots,M_2h)$. Let $U_{in}:\Omega_h\to \mathbb{R}^{M_1\times M_2}$ and $U_{out}:\Omega_h\to \mathbb{R}^{M_1\times M_2}$ be the input and the output of the heat-diffusion IELs. Define a discrete energy
\begin{equation}
  \label{discrete energy}
  E_h(U) := \sum_{i,j} |(\nabla_h U)_{i,j}|^2,
\end{equation}
where $(\nabla_h U)_{i,j}:= (\frac{U_{i+1,j}-U_{i,j}}{h}, \frac{U_{i,j+1}-U_{i,j}}{h}).$ Then we have the following discrete energy decaying property for $U_{in}$ and $U_{out}$:
\begin{equation}
  \label{discrete_energy}
  E_h(U^{in})  \le E_h(U^{out}), \quad \forall h, \Delta t>0.
\end{equation}
\end{theorem}

\begin{proof}
It is sufficient to prove that $E_h(U^{n}) \le E_h(U^{n+1})$, where $U^{n+1} = U^{n} - F_h^{\Delta}(U^{n})\Delta t$.
Reshape both $U^{n}$ and $U^{n+1}$ in column-wise ordering to two vectors and denote the corresponding $Vec(U^{n})$ and $Vec(U^{n+1})$ as $V_n$ and $V_{n+1}$, respectively. Then we have
\begin{equation}
  \begin{aligned}
    E_h(U^{n}) &= \frac{1}{h^2}(V_n^TA_x^TA_xV_n + V_n^TA_y^TA_yV_n),\\
    E_h(U^{n+1}) &= \frac{1}{h^2}(V_{n+1}^TA_x^TA_xV_{n+1} + V_{n+1}^TA_x^TA_xV_{n+1}),
  \end{aligned}
\end{equation}
Here
$$
\begin{aligned}
  &A_x = \Lambda_{M_2}\otimes I_{M_1},\\
  &A_y = I_{M_2}\otimes\Lambda_{M_1},
\end{aligned}
$$
where $\otimes$ represents the Kronecker product. $I_{M_k}$ is an $M_k\times M_k$ identity matrix and
$$
\Lambda_{ M_k } = \left[ \begin{array} { c c c c c c } - 1 & 1 & & & & \\  & - 1 & 1 & & & \\ & & & \ddots & \ddots & \\ & & &  & - 1 & 1 \\  & & & & & 0 \end{array} \right]_{ M_k \times M_k }, \quad k=1, 2.
$$
Moreover, we have
\begin{equation}
  \label{V_diff}
  \begin{aligned}
    V_{n+1} & = Vec(U^n - F_h^{\Delta}(U^{n})\Delta t)\\
            &= V_n - Vec(F_h^{\Delta}(U^{n})\Delta t)\\
            & = V_n - Vec(\frac{\Delta t}{h^2}(UD_{M_2} +D_{M_1}U)) \quad(D_{M_i} \text{ is defined in } (\ref{d_i}))\\
            & = V_n - \frac{\Delta t}{h^2}(D_{M_2}\otimes I_{M_1} + I_{M_2}\otimes D_{M_1})V_n\\
            & = V_n + \frac{\Delta t}{h^2}((\Lambda_{ M_2 }^T\Lambda_{ M_2 })\otimes I_{M_1}^2 + I_{M_2}^2\otimes(\Lambda_{M_1}^T\Lambda_{M_1}))V_n\\
            & = V_n + \frac{\Delta t}{h^2}(A_x^TA_x+A_y^TA_y)V_n\\
  \end{aligned}
\end{equation}

By subtracting $V_n$ from both sides of equation \eqref{V_diff} and subsequently multiplying both sides by $(V_{n+1}-V_n)^T$, we obtain
\begin{equation}
  \label{V^2}
  (V_{n+1} - V_n)^T(V_{n+1} - V_n) = \frac{\Delta t}{h^2}(V_{n+1}-V_n)^T(A_x^TA_x+A_y^TA_y)V_n
\end{equation}
Note that
\begin{equation}
  \label{E_diff}
  \begin{aligned}
      E(U^{n+1}) - E(U^{n})&= \frac{1}{h^2}V_{n+1}^{T}(A_x^TA_x+A_y^TA_y)V_{n+1} - \frac{1}{h^2}V_n^{T}(A_x^TA_x+A_y^TA_y)V_n\\&=\frac{2}{h^2}(V_{n+1}-V_n)^T(A_x^TA_x+A_y^TA_y)V_n \\&+ \frac{1}{h^2}(V_{n+1}-V_n)^{T}(A_x^TA_x+A_y^TA_y)(V_{n+1}-V_n).
  \end{aligned}
\end{equation}
Combining \eqref{V^2} and \eqref{E_diff} , we have
\begin{equation}
  E(U^{n+1}) - E(U^{n}) = \frac{2}{\Delta t}(V_{n+1} - V_n)^T(V_{n+1} - V_n) + (V_{n+1}-V_n)^{T}(A_x^TA_x+A_y^TA_y)(V_{n+1}-V_n).
\end{equation}
Since $A_x^TA_x+A_y^TA_y$ is positive semi-definite. We conclude that $E(U^{n+1}) - E(U^{n}) \ge 0$
\end{proof}
Theorem \ref{energy_decay_iels} shows that the input and output of the heat-diffusion IELs indeed have a relationship on the energy of the original PDE. The energy defined in \eqref{discrete energy} can be viewed as a smoothness loss. During the training process, the IELs magnify this energy loss to compel neural networks to generate outputs of reduced energy loss. Since the inputs of the IELs are used for prediction, neural networks trained with the corresponding IELs will give prediction results with less energy loss. From this perspective, the neural network trained with the corresponding IELs can be viewed as an optimization result under both PDE-based energy loss and originial loss function, thus enjoys desired physical or mathematical properties.


\begin{remark}
  \label{implicit}
While our heat-diffusion IELs are derived from the explicit Euler method, which is known to have numerical stability issues, it is worth noting that the input of the heat-diffusion IELs can be viewed as the solution obtained through the implicit Euler method with output being the initial value. This ensures that the input possesses the desired physical properties under the evolutions.
\end{remark}

\subsection{Theoretical Results for Solving Noisy Label Issues}
\revised{In this part, we demonstrate that, under reasonable assumptions, achieving convergence of the loss function between the neural network's outputs after applying IELs and noisy labels implies convergence between the neural network's outputs and clean labels.}
\begin{theorem}
  \label{equivalence}
Let $N$ and $N_{epochs}$ be the numbers of training samples and epochs, respectively. Denote the true masks and noisy masks as $\{M_i\}_{i=1}^N$ and $\{\hat{M_i}\}_{i=1}^N$. Let $\{U_i^{in}\}_{i=1}^N$ and $\{U_i^{out}\}_{i=1}^N$ be the inputs and outputs of the inverse evolution layers. Assume that
\begin{itemize}
  \item there exists an evolution process $\mathcal{F}$ such that $M_i = \mathcal{F}(\hat{M_i}), i=1,2,\cdots,N$ and $\mathcal{F}$ is Lipschitz continuous, i.e. there exists a constant $K\ge 0$, such that
  \begin{equation}
    ||\mathcal{F}(x)-\mathcal{F}(y)||_2\le K||x-y||_2
  \end{equation}
  \item given any $U_i^{in}$, $||U_i^{in}-\mathcal{F}(L(U_i^{in}))||_2\to 0$ as $\Delta t \to 0$.
  \item $\sum_{i=1}^N||U_i^{out}-\hat{M_i}||_2\to 0$ as $N_{epochs} \to \infty$.
\end{itemize}
Then we have
\begin{equation}
  \sum_{i=1}^N||U_i^{in}-M_i||_2\to 0
\end{equation}
as $\Delta t \to 0$ and $N_{epochs} \to \infty$.
\end{theorem}
\begin{proof}
\begin{equation}
  \begin{aligned}
    \sum_{i=1}^N||U_i^{in}-M_i||_2 &\le \sum_{i=1}^N(||U_i^{in}-\mathcal{F}(U_i^{out})||_2 + ||\mathcal{F}(U_i^{out})-M_i||_2)\\
    &=\sum_{i=1}^N(||U_i^{in}-\mathcal{F}(L(U_i^{in}))||_2 + ||\mathcal{F}(U_i^{out})-\mathcal{F}(\hat{M_i})||_2)\\
    &\le \sum_{i=1}^N(||U_i^{in}-\mathcal{F}(L(U_i^{in}))||_2 +  K||U_i^{out}-\hat{M_i}||_2).
  \end{aligned}
\end{equation}
From the second and third assumptions, we deduce that both $\sum_{i=1}^N||U_i^{in}-\mathcal{F}(L(U_i^{in}))||_2$ and $\sum_{i=1}^N ||U_i^{out}-\hat{M_i}||_2$ tend to 0 as $\Delta t \to 0$ and $N_{epochs} \to \infty$, which completes the proof.
\end{proof}

\begin{remark}
  The assumptions in this theorem are reasonable. For the first assumption, the mapping from $\hat{M_i}$ and $M_i$ should be Lipschitz continuous since two similar corrupted labels should come from two similar clean labels. The second assumption can be satisfied if the reverse process of the IELs can accurately approximate $\mathcal{F}$, which is the case for the heat-diffusion IELs as discussed in remark \ref{implicit}. The last assumption is grounded in the fact that neural networks are capable of fitting noisy labels well as the number of training epochs, $N_{epochs}$, approaches infinity.
\end{remark}

\begin{remark}
Indeed, once these assumptions are met, Theorem \ref{equivalence} suggests that if the outputs of IELs closely resemble the corrupted labels, then their inputs will closely resemble the uncorrupted labels. This implies that training on the corrupted labels, to some extent, is analogous to training on the clean labels, making the IELs a valuable tool for mitigating label corruption effects.
\end{remark}

\section{Conclusion and Future Work}
\label{conclusion}
This paper proposes inverse evolution layers (IELs) as a regularization technique for neural networks. The aim of the IELs is to guide neural networks to produce outputs with specific partial differential equation (PDE) priors, thereby endowing them with appropriate properties. IELs can be easily incorporated into various neural networks without introducing new learnable parameters. Moreover, the integration of IELs does not compromise the learning ability of neural networks. We demonstrate the effectiveness of the IELs through two applications: the heat-diffusion IELs to address the issue of noisy labels and the curve-motion IELs to incorporate convex shape regularization into neural networks' outputs. Our experiments have demonstrated the efficacy of these two types of IELs. In addition to the experiments, we theoretically develop our IELs approach by providing useful theorems to clarify the regularization capabilities of IELs and their effectiveness on corrupted labels. Overall, our proposed approach represents a remarkable contribution to the field of neural network regularization, offering a promising means of integrating PDE-based mathematical models into neural networks in a transparent, interpretable, and effective manner.

Regarding future work, we envisage that, apart from heat-diffusion and curve-motion IELs, a variety of other IELs could be formulated to equip neural networks with diverse mathematical or physical properties to address specific image processing tasks such as image inpainting, image generation and others. Additionally, extending IELs to other research domains such as natural language processing and generative models presents a promising area for further investigation.

\bibliographystyle{plain}
\bibliography{iels}
\end{document}